\documentclass[letterpaper,11pt, margin=1in]{article}

\newif\ifnps
\npsfalse

\ifnps
\usepackage{neurips_2022}
\else
\usepackage{fullpage}
\usepackage{natbib}
\fi

\usepackage{amsthm}
\newtheorem{lemma}{Lemma}
\newtheorem{theorem}{Theorem}

\newtheorem{corollary}{Corollary}
\usepackage[utf8]{inputenc} 
\usepackage[T1]{fontenc}    
\usepackage{hyperref}       
\usepackage{url}            
\usepackage{booktabs}       
\usepackage{amsfonts}       
\usepackage{nicefrac}       
\usepackage{microtype}      
\usepackage{xcolor}         
\usepackage{amsmath}
\usepackage{float}
\newtheorem{observation}{Observation}
\usepackage[ruled]{algorithm2e}
\SetKwInput{KwInp}{Input}

\def \R{{\mathbb{R}}}

\usepackage{tikz}
\usetikzlibrary{calc, graphs, graphs.standard, shapes, arrows, positioning, decorations.pathreplacing, decorations.markings, decorations.pathmorphing, fit, matrix, patterns, shapes.misc, tikzmark}

\DeclareMathOperator{\TV}{TV}

\newif\ifcomments
\commentsfalse

\newcommand{\F}{\mathbf{F}}
\newcommand{\E}{\mathbb{E}}
\newcommand{\lp}{\left}
\newcommand{\rp}{\right}

\newtheorem{definition}{Definition}

\title{Learning and Testing Latent-Tree Ising Models Efficiently}

\usepackage{authblk}
\author[2]{Davin Choo\thanks{davin@u.nus.edu}}
\author[1]{Yuval Dagan\thanks{dagan@csail.mit.edu}}
\author[1]{Constantinos Daskalakis\thanks{costis@csail.mit.edu}}
\author[1]{Anthimos-Vardis Kandiros\thanks{kandiros@mit.edu}}

\affil[1]{Department of Electrical Engineering and Computer Science, MIT}
\affil[2]{School of Computing, National University of Singapore}

\begin{document}

\maketitle

\begin{abstract}


We provide time- and sample-efficient algorithms for learning and testing latent-tree Ising models, i.e.~Ising models that may only be observed at their leaf nodes. On the learning side, we obtain efficient algorithms for learning a tree-structured Ising model whose leaf node distribution is close in Total Variation Distance, improving on the results of \cite{cryan2001evolutionary}. On the testing side, we provide an efficient algorithm with fewer samples for testing whether two latent-tree Ising models have leaf-node distributions that are close or far in Total Variation distance. We obtain our algorithms by showing novel localization results for the total variation distance between the leaf-node distributions of tree-structured Ising models, in terms of their marginals on pairs of leaves.
\end{abstract}

\section{Introduction}\label{sec:intro}

Statistical estimation and hypothesis testing challenges involving high-dimensional distributions are central in Statistics, Machine Learning and various other theoretical and applied fields. Core to this challenge is the fact that even the most basic of those challenges require exponential sample sizes in the dimension to solve if no structural or parametric assumptions are placed on the underlying distributions; see e.g.~\cite{daskalakis2017square,canonne2017testing,acharya2018learning,daskalakis2019testing} for discussions of this point and further references. 

The afore-described exponential sample-size barriers motivate the study of models that sidestep those requirements, e.g.~models encapsulating conditional independence structure in the distribution, such as Markov Random Fields (MRFs) and Bayesian networks. In turn, a vast line of research has studied  statistical inference questions involving MRFs and Bayesnets and their applications; see e.g.~\cite{pearl1988probabilistic,lauritzen1996graphical,jordan2004graphical,koller2009probabilistic} for an introduction to graphical models, their uses, and associated inference algorithms, and see e.g.~\cite{chow1968approximating,chow1973consistency,narasimhan2004,ravikumar2010high,tan2011large,jalali2011learning,santhanam2012information,bresler2015efficiently,vuffray2016interaction,klivans2017learning,hamilton2017information,dagan2021learning,kandiros2021statistical,daskalakis2020tree,bhattacharyya2021near,vuffray2022efficient} and the references in the previous paragraph for some classical work and some recent theoretical progress on learning and testing graphical models as well as other types of statistical inference with them.

Despite the vast study of graphical models, and a recent burst of activity towards computationally and statistically efficient algorithms for  inference with them, a broad outstanding challenge in this space lies in computationally efficient inference with graphical models that have {\em latent variables}. Those are widely motivated in practice (see e.g.~\cite{aigner1984latent,bishop1998latent,everett2013introduction,bartholomew2011latent,felsenstein1973maximum}) but inference with them is known to be computationally intractable in general. For example, learning graphical models with latent variables in total variation distance is known to be intractable, even when the underlying graph is a tree~\citep{mossel2005learning}, while in the absence of latent nodes the same problem is computationally tractable, owing to classical work of~\citet{chow1968approximating} and its recent analysis~\citep{daskalakis2020tree,bhattacharyya2021near}. Similarly, computing the likelihood of a tree-structured graphical model is tractable in the absence of latent nodes, but becomes intractable in the presence of latent nodes~\citep{chor2005maximum,roch2006short}. These computational challenges become  more daunting when the underlying graph gets cyclic, and the overall difficulty of handling latent variables has motivated the development of an array of widely-used approximate inference methods, such as the expectation-maximization algorithm of~\cite{dempster1977maximum} and variational inference (see e.g.~\cite{blei2017variational} for a survey). 

A main goal of this work is to advance the frontier of computationally efficient learning and testing of graphical models with latent variables. While learning general tree MRFs over general alphabets is hard \citep{mossel2005learning}, we focus on the binary-alphabet tree-structured Ising models, which have found extensive use in phylogenetics~\citep{felsenstein1973maximum}. We focus on two inference goals:
\begin{enumerate}
    \item {\bf (Proper Learning):} Given sample access to the distribution at the leaves of a tree-structured Ising model ${\cal P}$, we want to learn a tree-structured Ising model ${\cal Q}$ whose leaf-node distribution is $\varepsilon$-close it total variation distance to that of $\cal P$. \label{goal learn}
    \item {\bf (Identity Testing):} Given sample access to the distribution at the leaves of two tree-structured Ising models ${\cal P}$ and ${\cal Q}$ with the same leaf set, we want to distinguish whether the leaf-node distributions of $\cal P$ and $\cal Q$ are equal or at least $\varepsilon$-far in total variation distance.  \label{goal test}
\end{enumerate}
We provide computationally and statistically efficient algorithms for both Goals~\ref{goal learn} and~\ref{goal test}. Our contribution to Goal~\ref{goal learn} is an algorithm whose time- and sample- complexity provide substantial improvements compared to the algorithm by~\cite{cryan2001evolutionary} as well as the algorithm by~\cite{mossel2005learning}, which requires restricting the correlations across the edges of the Ising model. We also improve upon work in the phylogenetic literature (see Section~\ref{sec:related}) which has focused on identifying the latent tree-structure of the  model but also requires restrictions on the Ising model to achieve this. We finally note that, because we are targeting binary alphabet models, the computational intractability results of~\cite{mossel2005learning} for learning tree-structured graphical models with latent variables do not apply to Goal~\ref{goal learn}.

On the technical front, a fruitful approach towards statistical inference with graphical models uses the paradigm of {\em localization}, whose goal is to localize the difference between two graphical models to differences of their marginals involving a small number of  variables. Such localization properties can be used to distinguish between models for the purposes of hypothesis testing, or exploited to learn graphical models or perform hypothesis selection. Localization of the KL divergence between two Bayesian networks with the same DAG follows directly from their shared factorization, which implies that the KL divergence between the two Bayesnets is upper bounded by the sum of the KL divergences of their marginals on different neighborhoods of the graph, involving a node and  its parents. Similar subadditivity results have been established for total variation and square Hellinger distance~\citep{daskalakis2017square} as well as for other distances, for MRFs, and for causal models~\citep{acharya2018learning,daskalakis2019testing,ding2021gans}. They are also known for comparing graphical models on {\em different} graphs, as long as the underlying graphs are trees~\cite{daskalakis2017square}. 
In turn, the afore-described localization results have been exploited to show that the celebrated~\cite{chow1968approximating} algorithm learns tree-structured Ising models with optimal sample complexity~\cite{daskalakis2020tree} and to obtain optimal algorithms for testing Bayesian networks~\cite{daskalakis2017square}. Additionally, localization properties of graphical models are implicit in much of the recent burst of activity on learning graphical models referenced earlier in this section.

On the latent variable front,  \cite{bresler2020learning} proved localization bounds for the same model that we study: Ising models with zero external fields. Yet, their bound is exponential in the number of variables, hence we cannot apply it to obtain efficient algorithms. A main goal of our work is the following:
\begin{enumerate}
    \item[3.]
    {\bf (Localization of TV in  latent-tree Ising models):} We are seeking to upper bound the total variation distance between the leaf-node distributions of two tree-structured Ising models, which have the same leaf set but potentially different underlying trees, in terms of the marginals of these models on pairs of leaves. Further, the bound should be \emph{polynomial} in the size of the tree.
\end{enumerate}  
In the fully observable case, the localization of distance results that are known for tree-structured graphical models exploit  factorization properties of distributions defined on trees, plus combinatorial results that allow writing two tree-structured graphical models under a common factorization~\cite{daskalakis2017square}. The challenge that arises in tree-structured models with latent variables is that their leaf-node distributions result from marginalizing out all non-leaf vertices, thus they cease to have any tree-structured factorization.

\subsection{Results}\label{sec:results}
Let us formally define tree-structured Ising models. One is given some undirected tree $T = (V,E)$ whose leaves are denoted  $1,\dots,n$ and whose internal nodes are denoted  $n+1,\dots,n+n'$. Without loss of generality, we will assume that all non-leaf nodes have degree $3$\footnote{Every tree can be converted into one with all non-leaf nodes having degree $3$, without affecting the leaf distribution. We just contract every path that consists of nodes of degree $2$ into a single edge and split nodes with degree larger than $3$ by introducing edges with $\theta = 1$.}.
With any edge $(i,j) \in E$ there is an associated weight, $\theta_{ij} \in [-1,1]$. (Because the tree is undirected, edge $(i,j)$ is the same as edge $(j,i)$ and $\theta_{ij}=\theta_{ji}$).
Each node $i$ is assigned a \emph{spin}, $x_i \in \{-1,1\}$, and the probability of a spin-configuration is defined as

\[
\Pr\lp[ 
    x_1,\dots,x_{n+n'}
\rp] \propto
\prod_{(i,j) \in E} \frac{1 + \theta_{ij}x_i x_j}{2} \enspace.
\]
We notice that this definition allows for any tree-structured Ising model with zero external fields (a more common expression is $\Pr[x_1,\dots, x_{n+n'}]\propto\exp(\sum_{(i,j)\in E} \beta_{i,j}x_ix_j)/Z$, and this can be translated to our setting by substituting $\theta_{ij} = \E[x_ix_j] = (e^{\beta_{ij}}-e^{-\beta_{ij}})/(e^{\beta_{ij}}+e^{-\beta_{ij}})$). A sample can be obtained by first choosing some internal node as a root for the tree and directing all the edges of the tree away from the root. Then, we draw a uniformly random value for the spin of the root and randomly propagate the values of the spins along the tree as follows: for any directed edge $i \to j$ such that the spin $x_i$ was already set, we set the spin $x_j$ such that $\Pr[x_j = x_i] = (1+ \theta_{ij})/2$ and $\Pr[x_j = -x_i] = (1 - \theta_{ij})/2$. This process has been used as a model in a variety of applications. 



Our first result involves proving bounds on the total variation (TV) distance between two latent tree Ising models. 
In the case of a fully observable tree, these properties can be proven by using the product factorization of the probability distribution over the edges of the tree. 
While it is in general hard to analyze latent-variable graphical models in the same way, the pairwise-marginals between leaves are easily accessible. In the specific case of an Ising model, for any two nodes $i,j$ of the tree, the marginal distribution of $(x_i,x_j)$ is characterized by
\begin{equation} \label{eq:path-correlation}
\alpha_{ij} := \E[x_i x_j]
= \prod_{(k,l) \in P_{ij}}
\theta_{kl}\enspace,
\end{equation}
where $P_{ij}$ is the unique path that connects $i$ and $j$ on the tree. Given $m$ samples, the correlations $\alpha_{ij}$ between any pair of \emph{leaves} uniquely identifies the model and the parameters and further, they can be easily estimated from data approximately. 
Hence, a natural analogue of the marginal distribution of edges in the complete tree would be the marginal distribution of all pairs of leaves in the tree.
In that direction, we provide a bound on the total variation of two leaf distributions based solely on their pairwise correlations, in two settings: when both share the same tree and for a different tree. This can be seen as an approximate tensorization property for TV in latent tree Ising models. 
\begin{theorem}\label{thm:probabilistic}
Let $\mu$ and $\mu^*$ be distributions over the leaves of two tree Ising models with $n$ leaves and assume that $|\E_{\mu}[x_i x_j] - \E_{\mu'}[x_ix_j]| \le \epsilon$ for all leaves $i,j$. Assume also that the minimum diameter of the two trees is $D$. Then:
\begin{itemize}
    \item \textbf{Same topology:} If $\mu$ and $\mu^*$ are defined on the same graph, then $\TV(\mu,\mu^*)\le 2n^2\epsilon$.
    \item \textbf{Different topologies:} If $\mu$ and $\mu^*$ are defined on different graphs then $\TV(\mu,\mu^*) \le O(Dn^5\epsilon)$, where $O$ hides absolute constants.
\end{itemize}
\end{theorem}

We note that the previous bound by \cite{bresler2020learning} was $n 2^n\epsilon$, and it holds for different (and same) topologies. Their setting was slightly more general, as it applied to arbitrary subsets of the nodes of an Ising model, rather than only sets of leaves. Yet, we believe that using the same techniques as the ones we present here, one could prove a more general Theorem, similar to Proposition 1 in Appendix H of \cite{bresler2020learning}. Such a general result would improve the bounds for learning in $k$-local-TV \citep{bresler2020learning,boix2022chow} from being exponential in $k$ to polynomial (see Section~\ref{sec:related} for more details on this line of work).

Theorem~\ref{thm:probabilistic} is a result of independent interest, since it can be used in a variety of applications. A direct corollary of Theorem~\ref{thm:probabilistic} is a polynomial time algorithm for identity testing of latent tree Ising models, which uses a polynomial number of samples.

\begin{corollary}\label{cor:testing}
Let $P, Q$ be leaf distributions of two potentially different tree Ising models. Suppose we are given access to samples from $P$ and we wish to distinguish whether $P = Q$ or $TV(P,Q) > \epsilon$. Assume also that the minimum diameter of the two trees is $D$. Then, there exists a polynomial time algorithm that answers correctly with probability at least $1-\delta$, with sample size
$$
O\lp(\frac{n^{10}D^2\log \frac{n}{\delta}}{\epsilon^2}\rp).
$$
\end{corollary}

Corollary~\ref{cor:testing} can be proven directly by applying
Theorem~\ref{thm:probabilistic}.
It is worth noting that the diameter is typically of the order $O(\log n)$ in many applications of interest. 

%
%
%

To further show the utility of Theorem~\ref{thm:probabilistic}, we provide polynomial-time and polynomial-sample algorithms for learning tree-structured Ising models. We provide two algorithms: one that requires to know the structure of the tree in advance and the second does not require any assumption.

\begin{theorem}\label{thm:alg}
Given $m$ samples from the joint distribution over the leaves of some tree-structured Ising model and given a target error $\epsilon>0$ and confidence $\delta >0$, there exist polynomial-time algorithms for learning a tree-structured Ising model whose marginal over the leaves is $\epsilon$-close to the true marginal in total variation distance, with probability $1-\delta$, with the following sample complexities:
\begin{itemize}
    \item \textbf{Known topology:} If the tree topology $T = (V,E)$ is known and the weights are unknown, the sample complexity is $m = O(n^4 \log (n/\delta)/\epsilon^2)$.
    \item \textbf{Unknown topology:} If the tree and weights are unknown, then $m = O(n^{14}\log (n/\delta)/\epsilon^6)$. 
\end{itemize}
\end{theorem}

The algorithms for both settings consist of two steps: first, they empirically estimate the pairwise correlations between every pair of leaves. Then, they utilize an algorithm that, given these approximate correlations, finds \emph{some} tree Ising model whose pairwise correlations are close to the estimated ones. The guarantees of those algorithms then follow directly from Theorem~\ref{thm:probabilistic}. Notice that for the case of known topology, it is possible to implement this algorithm using a simple linear programming, as outlined in Section~\ref{s:algo}. While for unknown topology, the algorithm is slightly more complicated, and it is outlined in Section~\ref{sec:pr-alg-unknown}.

The only prior work that we are aware of, that provided a polynomial time algorithm for learning in total variation without any restrictions on the weights of the model, is \cite{cryan2001evolutionary}.
They do not explicitly state the sample complexity and it can be inferred from the proof to be at least $n^{89}/\epsilon^{18}$. We notice that their result is slightly more general and holds for general trees with a binary alphabet. We provide a brief technical comparison with their result in Section~\ref{s:algo}. We further note that information theoretically, $\tilde\Theta(n/\epsilon^2)$ samples are sufficient and necessary for learning the joint distribution, in the known and unknown topology settings, and these are also the best known lower bounds for efficient algorithms. This is well known, and can be shown, for example, by modifying the arguments in \citet{devroye2020minimax,brustle2020multi,koehler2020note}. For completeness, we provide a sketch of these arguments in Section~\ref{sec:info-theor}. It remains an open problem whether there is a statistical-computational gap in this setting.

\subsection*{Acknowledgements.}
Constantinos Daskalakis was supported by NSF Awards CCF-1901292, DMS-2022448 and DMS2134108, a Simons Investigator Award, the Simons Collaboration on the Theory of Algorithmic Fairness, a DSTA grant, and the DOE PhILMs project (DE-AC05-76RL01830).
Vardis Kandiros was supported by a Fellowship of the Eric and Wendy Schmidt Center at the Broad Institute of MIT and Harvard and by the Onassis
Foundation-Scholarship ID: F ZP 016-1/2019-2020.
Yuval Dagan gratefully acknowledges the NSF's support of FODSI through grant DMS-2023505.

\section{Technical Contributions -- Proof Sketch}

We describe the main tools for the proofs of Theorem~\ref{thm:probabilistic}.

\subsection{Preliminaries}
For the discussion, fix some tree $T=(V,E)$. For any leaves $i,j$, let $P_{ij}$ denote the path connecting them.
Denote by $\theta$ any vector in $[-1,1]^E$ whose entry $\theta_e$ denotes the correlation across the edge $e\in E$.
When we write $\alpha,\hat{\alpha}$ etc., this corresponds to a vector in $[-1,1]^{n\choose 2}$, whose entries, $\alpha_{ij}$ are indexed by two distinct leaves $i\ne j$. In general $\alpha$ can be an arbitrary vector, yet, we say that $\alpha$ is \emph{induced} by some probability distribution on a tree $T$ if it represents the pairwise correlations of the leaves in some Ising model that is defined over the tree (i.e. if \eqref{eq:path-correlation} holds for some edge-correlations $\{\theta_e\}_{e\in E}$). Given a tree $T$, edge-correlations $\theta$ and $x\in \{-1,1\}^n$,
denote by $\Pr_{T,\theta}[x]$ the probability that the leaves equal $x$ under the Ising model defined by $T$ and $\theta$. We say that $\Pr_{T,\theta}$ is the \emph{leaf distribution} over $\{-1,1\}^n$.

First, we define a pair of leaves $i,j$ to be a \emph{cherry} if they share their common neighbor (recall that a leaf has only one neighbor). In other words, if one directs the edges from some internal node to the leaves, $i$ and $j$ would share their parent. 



\subsection{An expression for the probability distribution on the leaves from \cite{bresler2020learning}} \label{sec:closed-form}

We describe a convenient closed-form expression for the probability distribution over the leaves of the tree.
To describe it, we begin with some definitions. Let $S$ be a subset of the leaves of even cardinality.
Then, there is a \emph{unique}\footnote{The uniqueness holds if each internal node has degree $3$ and this assumption is without loss of generality, as we explain in a footnote in Section~\ref{sec:intro}.} way to partition $S$ into $|S|/2$ pairs $(x_1,y_1), \ldots, (x_{|S|/2}, y_{|S|/2})$ such that the path connecting $x_i$ and $y_i$ is edge-disjoint from the path connecting $x_j$ and $y_j$, for all $i\ne j$. This partitioning can be obtained by matching leaves that are closest to being a cherry (i.e.\ siblings have highest precedence) repeatedly.
For example, if we have $S = \{i,j,k,\ell,m,p\}$ in the tree shown in \ref{fig:path-removal}, then we partition $S$ into $(i,\ell),(j,m)$ and $(k,p)$. 
The leaf distribution can be described as the following multilinear function of $x$, whose coefficients, that are indexed by sets $S$ of even cardinality, rely on the aforementioned partitioning into pairs:
\begin{equation}\label{eq:f-def}
f_x^T(\alpha) 
:= 2^{-n} \cdot \sum_{\stackrel{S \subseteq [n]}{|S| \text{ even}}}\alpha_S^T\prod_{i \in S} x_i~, \quad \text{where }
\alpha_S^T := \prod_{i=1}^{|S|/2} \alpha_{x_i y_i}
\end{equation}
The following lemma, which is a special case of Theorem H.1 in \cite{bresler2020learning}, argues that $f^T_x(\alpha)$ is the leaf distribution, as a function of $x$.

\begin{lemma}[\cite{bresler2020learning}]\label{l:basic_extend_intro}
For any latent tree distribution with tree-topology $T$ whose pairwise correlations over the leaves equal $\alpha = (\alpha_{ij})_{i,j \text{ leaves}}$. Then, the probability of any configuration $x = (x_1,\dots,x_n)\in \{-1,1\}^n$ on the leaves equals $f_x^T(\alpha)$.
\end{lemma}

For convenience, we give a proof of Lemma~\ref{l:basic_extend_intro} in Section~\ref{s:distribution}.



\subsection{Technical tools for the tensorization of Theorem~\ref{thm:probabilistic} (same topology)}\label{sec:tools-same-top}
We now utilize Lemma~\ref{l:basic_extend_intro} to prove Theorem~\ref{thm:probabilistic} (same topology).
The total variation distance between two distributions on the same topology $T$ with induced weight-vectors $\alpha$ and $\hat{\alpha}$ is $\sum_x |\Pr_{T,\alpha}[x]- \Pr_{T,\hat\alpha}[x]|/2 = \sum_x |f_x(\alpha) - f_x(\hat\alpha)|/2$, where $T$ is omitted from $f_x^T$ for brevity.
We would like to bound the above expression, assuming that $\alpha$ and $\hat{\alpha}$ are close and this corresponds to showing some Lipschitzness properties on $f_x(\alpha)$. We show that such a Lipschitzness property holds in the neighborhood of a probability distribution. Namely, that if $\alpha$ is induced by a distribution and we change one entry of $\alpha$, then $f_x(\alpha)$ does not change much: 


\begin{lemma}[Formal statement in Lemma~\ref{l:difference_fixed}]\label{l:lipschitz}
Suppose $\alpha \in [-1,1]^{n\choose 2}$ is induced by some probability distribution on a tree $T$. Denote by $\alpha^{(ij)} \in [-1,1]^{n\choose 2}$ the vector that agrees with $\alpha$ everywhere, except for pair of leaves $ij$, where $\alpha^{(ij)}_{ij} \ne \alpha_{ij}$.
Then,
$$
\sum_{x\in \{-1,1\}^n}|f_x(\alpha) - f_x(\alpha^{(ij)})| \leq |\alpha_{ij} - \alpha^{(ij)}_{ij}|
$$
\end{lemma}
\begin{proof}[Proof sketch]
Let $\theta$ denote the weight vector on the edges of $T$ that induces the correlation-vector $\alpha$ in accordance to \eqref{eq:path-correlation}, and let $\theta'$ be the weight vector that is obtained from $\theta$ by replacing any weight along the path from $i$ to $j$ with $0$. Then, it is straightforward to see that for all $x \in \{-1,1\}^n$,
\begin{equation}\label{eq:replacing-one}
\frac{|f_x(\alpha) - f_x(\alpha^{(ij)})|}{|\alpha_{ij} - \alpha^{(ij)}_{ij}|}
= \Pr_{T,\theta'}[x].
\end{equation}
In other words, the ratio above equals the probability of $x$ in the distribution that is obtained from $\Pr_{T,\theta}$ by removing the path from $i$ to $j$ in $T$, as depicted in Figure~\ref{fig:path-removal} (a). Since the right hand side represents a distribution over $x$, if we sum over $x$ the result equals $1$.
\end{proof}

\begin{figure}[t]
\centering
\resizebox{\linewidth}{!}{%
\begin{tikzpicture}
%
%
\node[draw, circle, minimum size=15pt, inner sep=2pt] at (0,0) (i-left) {};
\node[draw, circle, minimum size=15pt, inner sep=2pt, right=30pt of i-left] (v1-left) {};
\node[draw, circle, minimum size=15pt, inner sep=2pt, right=50pt of v1-left] (v2-left) {};
\node[draw, circle, minimum size=15pt, inner sep=2pt, right=50pt of v2-left] (v3-left) {};
\node[draw, circle, minimum size=15pt, inner sep=2pt, right=30pt of v3-left] (j-left) {};
\node[draw, circle, minimum size=15pt, inner sep=2pt, below=30pt of v1-left] (v4-left) {};
\node[draw, circle, minimum size=15pt, inner sep=2pt, left=30pt of v4-left] (k-left) {};
\node[draw, circle, minimum size=15pt, inner sep=2pt, right=30pt of v4-left] (v5-left) {};
\node[draw, circle, minimum size=15pt, inner sep=2pt, below=30pt of v3-left] (m-left) {};
\node[draw, circle, minimum size=15pt, inner sep=2pt, above=30pt of v2-left] (v6-left) {};
\node[draw, circle, minimum size=15pt, inner sep=2pt, left=30pt of v6-left] (v7-left) {};
\node[draw, circle, minimum size=15pt, inner sep=2pt, left=30pt of v7-left] (l-left) {};
\node[draw, circle, minimum size=15pt, inner sep=2pt, right=30pt of v6-left] (v8-left) {};
\node[draw, circle, minimum size=15pt, inner sep=2pt, below=15pt of v7-left] (v9-left) {};

\draw[thick, dashed] (i-left) -- (v1-left);
\draw[thick, dashed] (v1-left) -- (v2-left);
\draw[thick, dashed] (v2-left) -- (v3-left);
\draw[thick, dashed] (v3-left) -- (j-left);
\draw[thick] (v1-left) -- (v4-left);
\draw[thick] (v4-left) -- (k-left);
\draw[thick] (v4-left) -- (v5-left);
\draw[thick] (v3-left) -- (m-left);
\draw[thick] (v2-left) -- (v6-left);
\draw[thick] (v6-left) -- (v7-left);
\draw[thick] (v7-left) -- (l-left);
\draw[thick] (v6-left) -- (v8-left);
\draw[thick] (v7-left) -- (v9-left);

\node[above=5pt of i-left] {\small $i$};
\node[above=5pt of j-left] {\small $j$};
\node[above=5pt of k-left] {\small $k$};
\node[left=5pt of l-left] {\small $\ell$};
\node[right=5pt of m-left] {\small $m$};
\node[above=5pt of v5-left] {\small $p$};
\node[below=50pt of v2-left] {$(a)$};

%
%
\node[draw, circle, minimum size=15pt, inner sep=2pt] at (10,0) (i-right) {};
\node[draw, circle, minimum size=15pt, inner sep=2pt, right=30pt of i-right] (v1-right) {};
\node[draw, circle, minimum size=15pt, inner sep=2pt, right=50pt of v1-right] (v2-right) {};
\node[draw, circle, minimum size=15pt, inner sep=2pt, right=50pt of v2-right] (v3-right) {};
\node[draw, circle, minimum size=15pt, inner sep=2pt, right=30pt of v3-right] (j-right) {};
\node[draw, circle, minimum size=15pt, inner sep=2pt, below=30pt of v1-right] (v4-right) {};
\node[draw, circle, minimum size=15pt, inner sep=2pt, left=30pt of v4-right] (k-right) {};
\node[draw, circle, minimum size=15pt, inner sep=2pt, right=30pt of v4-right] (v5-right) {};
\node[draw, circle, minimum size=15pt, inner sep=2pt, below=30pt of v3-right] (m-right) {};
\node[draw, circle, minimum size=15pt, inner sep=2pt, above=30pt of v2-right] (v6-right) {};
\node[draw, circle, minimum size=15pt, inner sep=2pt, left=30pt of v6-right] (v7-right) {};
\node[draw, circle, minimum size=15pt, inner sep=2pt, left=30pt of v7-right] (l-right) {};
\node[draw, circle, minimum size=15pt, inner sep=2pt, right=30pt of v6-right] (v8-right) {};
\node[draw, circle, minimum size=15pt, inner sep=2pt, below=15pt of v7-right] (v9-right) {};

\draw[thick, dashed] (i-right) -- (v1-right);
\draw[thick, dashed] (v1-right) -- (v2-right);
\draw[thick, dashed] (v2-right) -- (v3-right);
\draw[thick] (v3-right) -- (j-right);
\draw[thick, dashed] (v1-right) -- (v4-right);
\draw[thick, dashed] (v4-right) -- (k-right);
\draw[thick] (v4-right) -- (v5-right);
\draw[thick, dashed] (v3-right) -- (m-right);
\draw[thick, dashed] (v2-right) -- (v6-right);
\draw[thick, dashed] (v6-right) -- (v7-right);
\draw[thick, dashed] (v7-right) -- (l-right);
\draw[thick] (v6-right) -- (v8-right);
\draw[thick] (v7-right) -- (v9-right);

\node[above=5pt of i-right] {\small $i$};
\node[above=5pt of j-right] {\small $j$};
\node[above=5pt of k-right] {\small $k$};
\node[left=5pt of l-right] {\small $\ell$};
\node[right=5pt of m-right] {\small $m$};
\node[above=5pt of v5-right] {\small $p$};
\node[below=50pt of v2-right] {$(b)$};
\end{tikzpicture}
}
\caption{Path removal: In (a) We depict the graph obtained from $T$ by removing the path from $i$ to $j$ (i.e. the dashed edges are being removed). In (b) we depict a removal of the quartet $\{i,k,\ell,m\}$.}
\label{fig:path-removal}
\end{figure}
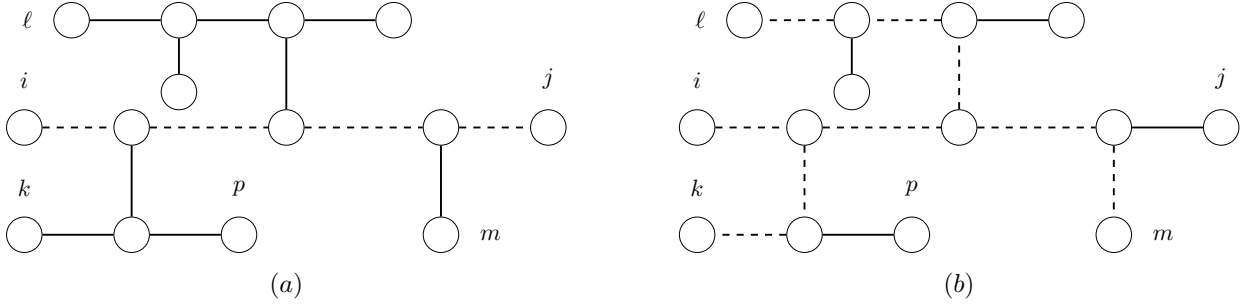
To bound the total variation between two weight vectors $\alpha$ and $\hat{\alpha}$, one would attempt to directly apply Lemma~\ref{l:lipschitz} multiple times, each time substituting one entry of $\alpha$ with its corresponding entry of $\hat{\alpha}$. However, in the process of transforming the vector one coordinate at a time, we may stumble upon an intermediate weight vector $\alpha'$ that is \emph{not} induced by a probability distribution and so Lemma~\ref{l:lipschitz} does not apply. Hence, one has to prove an analogue of Lemma~\ref{l:lipschitz} for the case that $\alpha$ is close to being induced by a distribution.
Interestingly, this can be proved by an inductive application of Lemma~\ref{l:lipschitz}.


\subsection{Technical Tools for the tensorization of Theorem~\ref{thm:probabilistic} (different topologies)}

In this section, we aim to bound the total variation distance between a probability distribution defined on a tree $T$ with weights $\alpha$ and another defined on $\hat{T}$ with weight $\hat{\alpha}$, under the assumptions that the weights are $\epsilon$-close: $|\alpha_{ij}-\hat\alpha_{ij}|\le \epsilon$.
To compare between two different topologies, we will use the known fact (folklore) that the topology on the nodes of a tree is completely determined by the set of all subgraphs that are induced by four leaves (\emph{quartets}).
Hence, in order to analyze the difference between two graphs, we can analyze the difference between these subgraphs.
This is significantly easier to analyze since the subgraphs contain only four nodes each.
For this purpose, we introduce below some useful concenpts, inspired by the phylogenetics literature.

\paragraph{Definitions of a quartet.} 

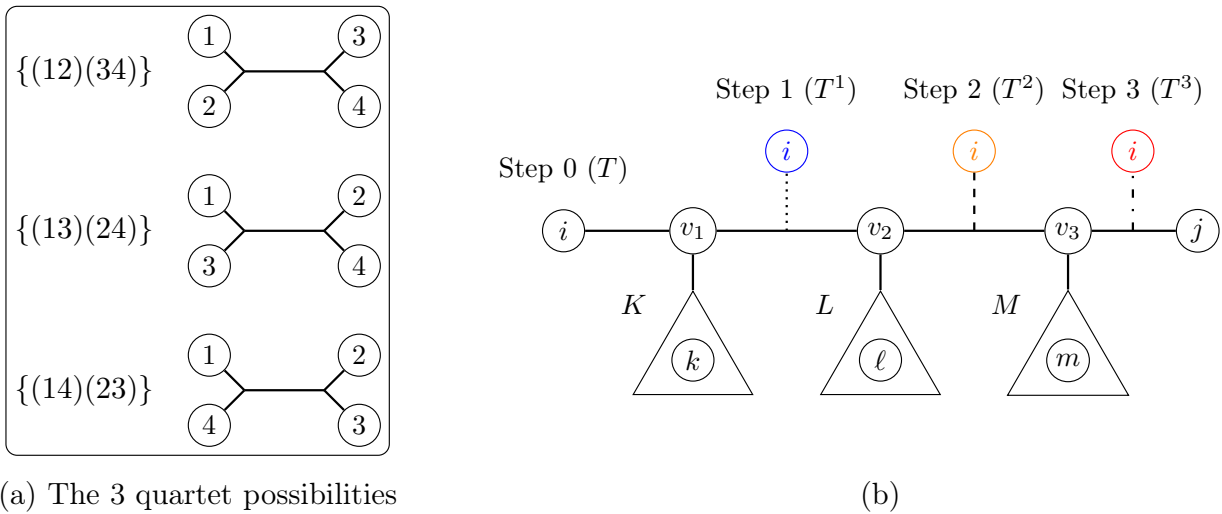
\begin{figure}[htbp]
\centering
\resizebox{\linewidth}{!}{%
\begin{tikzpicture}
%
%
\coordinate (midL1) at (-4,2);
\coordinate (midR1) at (-3,2);
\node[draw, circle, minimum size=15pt, inner sep=2pt, above left=10pt of midL1] (i1) {\small $1$};
\node[draw, circle, minimum size=15pt, inner sep=2pt, below left=10pt of midL1] (j1) {\small $2$};
\node[draw, circle, minimum size=15pt, inner sep=2pt, above right=10pt of midR1] (k1) {\small $3$};
\node[draw, circle, minimum size=15pt, inner sep=2pt, below right=10pt of midR1] (l1) {\small $4$};
\draw[thick] (midL1) -- (midR1);
\draw[thick] (midL1) -- (i1);
\draw[thick] (midL1) -- (j1);
\draw[thick] (midR1) -- (k1);
\draw[thick] (midR1) -- (l1);

%
%
\coordinate (midL2) at (-4,0);
\coordinate (midR2) at (-3,0);
\node[draw, circle, minimum size=15pt, inner sep=2pt, above left=10pt of midL2] (i2) {\small $1$};
\node[draw, circle, minimum size=15pt, inner sep=2pt, below left=10pt of midL2] (k2) {\small $3$};
\node[draw, circle, minimum size=15pt, inner sep=2pt, above right=10pt of midR2] (j2) {\small $2$};
\node[draw, circle, minimum size=15pt, inner sep=2pt, below right=10pt of midR2] (l2) {\small $4$};
\draw[thick] (midL2) -- (midR2);
\draw[thick] (midL2) -- (i2);
\draw[thick] (midL2) -- (k2);
\draw[thick] (midR2) -- (j2);
\draw[thick] (midR2) -- (l2);

%
%
\coordinate (midL3) at (-4,-2);
\coordinate (midR3) at (-3,-2);
\node[draw, circle, minimum size=15pt, inner sep=2pt, above left=10pt of midL3] (i3) {\small $1$};
\node[draw, circle, minimum size=15pt, inner sep=2pt, below left=10pt of midL3] (l3) {\small $4$};
\node[draw, circle, minimum size=15pt, inner sep=2pt, above right=10pt of midR3] (j3) {\small $2$};
\node[draw, circle, minimum size=15pt, inner sep=2pt, below right=10pt of midR3] (k3) {\small $3$};
\draw[thick] (midL3) -- (midR3);
\draw[thick] (midL3) -- (i3);
\draw[thick] (midL3) -- (l3);
\draw[thick] (midR3) -- (j3);
\draw[thick] (midR3) -- (k3);

%
%
\node[inner sep=0pt] at ($(midL1) + (-2,0)$) (1234) {$\{(12)(34)\}$};
\node[inner sep=0pt] at ($(midL2) + (-2,0)$) (1324) {$\{(13)(24)\}$};
\node[inner sep=0pt] at ($(midL3) + (-2,0)$) (1423) {$\{(14)(23)\}$};
\node[fit=(i1)(j1)(k1)(l1)(i2)(j2)(k2)(l2)(i3)(j3)(k3)(l3)(1234)(1324)(1423), draw, inner sep=3pt, rounded corners] (possibilities) {};
\node[below=5pt of possibilities] (a-label) {(a) The 3 quartet possibilities};


%
%
\node[draw, circle, minimum size=15pt, inner sep=2pt] at (0,0) (i) {\small $i$};
\node[draw, circle, minimum size=15pt, inner sep=2pt, right=30pt of i] (v1) {\small $v_1$};
\node[draw, circle, minimum size=15pt, inner sep=2pt, right=50pt of v1] (v2) {\small $v_2$};
\node[draw, circle, minimum size=15pt, inner sep=2pt, right=50pt of v2] (v3) {\small $v_3$};
\node[draw, circle, minimum size=15pt, inner sep=2pt, right=30pt of v3] (j) {\small $j$};

\draw[thick] (i) -- (v1);
\draw[thick] (v1) -- (v2);
\draw[thick] (v2) -- (v3);
\draw[thick] (v3) -- (j);

%
%

\node[draw, circle, minimum size=15pt, inner sep=2pt, below=30pt of v1] (k) {\small $k$};
\node[draw, circle, minimum size=15pt, inner sep=2pt, below=30pt of v2] (l) {\small $\ell$};
\node[draw, circle, minimum size=15pt, inner sep=2pt, below=30pt of v3] (m) {\small $m$};
\node[fit=(k), draw, inner sep=1pt, regular polygon, regular polygon sides=3] (k-tree) {};
\node[fit=(l), draw, inner sep=1pt, regular polygon, regular polygon sides=3] (l-tree) {};
\node[fit=(m), draw, inner sep=1pt, regular polygon, regular polygon sides=3] (m-tree) {};

\draw[thick] (v1) -- (k-tree);
\draw[thick] (v2) -- (l-tree);
\draw[thick] (v3) -- (m-tree);

%
%
\node[draw, blue, circle, minimum size=15pt, inner sep=2pt] at ($(v1)!0.5!(v2) + (0,1)$) (i1) {\small $i$};
\node[draw, orange, circle, minimum size=15pt, inner sep=2pt] at ($(v2)!0.5!(v3) + (0,1)$) (i2) {\small $i$};
\node[draw, red, circle, minimum size=15pt, inner sep=2pt] at ($(v3)!0.5!(j) + (0,1)$) (i3) {\small $i$};

\draw[thick, dotted] ($(v1)!0.5!(v2)$) -- (i1);
\draw[thick, dashed] ($(v2)!0.5!(v3)$) -- (i2);
\draw[thick, loosely dashdotted] ($(v3)!0.5!(j)$) -- (i3);

%
%
\node[above=5pt of i] {\small Step 0 ($T$)};
\node[above=5pt of i1] {\small Step 1 ($T^1$)};
\node[above=5pt of i2] {\small Step 2 ($T^2$)};
\node[above=5pt of i3] {\small Step 3 ($T^3$)};
\node[above left=5pt of k-tree] {\small $K$};
\node[above left=5pt of l-tree] {\small $L$};
\node[above left=5pt of m-tree] {\small $M$};
\node[] at (a-label -| l-tree) {(b)};
\end{tikzpicture}
}
\caption{(a) The three possible topologies for a quartet. 
In the first topology, $\{(12)(34)\}$, the path from $1$ to $2$ does not intersect the path from $3$ to $4$. Further, $\alpha_{12}\alpha_{34} \ge \alpha_{13}\alpha_{24} = \alpha_{14}\alpha_{23}$; (b) The different positions of $i$ in its movement across the tree towards $j$. Each position corresponds to a different tree $T^i$.
}
\label{fig:quartet-KLM}
\end{figure}

In the discussion below, we focus without loss of generality in the case where $\alpha_{ij} \geq 0$ for all $i,j$. Similar claims can be made for arbitrary signs.
We will use the notion of a \emph{quartet} of leaves: this is a collection of four leaves, $\{i,j,k,\ell\}$. It is well known in the phylogenetics literature that for every quartet there are $3$ topologically distinct ways for these $4$ leaves to connect with each other, if we contract all the paths leading to other leaves in the tree.
Depending on which of the $3$ ways we have, we say that the tree \emph{induces} a topology for a specific quartet.
We denote the three induced topologies by $\{(ij)(k\ell)\},\{(ik)(j\ell)\}$ and $\{(i\ell)(jk)\}$, where $\{(ij)(k\ell)\}$ means that the path $P_{ij}$ is edge disjoint from the path $P_{k\ell}$, and similarly for the other topologies,
as depicted in Figure~\ref{fig:quartet-KLM} (a). 
For a fixed quartet $\{i,j,k,l\}$,
there are three quantities that determine the topology, out of the three possibilities, and those are $\alpha_{ij}\alpha_{k\ell}$, $\alpha_{ik}\alpha_{j\ell}$ and $\alpha_{i\ell}\alpha_{jk}$. It is known that two of these quantities are always equal and always smaller than the third, which determines the true topology: if $\alpha_{ij}\alpha_{k\ell}$ is the largest then $\{(ij)(k\ell)\}$ is the topology.
If two trees induce a different topology for a quartet, we say that the trees \emph{disagree} on that quartet, otherwise we say that they \emph{agree} on it. 
It is known (folklore) that if two trees agree on all the quartets, then they should have the same topology.
Lastly, note that in the special case where leaves $\{i,j\}$ form a cherry, the path $P_{ij}$ does not share edges with the path between any two other leaves. This implies that the topology of $\{ijk\ell\}$ would necessarily be $\{(ij),(j\ell)\}$ for any two other leaves $k$ and $\ell$. 

For the analysis, we would like to quantify how sensitive is the topology of a quartet to changing the weights. For any quartet $\{i,j,k,\ell\}$ and weight-vector $\alpha$, define
\begin{align*}
\Delta_{ijk\ell} := \max\{\alpha_{ij}\alpha_{k\ell}, \alpha_{ik}\alpha_{j\ell}, \alpha_{i\ell}\alpha_{jk}\} - \min\{\alpha_{ij}\alpha_{k\ell}, \alpha_{ik}\alpha_{j\ell}, \alpha_{i\ell}\alpha_{jk}\}
\end{align*}
(where the dependence on $\alpha$ is omitted for brevity).
Recall that the topology is determined by the largest of these three products, while the other two smaller products are identical.
Hence, it is easy to see that two trees $T$ and $\hat{T}$, with $\epsilon$-close weights $\alpha$ and $\hat{\alpha}$, disagree only on quartets where $\Delta_{ijk\ell}\le 2\epsilon$.

\paragraph{General approach}
Recall from Section~\ref{sec:tools-same-top} that we transformed $\alpha$ into $\hat{\alpha}$ by replacing its values one coordinate at a time for the case where $\hat{T} = T$. Here, we first fix $\alpha$ and transform $T$ into $\hat{T}$.
In other words, we find a sequence of topologies $T^1=T, T^2,\ldots,T^k = \hat{T}$ that interpolates between $T$ and $\hat{T}$ in a way that $T^i,T^{i+1}$ only differ in a small part of the graph. After transforming $T$ into $\hat{T}$, we then transform $\alpha$ into $\hat{\alpha}$.

\paragraph{Transforming one tree into the other.}
We now describe in more detail the sequence of local moves from $T$ to $\hat{T}$.
Intuitively, the quartets are the analog of the pairwise distances in the fixed topology setting, and so we will measure dissimilarity between trees by the number of quartets that they disagree on.
Thus, the goal is to produce a sequence of local topological changes that reduces quartet disagreements between $T$ and $\hat{T}$ while ensuring that each consecutive pair of terms $f^{T^i}_x$ and $f^{T^{i-1}}_x$ is close.

The sequence of moves starts by identifying two leaves $i,j$ that are a cherry in $\hat{T}$ but are not a cherry in $T$ (if one exists). Since $i,j$ are not a cherry in $T$, there is a path connecting them, which involves at least $2$ other nodes, by definition of a cherry. Denote the path by $P_{ij} = v_1 - v_2 - \cdots - v_\ell$ where $v_1 = i, v_l = j$ and $l \geq 4$. 
In the process of transforming $T$ into $\hat{T}$, we select one of the two nodes (according to some criterion), say it is $i$, and move it along the path $P_{ij}$ in $T$ until it becomes a cherry with $j$. The different steps of this movement are shown in Figure~\ref{fig:quartet-KLM}~(b), where we can see the different positions of $i$. 
When we are done moving $i$ towards $j$, we will find a new pair to make a cherry and so on. If $T$ and $\hat{T}$ agree on all cherries, 
we look for disagreements due to parents of cherries, grandparents of cherries, and so on.
When this process ends, $T$ and $\hat{T}$ are guaranteed to be the same.

\paragraph{Analyzing one step.}
Let us focus on one of the steps in the above process. While the first step does not change the topology over the leaves, we will analyze the second step, transforming $T^1$ into $T^2$ as shown in Figure~\ref{fig:quartet-KLM} (b).
This involves cutting $i$ from the middle of edge $(v_1,v_2)$ and pasting it in the middle of edge $(v_2,v_3)$. We will bound $|f_x^{T^1}(\alpha) - f_x^{T^2}(\alpha)|$ 
in terms of the parameters $\Delta_{ijk\ell}$ of the quartets that $T^1$ and $T^2$ disagree on: 
\begin{lemma}[Formal statement in Lemma~\ref{l:lip_expression}]\label{l:lipschitz_quartet}
Let $i,j$ be a cherry in $\hat{T}$ but not in $T$ and let $T^1$ and $T^2$ be the topologies defined by procedure above (depicted in Figure~\ref{fig:quartet-KLM} (b)). Also, denote by $U$ the set of quartets where $T^1$ and $T^2$ disagree on. 
Then,
\begin{equation} \label{eq:one-step}
\sum_{x \in \{-1,1\}^n} |f_x^{T^1}(\alpha) - f_x^{T^2}(\alpha)| \leq \sum_{\{i,k,\ell,m\}\in U} \Delta_{ik\ell m}
\end{equation}
\end{lemma}
Lemma~\ref{l:lipschitz_quartet} can be seen as an analogue of Lemma~\ref{l:lipschitz} but we have quartets instead of pairwise distances.
It shows a type of Lipschitzness of $f_x^{T^1}$ when we change the topology of some of the quartets.
The way to prove it is to notice that the difference $|f_x^{T^1}(\alpha) - f_x^{T^2}(\alpha)|$ can be factored into a sum of terms, each for a quartet in $U$. Each term is the expression of a probability distribution on a topology that results when we remove all paths between leaves in that quartet from the tree, as depicted in Figure~\ref{fig:path-removal}~(b).
Next, we will bound the right hand side of \eqref{eq:one-step}.
\begin{lemma}[Formal statement in Lemma~\ref{l:bad_quartet}]\label{l:bad_quartet_intro}
Each term $\Delta_{ik\ell m}$ in the right hand side of \eqref{eq:one-step} is bounded by $2\epsilon$.
\end{lemma}
\begin{proof}[Proof sketch]
Using a simple case analysis, it can be shown that the only quartets that can change topology are those that contain $i$, some $k \in K$, some $\ell \in L$ and some $m \in M \cup \{j\}$ (see Figure~\ref{fig:quartet-KLM}~(b)). The quartet topology is $\{(ik)(\ell m)\}$ in $T^1$ and $\{(im)(k \ell)\}$ in $T^2$ and we would like to argue that $\Delta_{ik\ell m} \le 2\epsilon$. First, let us assume that $m=j$. Then, the topology of $\{i,k,l,m\}$ in $T^1$ equals that of $T$, since $T$ and $T^1$ share the same leaf topology, while the quartet topology of $T^2$ equals that of $\hat{T}$: indeed, the topology of $\hat{T}$ is $\{(ij)(k\ell)\}$ since $(i,j)$ is a cherry in $\hat{T}$, as assumed in the algorithm and this is also the topology in $T^2$. In particular, since the topology in $T^1$ is different than that in $T^2$, then the topology in $T$ is different than that in $\hat{T}$. As explained after the definition of $\Delta_{ik\ell m}$ above, this implies that $\Delta_{ik\ell m} \le 2\epsilon$. In particular, this provides a bound on $\Delta_{ik\ell m}$ as required.
The case that $j \ne m$ is more complicated and it relies on the fact that $i$ can be selected such that $\Delta_{ik\ell m} \le \Delta_{jk\ell m}$ and cosidering the quartet $(j,k,\ell, m)$.
\end{proof}

\paragraph{Completing the proof}


After ensuring that the move described in Lemma~\ref{l:lipschitz_quartet} incurs a small loss, the natural next step is to repeatedly apply a variant of Lemma~\ref{l:lipschitz_quartet} for each other step of the sequence and obtain a bound for $|f_x^T(\alpha) - f_x^{\hat{T}}(\alpha)|$, using the triangle inequality.
We would like to analyze the total loss incurred in all the steps.
We divide these steps into rounds, where in the first round we move leaves, in the second round we move parents of leaves etc. The number of rounds is bounded by the diameter $D$ of $\hat{T}$. 
We can show that in each round, every quartet changes topology at most $4$ times. Furthermore, in the general variant of Lemma~\ref{l:bad_quartet_intro} that corresponds to a movement of a subtree (rather than a leaf), $\epsilon$ is replaced with $n\epsilon$.
Since there are at most $\binom{n}{4}$ quartets, the total loss incurred in TV during all these steps is $O(Dn^5\epsilon)$.

\section{Related Work} \label{sec:related}
A popular method for latent-tree estimation are \emph{tree-metric} approaches, which rely on estimating the pairwise correlations between any two leaves. For these algorithms, there is a vast theoretical analysis which is largely focused on estimating the structure of the tree, namely, finding the set of edges $E$
\cite{felsenstein1973maximum,chang1996full,erdHos1999few,huson1999disk,csuros2002fast,felsenstein2004inferring,king2003complexity,daskalakis2006optimal,roch2006short,mossel2007distorted,gronau2008fast,roch2010toward,roch2017phase}. The results typically require some upper and lower bounds on the edge-weights $a_{ij}$. Such bounds guarantee that the structure of the tree can be completely identified from polynomially many samples.
In contrast, \cite{daskalakis2009phylogenies} design an algorithm that reconstructs as much of the true topology as possible, without assuming bounds on the edge-weights. 
However, they do not provide any guarantee on the closeness of the learned distribution to the true distribution.
Another popular family of algorithms are \emph{likelihood-based} methods \cite{felsenstein1981evolutionary,yang1997paml,stamatakis2006raxml,lee2006comet,wang2006severity,truell2021maximum}, but their convergence guarantees are barely understood \citep{zwiernik2017maximum,daskalakis2018bootstrapping,daskalakis2022EM}.

Beyond trees, the general problem of latent graphical model estimation has received some attention~\citep{bresler2019learning,bresler2020learning,moitra2021learning,goel2019learning,goel2020boltzmann}. However, all these algorithms have time- and sample- complexity that is exponential in the maximum degree of the graph. 
Also, \cite{acharya2018learning} study testing of Bayesnets with latent variables, but under the assumption that the c-components have constant size.

Another related line of work is that of estimating a tree from fully observable data, while guaranteeing that the error is bounded in $k$-\emph{local-TV}: this means that the output model is $\epsilon$-close in total variation to the true model in any marginal of $k$ nodes (where $k$ is considered small). While the complexity of learning the full tree to $\epsilon$ total variation distance is $\Theta(n \log n/\epsilon^2)$ \citep{daskalakis2020tree,koehler2020note}, the algorithm of \cite{boix2022chow} has a sample complexity of $O(\log n \cdot k^2 2^{2k}/\epsilon^2)$ for learning in $k$-local TV. The preceding paper of \cite{bresler2020learning} obtained the same guarantee, however, they assume some upper bounds on the edge correlations $\theta_{ij}$.
\section{Algorithm}\label{s:algo}

\begin{algorithm}[ht]
\caption{Learn a tree-structured distribution with a \emph{known} topology}
\KwInp{An unweighted tree $T = (V,E)$ whose leaves are labeled $1,\dots,n$}
\KwInp{Estimates on the correlations between the leaves: $(\hat{\alpha}_{ij})_{i,j \in \{1,\dots,n\},\ i \ne j}$}
\KwInp{A parameter $\eta>0$ that bounds the errors of $\hat{\alpha}_{ij}$ in absolute value.}
$(w_{k\ell})_{k\ell \in E} \gets$ any solution for the following linear program:
\begin{align*}
    &\text{Find a feasible solution for: } (w_{k\ell})_{k\ell \in E} \\
    &\text{subject to}\\
    &\quad \forall \text{ leaves } i \ne j \colon \log(|\hat{\alpha}_{ij}| - \eta)
    \le \textstyle\sum_{(k,l) \in \mathrm{path}(i,j)} w_{kl}
    \le \log(|\hat{\alpha}_{ij}| + \eta)  \\
    &\qquad\text{(here $\log x$ is interpreted as $-\infty$ for $x\le 0$)}\\
    &\quad \forall \text{ edge } (k,\ell) \colon w_{k\ell} \le 0
\end{align*}
$U \gets \{(i,j): |\hat{\alpha}_{ij}| > \eta \}$\\
$(s_{k\ell})_{k\ell \in E} \gets$ any solution for the following linear system in $\F_2$:
\begin{align*}
\sum_{(k,\ell) \in P_{ij}} s_{k\ell} = sgn(\hat{\alpha}_{ij}) \text{ for all $(i,j) \in U$}
\end{align*}
\Return{\rm the weights $\theta_{k\ell} \gets (-1)^{s_{kl}}e^{w_{k\ell}}$ for all edges $(k,\ell)\in E$}
\label{alg:fixed}
\end{algorithm}

We describe the algorithms of Theorem~\ref{thm:alg} both for the case that the tree topology is known and when it is unknown, given $m$ samples $(x_1^1,\dots x_n^1), \dots, (x_1^m,\dots, x_n^m) \in \{-1,1\}^n$.
Both algorithms first estimate the covariance between any two leaves from samples, setting $\hat{\alpha}_{ij} = \frac{1}{m} \sum_{\ell=1}^m x_i^\ell x_j^\ell$. We note that by Chernoff-Hoeffding and a union bound, with high probability all the estimated correlations are close to their true values $\alpha^*_{ij}$:
\begin{equation}\label{eq:estimated-corr}
\text{With probability } 1-\delta, \quad
\forall i \ne j \colon 
|\hat{\alpha}_{ij} - \alpha^*_{ij}|
\le \eta :=
\sqrt{2 \log(n^2/\delta)/m}.
\end{equation}
Given such estimates on the covariance, our algorithms will find some weighted tree whose correlations $\alpha_{ij}$ between the leaves are close to the estimated correlations $\hat{\alpha}_{ij}$. By the triangle inequality, the correlations of the estimated tree are close to the true correlations and the result will follow by applying Theorem~\ref{thm:probabilistic}.


\paragraph{Known tree topology.}

We describe an algorithm that learns the weights of a fixed tree, given the estimated correlations $\hat{\alpha}_{ij}$. From \eqref{eq:estimated-corr}, we can assume that all these estimations are accurate up to an additive error of $\eta=\sqrt{2 \log(n^2/\delta)/m}$.

We start by assuming that the edge weights $\theta^*_{kl}$ are non-negative, which also implies that the pairwise correlations between the leaves $\alpha^*_{ij}$ are non-negative by \eqref{eq:path-correlation}. Such Ising models are called \emph{ferromagnetic}. Later, we will show how to use this algorithm to solve the unrestricted problem. 

Starting with a ferromagnetic (non-negative weight) model, we will show how to write a linear program that returns weights on the edges of this tree, such that the pairwise correlations $\alpha_{ij}$ over the leaves are $\eta$-close to $\hat{\alpha}_{ij}$, if such weights exist. Yet, a solution is guaranteed to exist since the true edge-weights satisfy these constraints.
This gives us a model whose pairwise correlations are $2\eta$ close to those of the true model:
\[
|\alpha_{ij} - \alpha^*_{ij}|
\le |\alpha_{ij} - \hat{\alpha}_{ij}| +
|\hat{\alpha}_{ij} - \alpha^*_{ij}|
\le \eta + \eta = 2\eta.
\]
First, we discuss the linear program that finds weights $\theta_{k\ell}\ge 0$ on the edges $(k,\ell)\in E$.
The variables of the linear program are $(w_{k\ell})_{(k,\ell)\in E}$ and they signify $w_{k\ell} = \log \theta_{k\ell}$. We would like our output to satisfy the following constraints: (1) $\theta_{k\ell} \in [0,1]$, which can be rewritten as $w_{k\ell} \le 0$; and (2) For any leaves $i,j$, $|\hat{\alpha}_{ij}| - \eta \le |\alpha_{ij}| \le |\hat{\alpha}_{ij}| + \eta$. If we take log and substitute $|\alpha_{ij}| = \prod_{(k,l) \in \mathrm{path}(i,j)}\theta_{k\ell}$ according to \eqref{eq:path-correlation}, we get the following linear constraints on the variables $w_{k\ell}$: 
\[
\text{For any leaves $i\ne j$}\colon \quad
\log(|\hat{\alpha}_{ij}| - \eta)
    \le \sum_{(k,l) \in \mathrm{path}(i,j)} w_{kl}
    \le \log(|\hat{\alpha}_{ij}| + \eta)
\]
(while using the convention $\log x = -\infty$ for $x\le 0$.)
This yields a linear program for finding the logarithms of the weights of the tree, and we can obtain weights for the tree by exponentiation of these log-values.

It now remains to handle the case where the edge weights $\theta^*_{kl}$ can be negative. 
We will first consider the tree whose edge weights are $|\theta^*_{kl}|$, which we call the \emph{ferromagnetic variant} of the original tree. Notice that by \eqref{eq:path-correlation}, the pairwise correlations in this tree equal $|\alpha^*_{ij}|$. In the first part of our algorithm, we will find a tree whose pairwise correlations are $2\eta$ close to those of the ferromagnetic variant of our tree. This can be done by feeding the linear program with the absolute values of the estimated correlations, $|\hat{\alpha}_{ij}|$. Notice that these are $\eta$-close to the true correlations in the ferromagnetic variant. Indeed, by the triangle inequality,
\[
||\hat{\alpha}_{ij}| - |\alpha^*_{ij}||
\le |\hat{\alpha}_{ij} - \alpha^*_{ij}|
\le \eta.
\]
Since the linear program is fed with the estimates  $|\hat{\alpha}_{ij}|$ that are $\eta$-close to the correlations of the ferromagnetic variant, by the guarantees of the linear program, it outputs a tree whose pairwise correlations are $2\eta$ close to those of the ferromagnetic variant.

It now remains to find a sign for each edge. Let $s_{k\ell}\in \{0,1\}$ be a variable for each edge $(k,\ell)$ so that the sign of edge $(k,\ell)$ is $(-1)^{s_{k\ell}}$. 
By the approximation guarantee $|\hat{\alpha}_{ij} - \alpha^*_{ij}|\le \eta$, we can accurately estimate the sign of $\alpha^*_{ij}$ whenever $|\hat{\alpha}_{ij}| > \eta$. 
Let $i,j$ be such a pair of leaves and denote the sign as $sgn(\hat{\alpha}_{ij}) \in \{-1,1\}$. By \eqref{eq:path-correlation},
$$
\prod_{(k,\ell) \in P_{ij}} (-1)^{s_{k\ell}} = sgn(\hat{\alpha}_{ij}) = sgn(\alpha^*_{ij}).
$$
This gives us a linear equation in $\F_2$ with variables $s_{k\ell}$ for all such pairs (i.e. there is a linear equation for each pair $i,j$ such that $\hat{\alpha}_{ij} > \eta$). We subsequently find a solution to this system to obtain signs for all edges. We note that there exists at least one solution since the true signs of our model is a solution, yet, the algorithm can output any solution. As a final output, the edge-weights $\theta_{kl}$ are determined such that their absolute values equal the output of the linear program over $\mathbb{R}$, while the signs of the edges are taken according to the system of equations over $\F_2$, as summarized in Algorithm~\ref{alg:fixed}.

We argue that the pairwise correlations in the output model, which we denote by $\alpha_{ij}$, are $4\eta$ close to the true correlations. Indeed, divide into cases:
\begin{itemize}
\item If $|\hat{\alpha}_{ij}| > \eta$, then, by the guarantee of the system of equations over $\F_2$, the sign of $\alpha_{ij}$ equals that of $\alpha^*_{ij}$. Their absolute values are $2\eta$-close by the guarantee of the linear program over $\mathbb{R}$. Hence, their values are $2\eta$-close.
\item If $|\hat{\alpha}_{ij}| \le \eta$, then, by the approximation assumption, we have that
\[
|\alpha^*_{ij}| \le |\hat{\alpha}_{ij}| + \eta \le 2\eta.
\]
Further, recall that by the guarnatee of the linear program over $\mathbb{R}$, we have that
\[
|\alpha_{ij}| \le |\hat{\alpha}_{ij}| + \eta \le 2\eta.
\]
We derive that
\[
|\alpha_{ij} - \alpha^*_{ij}|
\le |\alpha_{ij}| + |\alpha^*_{ij}|
\le 2\eta+2\eta=4\eta.
\]
\end{itemize}
According to Theorem~\ref{thm:probabilistic}, the output model is $O(n^2\eta)$-close in total variation distance. By the assumption on $\eta$ in \eqref{eq:estimated-corr}, the proof follows. See Section~\ref{sec:known-pr} for a more detailed proof.

\paragraph{Unknown tree topology.}

\begin{algorithm}[t]
\caption{Learn a tree-structured distribution with an \emph{unknown} topology}
\KwInp{Estimates on the correlations between the leaves: $(\hat{\alpha}_{ij})_{i,j \in \{1,\dots,n\},\ i \ne j}$}
\KwInp{Parameters $\eta', \xi, \delta > 0$.}
$F \gets $ the forest output by the algorithm of \cite{daskalakis2011evolutionary} given weights $(\hat{\alpha}_{ij})_{i\ne j}$ and parameters $\xi,\delta>0$ \\
\For{\rm Tree $T= (V,E) \in F$}{
    $(\theta_{k\ell})_{(k,\ell) \in E} \gets$ the
    output of Algorithm~\ref{alg:fixed} given $T$ and the correlations $\hat{\alpha}_{i,j}$ between any two leaves of $T$ (while substituting its parameter $\eta$ with $\eta'$)
}
\Return{\rm The forest $F$, where each $T=(V,E)$ is weighted according to $(\theta_{k\ell})_{k\ell \in E}$}
\label{alg:unkonwn-top}
\end{algorithm}

If we do not know the tree structure, we use the algorithm of \cite{daskalakis2011evolutionary} that, given approximations $\hat{\alpha}_{ij}$ of the correlations between the leaves, finds a forest that shares multiple properties with the original tree. Then, for any tree in this forest, we compute weights on the edges, using Algorithm~\ref{alg:fixed} and return the weighted forest, as summarized in Algorithm~\ref{alg:unkonwn-top}. 

To analyze this algorithm, we perform a series of careful contractions and deletions of edges, that transform this forest into one where each subtree has exactly the same topology as the one induced by the true tree on that particular subset of leaves. Then, crucially, we use the analysis for the known topology setting, to bound the difference in total variation between learned marginal distribution over the leaves of each subtree and the true marginal distribution.


To compare with \cite{cryan2001evolutionary}, they employ a similar process of splitting the tree into subtrees. However, they then rely on learning the weights within in each subtree accurately. This yields a bound in total variation between the learned and true distribution within each tree. The difference between these approaches is that we learn in total variation only the \emph{leaf distribution} within each subtree, while their analysis relies on learning the distribution on both the leaves and the \emph{internal nodes} of each subtree. This requires them to cut the original tree into significantly smaller subtrees, which harms the complexity.


\section{A Formula for the leaf distribution from \citep{bresler2020learning}}\label{s:distribution}
We first introduce some convenient notation. Let $S$ be a subset of the leaves of even cardinality. Then, there is a natural way to partition the leaves in $S$ in $|S|/2$ pairs using the following criterion: each leaf $i$ in $S$ is matched with its closest relative in $S$ in the tree. Yet, to be more exact, we say that a matching of $S$ is a \emph{closest relative matching} if for any two distinct pairs $(i,j)$ and $(k,\ell)$ in the matching, the path from $i$ to $j$ does not intersect the path from $k$ to $\ell$.
An example of such a matching is given in Figure~\ref{fig:matching}. 
In the following proposition, we prove that there is a unique such matching, and use this matching to find an expression to the leaf distribution of an Ising model:

\begin{lemma}\label{l:basic_extend}
Let $x_1,\dots,x_n$ denote the values over the $n$ leaves of a tree $T$ with pairwise correlations $\alpha \in [-1,1]^{n \choose 2}$. Then, the following holds:
\begin{itemize}
    \item Any subset $S \subseteq [n]$ of even cardinality has a unique closest relative matching.
    \item Define for any subset $S\subseteq [n]$ of even cardinality 
    $$
    \alpha_S := \prod_{k=1}^{|S|/2} \alpha_{i_k j_k}
    $$
    where $(i_1,j_1),\dots,(i_{|S|/2-1}, j_{|S|/2-1})$ are the pairs in the closest relative matching.
    Then, we have
    \begin{equation}\label{eq:dist-formula}
    \Pr[x_1,\ldots,x_n] = \frac{\sum_{\text{even subsets $S\subseteq[n]$}}\alpha_S \prod_{i \in S}x_i}{2^n}
    \end{equation}
\end{itemize} 
\end{lemma}

\begin{proof}
Recall that we can assume that there are no nodes of degree $2$ (otherwise, we can contract maximal paths of degree-$2$ nodes and replace them by a single edge whose weight is the product of all edges in the path. This does not change the leaf distribution.)
We will prove both statements together by induction. We will focus on the proof for the probability expression, and the uniqueness of the matching will come as a by-product. The base case considers either $0$, $1$ or $2$ leaves, and follows trivially. For the induction step, suppose the claim is true for all trees having at most $n-1$ leaves. Let $T$ be a tree with $n$ leaves. We are interested in the probability $\Pr[x_1,\ldots, x_n]$ of the leaves taking some specific values.
First of all, since $T$ is assumed to contain no nodes of degree $2$, we know that there exists at least one cherry, i.e. a pair of leaves that share their parent. Also, without loss of generality, suppose that leaves $n-1$ and $n$ form a cherry and denote by $p$ their common parent in the tree.
Lastly, denote by $\theta_{(p,n-1)}$ and $\theta_{(p,n)}$ the weights of the edges $(p,n-1)$ and $(p,n)$ respectively.
Then, we know that $x_n, x_{n-1}$ are conditionally independent from the rest of the tree conditioned on $y_p$, where $y_p$ denotes the value of node $p$. 
Thus, we can write
\begin{align*}
\Pr[x_1,\ldots, x_n] &= \Pr[x_1,\ldots, x_{n-2},y_p = 1] \Pr[x_{n-1},x_n\mid y_p = 1] \\
&+\Pr[x_1,\ldots, x_{n-2},y_p = -1] \Pr[x_{n-1},x_n\mid y_p = -1]
\end{align*}
Now, notice that we can view the nodes $1,\dots,n-2,p$ as the leaves of a tree $T'$ which is simply $T$ after deleting leaves $n-1$ and $n$ and the edges $(n-1,p)$ and $(n,p)$. Hence, we can apply the induction hypothesis on the distribution of $x_1,\ldots,x_{n-2},y_p$.
Recall that the expression for the probability distribution is a function of the expressions $\alpha_S$ of all the even subsets $S$ of the leaves. Hence, we would like to compare the coefficients $\alpha_S$ between the distribution over $x_1,\dots,x_n$ and the distribution over $x_1,\dots,x_{n-2},y_p$ that is used for the induction hypothesis. In order for such a comparison to be possible, we divide the collections of even subsets of the leaves of $T$ and $T'$ into categories. We start with the leaves of tree $T$.
Denote by $\mathcal{S}$ the set of all even subsets of the set of leaves $\{1,\dots,n\}$. Clearly, we can partition $\mathcal{S}$ into $4$ disjoint subsets:
\begin{align*}
&\mathcal{S}_{--} := \{\text{even subsets of $[n]$ not containing neither $n-1$ nor $n$}\}\\
&\mathcal{S}_{+-} := \{\text{even subsets of $[n]$  containing $n-1$ but not containing  $n$}\}\\
&\mathcal{S}_{-+} := \{\text{even subsets of $[n]$ not containing $n-1$ but containing $n$}\}\\
&\mathcal{S}_{++} := \{\text{even subsets of $[n]$  containing both $n-1$ and $n$}\}
\end{align*}
Similarly, we define analogues to be applied on the distribution that is used in the induction hypothesis. In particular, define by $\mathcal{R}$ the collection of all even subsets of $[n-2]\cup \{p\}$. 
We can also partition $\mathcal{R}$ into the following subsets:
\begin{align*}
&\mathcal{R}_{-} := \{\text{even subsets of $[n-2]\cup\{p\}$ not containing $p$}\}\\
&\mathcal{R}_{+} := \{\text{even subsets of $[n-2]\cup\{p\}$ containing $p$}\}
\end{align*}
While applying the induction hypothesis, once computing $\Pr[x_1,\dots,x_n]$ we will split the sum in \eqref{eq:dist-formula} into four sums over the different subsets of $\mathcal{S}$ that were defined above. Similarly, while computing $\Pr[x_1,\dots,x_{n-2},y_p]$ we will split the sum into terms corresponding to $\mathcal{R}_-$ and $\mathcal{R}_+$. In order to be able to compare between these two sums, we will map each of the four subsets of $\mathcal{S}$ to a subset of $\mathcal{R}$.

First, note that $\mathcal{S}_{--} = \mathcal{R}_{-}$ since both equal the collection of even subsets of $[n-2]$. Hence, it follows by induction hypothesis that the sets $S \in \mathcal{S}_{--}$ have a unique closest relative matching. Further,
\[
\sum_{S \in \mathcal{S}_{--}} \prod_{i\in S} x_i \alpha_S
= \sum_{S \in \mathcal{R}_{-}} \prod_{i\in S} x_i \alpha_S \enspace.
\]

Next, notice that there is a bijection between $\mathcal{S}_{++}$ and $\mathcal{R}_-$, which takes any $S \in \mathcal{S}_{++}$ to $S\setminus \{n-1,n\}$. We use this to prove that there exists a closest-relative matching for any $S \in \mathcal{S}_{++}$. Indeed, we can match $n-1$ with $n$ and then match $S \setminus \{n-1,n\}$. This is possible by induction hypothesis. Further, this matching is unique. This is true because any such matching must match $n-1$ with $n$, and the matching in $S \setminus \{n-1,n\}$ is unique by induction hypothesis. Using this bijection between the matchings of $\mathcal{S}_{++}$ and $\mathcal{R}_-$, we derive that
\[
\sum_{S \in \mathcal{S}_{++}} \prod_{i\in S} x_i \alpha_S
= x_{n-1}x_n \alpha_{n-1,n} \sum_{S \in \mathcal{R}_{-}} \prod_{i\in S} x_i \alpha_S \enspace.
\]

Further, notice that there is a bijection between $\mathcal{S}_{+-}$ which takes $S \in \mathcal{S}_{+-}$ to $S \setminus \{n-1\}\cup \{p\}$. Further, to argue that each $S \in \mathcal{S}_{+-}$ contains a unique closest-relative matching, it is easy to see that there is a one-to-one correspondence between the closest-relative matchings of $S\setminus\{n-1\}\cup \{p\}$ and that of $S$. Indeed, for any closest-relative matching of $S\setminus\{n-1\}\cup \{p\}$, we can obtain a closest relative matching of $S$ by replacing $p$ with $n-1$, and vice versa. Notice that the path from $n-1$ to any other vertex can be obtained from the path from $p$ to that vertex by adding the edge $(n-1,p)$ in the beginning. Hence, for any leaf $i$ we have $\alpha_{n-1,i} = \theta_{n-1,p}\alpha_{p,i}$. In particular, this implies that $\alpha_{S} = \theta_{n-1,p}\alpha_{S\setminus\{n-1\}\cup \{p\}}$. Summing over $S \in \mathcal{S}_{+-}$, we get
\[
\sum_{S \in \mathcal{S}_{+-}} \prod_{i\in S} x_i \alpha_S
= x_{n-1} \theta_{n-1,p} \sum_{S \in \mathcal{R}_{+}} \prod_{i\in S\setminus\{p\}} x_i \alpha_S \enspace.
\]
Similarly, the sets $S \in \mathcal{S}_{-+}$ also have a unique closest-relative matching, and 
\[
\sum_{S \in \mathcal{S}_{-+}} \prod_{i\in S} x_i \alpha_S
= x_{n} \theta_{n,p} \sum_{S \in \mathcal{R}_{+}} \prod_{i\in S\setminus\{p\}} x_i \alpha_S \enspace.
\]
Using the above expressions, we can complete the proof, using the induction hypothesis:
\begin{align*}
&\Pr[x_1,\ldots, x_n] = \Pr[x_1,\ldots, x_{n-2},y_p = 1] \Pr[x_{n-1},x_n|y_p = 1]\\ 
&\qquad \qquad+\Pr[x_1,\ldots, x_{n-2},y_p = -1] \Pr[x_{n-1},x_n|y_p = -1]\\
&= \frac{\sum_{S \in \mathcal{R}_{-}} \prod_{i\in S}x_i\alpha_S + \sum_{S \in \mathcal{R}_+} \prod_{i\in S\setminus \{p\}}x_i \alpha_S}{2^{n-1}} \frac{1 + x_n \theta_{n,p}}{2}\frac{1 + x_{n-1} \theta_{n-1,p}}{2}   \\
&\quad\quad + 
 \frac{\sum_{S \in \mathcal{R}_-} \prod_{i\in S}x_i\alpha_S - \sum_{S \in \mathcal{R}_+} \prod_{i\in S\setminus \{p\}}x_i\alpha_S}{2^{n-1}} \frac{1 - x_n \theta_{n,p}}{2}\frac{1 - x_{n-1} \theta_{n-1,p}}{2} \\
 &= \frac{\sum_{S \in \mathcal{R}_-} \prod_{i\in S}x_i\alpha_S + x_nx_{n-1} \theta_{n,p} \theta_{n-1,p} \sum_{S \in \mathcal{R}_-} \prod_{i\in S}x_i\alpha_S}{2^{n}} \\
 &\quad\quad + 
 \frac{ x_n \theta_{n,p} \sum_{S \in \mathcal{R}_+} \prod_{i\in S\setminus \{p\}}x_i\alpha_S + x_{n-1} \theta_{n-1,p} \sum_{S \in \mathcal{R}_+} \prod_{i\in S\setminus \{p\}}x_i\alpha_S}{2^{n}}\\
 &= \frac{\sum_{S \in \mathcal{S}_{--}}\prod_{i \in S}x_i\alpha_S + \sum_{S \in \mathcal{S}_{++}}\prod_{i \in S}x_i\alpha_S +
 \sum_{S \in \mathcal{S}_{-+}}\prod_{i \in S}x_i\alpha_S +
 \sum_{S \in \mathcal{S}_{+-}}\prod_{i \in S}x_i\alpha_S}{2^n}\\
 &= \frac{\sum_{S \in \mathcal{S}}\prod_{i \in S}x_i\alpha_S}{2^n}
\end{align*}

\end{proof}
\section{Proof of Theorem~\ref{thm:probabilistic} (Same topology) and Theorem~\ref{thm:alg} (Known topology)} \label{sec:known-pr}

In this Section, we provide the proof for bounding the TV distance between two models with the same topology (Theorem~\ref{thm:probabilistic}). This argument immediately implies an algorithm for TV-learning using $O(n^4/\epsilon^2)$ samples from the leaves.

We first restate Theorem~\ref{thm:probabilistic} for the same topology with a bit more detail. 

\begin{theorem}\label{thm:probabilistic_fixed}
Let $T$ be a tree and $\alpha,\hat{\alpha} \in [-1,1]^{n \choose 2}$ be two tree metrics on $T$.
Suppose $\|\alpha - \hat{\alpha}\|_\infty \leq \epsilon$, for some $\epsilon > 0$. 
Let $\mu,\hat{\mu}$ be the corresponding distribution on the leaves of $T$ with metric $\alpha,\hat{\alpha}$ respectively. 
Then,
$$
TV(\mu,\hat{\mu}) \leq 2n^2 \epsilon
$$
\end{theorem}

To start, let $T$ be a tree with $n$ leaves. We will refer to the leaf set as $[n]$, so each number corresponds to one leaf. We define for each $x \in \{-1,1\}^n$ and for each tree topology $T$ a function $f_x^T:[-1,1]^{n\choose 2} \mapsto \R$ as
\begin{equation} \label{def:f}
f_x^T(\alpha) := 
 \frac{\sum_{\text{even subsets $S\subseteq[n]$}}\alpha^T_S\prod_{i \in S}x_i}{2^n}
\end{equation}
Notice the similarity of this expression with the probability distribution of the leaves. However, this is a multilinear function that is defined for any vector $\alpha \in [-1,1]^{n\choose 2}$, which might not necessarily arise from a tree metric on the leaves.
This motivates the following definition. 
For a vector $\alpha \in [-1,1]^{n \choose 2}$, we say that $\alpha$ is \emph{induced} by a metric in $T$ if there exists an assignment $\theta_e$ of weights for each edge $e$ of $T$, such that
for all leaves $i,j$
$$
\alpha_{ij} = \prod_{e \in P_{ij}} \theta_e
$$
In that case, we will refer to $\alpha$ as a \emph{tree metric} on $T$. 
Now, bounding $TV(\mu,\hat{\mu})$ essentially amounts to bounding 
$$
\sum_{x \in \{-1,1\}^n} |\mu(x) - \hat{\mu}(x)| = 
\sum_{x \in \{-1,1\}^n} |f_x(\alpha) - f_x(\hat{\alpha})|
$$
Thus, the problem amounts to bounding the Lipschitzness of $f_x$. 
We will bound this quantity by substituting one by one the coordinates of $\alpha$ with $\hat{\alpha}$. 
We first introduce some relevant definitions. 
We will need a total ordering of the pairs $(i,j)$ of leaves. The precise ordering doesn't matter, but for simplicity let's say we pick the lexicographic order. This means that $(i,j) < (k,l)$ if and only if $i< k$ or $i = k $ and $j < l$. We use the notation $(i,j) \leq (k,l)$ as a substitute for $(i,j) < (k,l)$ or $(i,j) = (k,l)$. 
Suppose we order all pairs of leaves in lexicographic order. Then, we denote the $t$-th pair in this order as $(i_t,j_t)$.
For each $0 \leq t \leq {n \choose 2}$, we define the vector $\alpha^{t} \in [-1,1]^{n \choose 2}$
as 
$$
\alpha^{t}_{kl} = \left\{
\begin{array}{ll}
      \hat \alpha_{kl} \text{ , if $(k,l) \leq (i_t,j_t)$}\\
      \alpha_{kl} \text{ , otherwise}
\end{array} 
\right.
$$
For $t = 0$, the convention is that $\alpha^t = \alpha$. Notice that $\alpha^{n \choose 2} = \hat{\alpha}$.
Also, denote $T\setminus \{i,j\}$ the topology that is obtained from $T$ if we remove all edges on the path $P_{ij}$ from $T$.

We first prove a Lemma about what happens to the expression of $f_x^T(\alpha)$ if we change exactly one coordinate of $\alpha$. 
This is a purely combinatorial statement that relies on the structure of the coefficients $\alpha_S$. 

\begin{lemma}\label{l:difference_fixed}
Let $T$ be a tree and $i,j$ two leaves of $T$. Also, let $\alpha,\beta \in [-1,1]^{n\choose 2}$ such that $\alpha_{kl} = \beta_{kl}$ if $(k,l) \neq (i,j)$. Let $\gamma \in [-1,1]^{n \choose 2}$ be defined as follows
$$
\gamma_{kl} = \left\{
\begin{array}{ll}
       0 \text{ , if $P_{ij}$ and $P_{kl}$ have common edges }\\
      \alpha_{kl} \text{ , otherwise}
\end{array} 
\right.
$$
Then, 
\begin{equation}\label{eq:diff_fixed}
f_x^T(\alpha) - f_x^T(\beta) = x_ix_j (\alpha_{ij} - \beta_{ij}) f_x^{T \setminus\{i,j\}}(\gamma)
\end{equation}

\end{lemma}

\begin{proof}
We have
$$
f_x^T(\alpha) - f_x^T(\beta) =  \frac{\sum_{\text{even subsets $S\subseteq[n]$}}\lp(\alpha^T_S - \beta^T_S\rp)\prod_{i \in S}x_i}{2^n}
$$
We first notice that if $\{i,j\}$ is not a subset of $S$, then $\alpha_{ij},\beta_{ij}$ will not appear in $\alpha_S^T,\beta_S^T$ respectively. This means that $\alpha_S^T = \beta_S^T$, since $\alpha,\beta$ agree on the rest
of the coordinates.
Hence, we focus on the collection of subsets
$$
\mathcal{S}_1 := \{S \subseteq [n]: \{i,j\} \subseteq S \text{ and $|S|$ even}\} 
$$
Suppose $v_0 = i , v_1,\ldots, v_k = j$ is the path connecting $i,j$ in $T$. 
Since each non-leaf node has degree $3$, each one of the nodes $v_1,\ldots, v_{k-1}$ has exactly one other neighbor outside of the path. We can view this as each $v_i$ being the root of some subtree $T_i$ that starts from the neighbor of $v_i$ that is outside of the path. 
Let $A_i$ be the set of leaves on subtree $T_i$. 
Notice that the sets $A_i$  partition $[n]\setminus \{i,j\}$. Figure~\ref{fig:path}  depicts these sets of leaves along the path.
\begin{figure}[htbp]
\centering
\scalebox{0.7}{\resizebox{\linewidth}{!}{%
\begin{tikzpicture}
%
%
\node[draw, circle, minimum size=25pt, inner sep=2pt] at (0,0) (v0) {\small $v_0$};
\node[above=5pt of v0] {\small $v_0 = i$};
\node[draw, circle, minimum size=25pt, inner sep=2pt, right=50pt of v0] (v1) {\small $v_1$};
\node[draw, circle, minimum size=25pt, inner sep=2pt, right=50pt of v1] (v2) {\small $v_2$};
\node[draw, circle, minimum size=25pt, inner sep=2pt, right=100pt of v2] (v3) {\small $v_{k-1}$};
\node[draw, circle, minimum size=25pt, inner sep=2pt, right=50pt of v3] (v4) {\small $v_k$};
\node[above=5pt of v4] {\small $v_k = j$};

\draw[thick] (v0) -- (v1);
\draw[thick] (v1) -- (v2);
\draw[thick] (v2) -- node[midway, fill=white] {$\ldots$} (v3);
\draw[thick] (v3) -- (v4);

%
%

\node[draw, white, circle, minimum size=15pt, inner sep=2pt, below=50pt of v1] (S1) {};
\node[draw, white, circle, minimum size=15pt, inner sep=2pt, below=50pt of v2] (S2) {};
\node[draw, white, circle, minimum size=15pt, inner sep=2pt, below=50pt of v3] (S3) {};
\node[fit=(S1), draw, inner sep=5pt, isosceles triangle, rotate=90, anchor=center] (S1-tree) {};
\node[fit=(S2), draw, inner sep=5pt, isosceles triangle, rotate=90, anchor=center] (S2-tree) {};
\node[fit=(S3), draw, inner sep=5pt, isosceles triangle, rotate=90, anchor=center] (S3-tree) {};

\draw[thick] (v1) -- (S1-tree);
\draw[thick] (v2) -- (S2-tree);
\draw[thick] (v3) -- (S3-tree);

%
%
\node[below=5pt of S1] {\small $A_1$};
\node[below=5pt of S2] {\small $A_2$};
\node[below=5pt of S3] {\small $A_{k-1}$};
\node[above left=5pt of S1, yshift=10pt] {\small $T_1$};
\node[above left=5pt of S2, yshift=10pt] {\small $T_2$};
\node[above left=5pt of S3, yshift=10pt] {\small $T_{k-1}$};
\node[] at ($(v2)!0.5!(v3) + (0,-2)$) {$\ldots$};

\end{tikzpicture}
}}
\caption{The figure shows the path connecting $i$ to $j$ and the subtrees that will become connected components if we remove this path from the graph.}
\label{fig:path}
\end{figure}
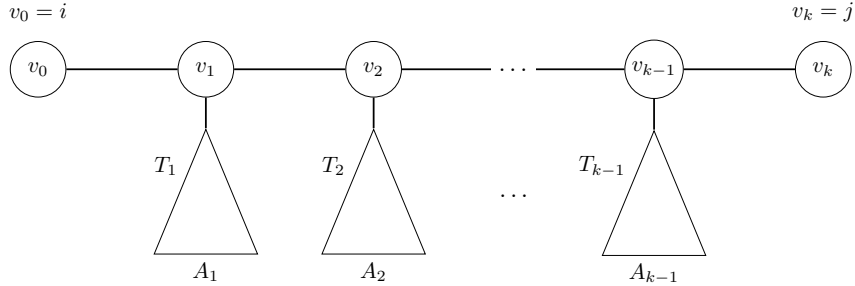

We now want to determine which elements of the family $\mathcal{S}_1$ of subsets gives different coefficients for $\alpha,\beta$. 
We first show that if for some $S \in \mathcal{S}_1$ we have $|S\cap A_r|$ being odd for some $0 <r< k$, then $\alpha_S^T = \beta_S^T$. 
The reason is the following: suppose there exists $r_0$ with $|S\cap A_{r_0}|$ being odd. 
Suppose also without loss of generality that this is the smallest $r$ for which this property holds. This means that for $r < r_0$ we have $|S \cap A_{r}|$ is even. Thus, by the matching process described in Section~\ref{s:distribution}, it is clear that for each $r<r_0$, the leaves in $S\cap A_r$ will be matched in pairs inside the tree $T_r$ and not with some leaves outside of the tree. Also, since $|S\cap A_{r_0}|$ is odd, the leaves in $S\cap A_{r_0}$ will be matched with 
each other, except one leaf, call it $w$, which will be left unmatched. Then, the matching process dictates that $w$ should be matched with $i$. Hence, $i$ will not be matched with $j$ for this subset $S$. An example of this situation can be seen in Figure~\ref{fig:matching}. 

\begin{figure}[htbp]
\centering
\scalebox{0.5}{\resizebox{\linewidth}{!}{%
\begin{tikzpicture}
\node[draw, circle, minimum size=15pt, inner sep=2pt] at (0,0) (v2) {\small $v_2$};
\node[draw, circle, minimum size=15pt, inner sep=2pt] at ($(v2) + (2,0)$) (j) {\small $j$};
\node[draw, circle, minimum size=15pt, inner sep=2pt] at ($(v2) + (-3,0)$) (v1) {\small $v_1$};
\node[draw, circle, minimum size=15pt, inner sep=2pt] at ($(v1) + (0,-1)$) (u) {};
\node[draw, circle, minimum size=15pt, inner sep=2pt] at ($(v1) + (-1,0)$) (i) {$i$};
\node[draw, circle, minimum size=15pt, inner sep=2pt] at ($(u) + (-1,-1)$) (l1) {\small $1$};
\node[draw, circle, minimum size=15pt, inner sep=2pt] at ($(u) + (1,-1)$) (l2) {\small $2$};
\node[draw, circle, minimum size=15pt, inner sep=2pt] at ($(v2) + (0,-1)$) (v4) {};
\node[draw, circle, minimum size=15pt, inner sep=2pt] at ($(v4) + (-1,-1)$) (l3) {\small $3$};
\node[draw, circle, minimum size=15pt, inner sep=2pt] at ($(v4) + (1,-1)$) (v5) {};
\node[draw, circle, minimum size=15pt, inner sep=2pt] at ($(v5) + (-1,-1)$) (l4) {\small $4$};
\node[draw, circle, minimum size=15pt, inner sep=2pt] at ($(v5) + (1,-1)$) (l5) {\small $5$};

\draw[thick] (v1) -- (u);
\draw[thick, blue] (v1) -- (v2);
\draw[thick, blue] (v2) -- (j);
\draw[thick, blue] (v1) -- (i);
\draw[thick, red] (u) -- (l1);
\draw[thick, red] (u) -- (l2);
\draw[thick] (v2) -- (v4);
\draw[thick, orange] (v4) -- (v5);
\draw[thick, orange] (v4) -- (l3);
\draw[thick, orange] (v5) -- (l4);
\draw[thick] (v5) -- (l5);

\draw[decorate, decoration={mirror,brace,raise=5pt,amplitude=5pt}] ($(l1 |- l5) + (-0.25,-0.5)$) -- ($(l2 |- l5) + (0.25,-0.5)$) node[midway,yshift=-20pt] {$A_1$};
\draw[decorate, decoration={mirror,brace,raise=5pt,amplitude=5pt}] ($(l3 |- l5) + (-0.25,-0.5)$) -- ($(l5) + (0.25,-0.5)$) node[midway,yshift=-20pt] {$A_2$};
\end{tikzpicture}
}}
\caption{In this example, we have the subset $S = \{i,j,1,2,3,4\}$. Notice that $|S\cap A_1|, |S\cap A_2|$ are even. Clearly, the closest relative matching is $(i,j),(1,2),(3,4)$. If we had the set $S' = \{i,1,3,j\}$, then the matching would be $(i,1),(3,j)$.}
\label{fig:matching}
\end{figure}
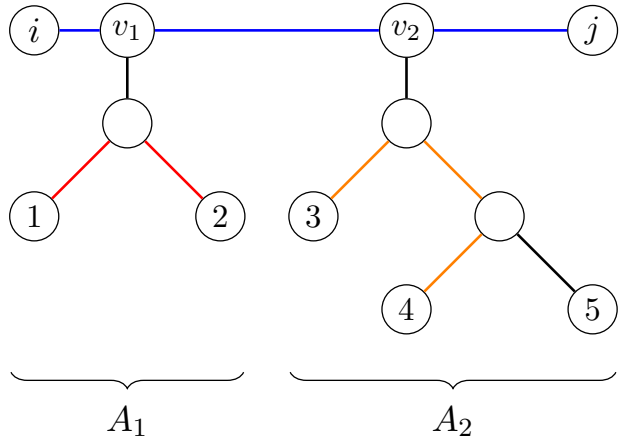

Hence, in $\alpha_S^T$ we will have the factor $\alpha_{iw}$ instead of $\alpha_{ij}$. This means that $\alpha_{ij}$ does not appear in $\alpha_S^T$ and similarly $\beta_{ij}$ does not appear in $\beta_S^T$. But this implies that $\alpha_S^T = \beta_S^T$, since $\alpha,\beta$ agree on all the other coordinates.
This proves our claim.

Hence, if we define the set
$$
\mathcal{S}_2 := \{\{i,j\} \cup\lp(\cup_{r=1}^{k-1} S_r\rp) : S_r \subseteq A_r , |S_r| \text{ even, for all $r$}\}
$$
then
$$
f_x^T(\alpha) - f_x^T(\beta) =  \frac{1}{2^n}\sum_{S \in \mathcal{S}_2}\lp(\alpha^T_S - \beta^T_S\rp)\prod_{i \in S}x_i
$$
Since $S_r$ has an even number of leaves, they will be matched inside tree $T_r$, regardless of the topology of the rest of the tree. This leaves $i,j$, which will be mathced together.
This enables us to write 
$$
\alpha_S = \alpha_{ij} \prod_{r=1}^{k-1} \alpha_{S_r}^{T_r} \quad , \quad \beta_S = \beta_{ij} \prod_{r=1}^{k-1} \alpha_{S_r}^{T_r}
$$
The reason we wrote $\alpha_{S_r}$ in the expression of $\beta_S$ is that $\alpha,\beta$ agree on all coordinates other than $ij$. 
Hence,
\begin{align*}
f_x^T(\alpha) - f_x^T(\beta) &=  (\alpha_{ij} - \beta_{ij}) x_ix_j
\lp(\frac{1}{4}\prod_{r=1}^{k-1} \underbrace{\frac{1}{2^{|A_r|}}\sum_{S_r \subseteq A_r, |S_r| \text{even}}\alpha_{S_r}^{T_r} \prod_{u \in S_r} x_u}_{f_{x_{A_r}}^{T_r}(\alpha)}\rp)\\
&= (\alpha_{ij}- \beta_{ij})x_ix_j f_x^{T\setminus\{i,j\}}(\gamma)
\end{align*}
Let us explain the last equality.
The $r$-th element in this product corresponds to the expression of the probability distribution on the leaves $A_r$ of the subtree $T_r$, with pairwise correlations that are given by $\alpha$ (we are slightly abusing notation when we pass $\alpha$ as an argument in $f_x^{T_r}$, since it is the vector of pairwise correlations for the whole tree). 
Hence, the product of these terms is the expression of a product probability distribution over the subsets $A_r$ and $i,j$ being independent from everyone else (this is the $1/4$ term). 
This is exactly the expression of the distribution on $T\setminus\{i,j\}$ with correlations given by $\alpha$, except for pairs of leaves that belong to different subtrees, which have correlation $0$.
This is exactly how we defined $\gamma$, hence the result follows.

\end{proof}

Lemma~\ref{l:difference_fixed} tells us how much $f_x$ changes when we change one coordinate, corresponding to some pair $(i,j)$. Hence, our efforts will now be focused on bounding the expression on the RHS of \eqref{eq:diff_fixed}. If this expression corresponds to a distribution on $T\setminus\{i,j\}$, then this term can easily be bounded. However, as we will see, that will not always be the case and we need to be more careful. The following Lemma contains the Lipschitzness property that we would like to prove. 

\begin{lemma}\label{l:fixed_induction}
Let $T$ be a tree and $\alpha,\hat{\alpha} \in [-1,1]^{n\choose 2}$, where $\alpha$ is a metric on $T$. Suppose that $\|\alpha - \hat{\alpha}\|_\infty\leq \epsilon/(2n^2)$, where $\epsilon \in (0,1)$.  
Then, for any $t \geq 1$
\begin{equation}\label{eq:lip_fixed}
\sum_{x \in \{-1,1\}^n} \lp|f_x^T(\alpha^{t}) - f_x^T(\alpha^{t-1}) \rp| \leq \frac{\epsilon}{n^2}
\end{equation}
\end{lemma}
\begin{proof}
We will prove \eqref{eq:lip_fixed} via induction on $t$. 
First, we define for all $s,t \leq {n \choose 2}$ the vector
$$
\gamma^{t,s}_{kl} = \left\{
\begin{array}{ll}
       0 \text{ , if $P_{i_tj_t}$ and $P_{kl}$ have common edges }\\
      \alpha^{s}_{kl} \text{ , otherwise}
\end{array} 
\right.
$$
The base case $t=1$ corresponds to the pair of leaves $(1,2)$. Since $\alpha^{1}$ and $\alpha^0 = \alpha$ differ only in the pair $(1,2)$, by Lemma~\ref{l:difference_fixed} we have

\begin{align*}
\sum_{x \in \{-1,1\}^n} \lp|f_x^T(\alpha^{1}) - f_x^T(\alpha) \rp| \leq |\alpha_{12} - \hat{\alpha}_{12}| \sum_{x \in \{-1,1\}^n} \lp| f_x^{T\setminus\{1,2\}}(\gamma^{1,0})\rp|
\end{align*}
Now, we notice that $f_x^{T\setminus\{1,2\}}(\gamma^{1,0})$ is actually the probability distribution on $T$ that results from $\alpha$ if we set $\theta_e = 0$ for all $e \in P_{12}$. Hence, we can remove the absolute value and get

\begin{align*}
\sum_{x \in \{-1,1\}^n} \lp|f_x^T(\alpha^{1}) - f_x^T(\alpha) \rp| \leq |\alpha_{12} - \hat{\alpha}_{12}| \sum_{x \in \{-1,1\}^n}  f_x^{T\setminus\{1,2\}}(\gamma^{1,0}) = |\alpha_{12} - \hat{\alpha}_{12}| \leq \frac{\epsilon}{2n^2}
\end{align*}
Hence, the base case of the induction is proven.

Now, suppose we have proved the claim for all $t' < t$. 

By applying Lemma~\ref{l:difference_fixed}, we again obtain
\begin{align}\label{eq:t-step-diff}
\sum_{x \in \{-1,1\}^n} \lp|f_x^T(\alpha^{t}) - f_x^T(\alpha^{t-1}) \rp| \leq |\alpha_{i_tj_t} - \hat{\alpha}_{i_tj_t}| \sum_{x \in \{-1,1\}^n}  \lp|f_x^{T\setminus\{i_t,j_t\}}(\gamma^{t,t-1})\rp| 
\end{align}
Now, the problem is that  $\gamma^{t,s}$ contains some coordinates of $\alpha$ and some coordinates of $\hat{\alpha}$. As a result, the expression $f_x^T(\gamma^{t})$ is not necessarily a probability distribution anymore.
We will try to relate this quantity to a true distribution. 

Essentially, $\gamma^{t,s}$ is the same as $\alpha^s$, except for pairs of leaves that belong to different components of $T \setminus\{i_t,j_t\}$. 
Now, we can write

\begin{align*}
&\sum_{x \in \{-1,1\}^n}  \lp|f_x^{T\setminus\{i_t,j_t\}}(\gamma^{t,s})\rp|\\ &\leq
\sum_{x \in \{-1,1\}^n} \lp|f_x^{T\setminus\{i_t,j_t\}}(\gamma^{t,0})\rp| + \sum_{s = 1}^{t-1} \sum_{x \in \{-1,1\}^n} \lp|f_x^{T\setminus\{i_t,j_t\}}(\gamma^{t,s}) - f_x^{T\setminus\{i_t,j_t\}}(\gamma^{t,s-1}) \rp| 
\end{align*}
First of all, we notice that the expression $f_x^{T\setminus\{i_t,j_t\}}(\gamma^{t,0})$ is the expression of a probability distribution on $T$, which is obtained from $\alpha$ by setting $\theta_e = 0$ for all edges $e \in P_{i_tj_t}$. Hence, the first term of the RHS sums up to $1$.
As for the second sum, each term of the outer sum has exactly the form of \eqref{eq:lip_fixed}, where the starting $\alpha$ is $\gamma^{t,0}$ and we have substituted at most $t-1$ with $\hat{\alpha}$, since $s \leq t-1$. Hence, we can apply the inductive assumption to get
$$
\sum_{s = 1}^{t-1} \sum_{x \in \{-1,1\}^n} \lp|f_x^{T\setminus\{i_t,j_t\}}(\gamma^{t,s}) - f_x^{T\setminus\{i_t,j_t\}}(\gamma^{t,s-1}) \rp| \leq (t-1) \frac{\epsilon}{n^2} \leq \epsilon
$$
since $t \leq {n \choose 2} \leq n^2$. 
Overall, this gives
$$
\sum_{x \in \{-1,1\}^n}  \lp|f_x^{T\setminus\{i_t,j_t\}}(\gamma^{t,s})\rp| \leq 1 + \epsilon
$$
Now, plugging this in \eqref{eq:t-step-diff} gives us
$$
\sum_{x \in \{-1,1\}^n} \lp|f_x^T(\alpha^{t}) - f_x^T(\alpha^{t-1}) \rp|  \leq |\alpha_{i_tj_t} - \hat{\alpha}_{i_tj_t}|(1 + \epsilon) \leq \frac{\epsilon}{n^2} 
$$
since $\epsilon < 1$. Thus, the inductive step is complete and the claim is proved.

\end{proof}

We are now ready to prove Theorem~\ref{thm:probabilistic_fixed}.

\begin{proof}[Proof of Theorem~\ref{thm:probabilistic_fixed}]
Let $\epsilon' = 2n^2\epsilon$. We divide into cases.

\emph{Case 1:} Suppose $\epsilon' < 1$. Then, Lemma~\ref{l:fixed_induction} applies and we get
\begin{align*}
TV(\mu,\hat{\mu}) &= \sum_{x \in \{-1,1\}^n}
\lp|f_x^T(\alpha) - f_x^T(\hat{\alpha})\rp| \leq \sum_{t = 1}^{n \choose 2} 
\sum_{x \in \{-1,1\}^n} \lp|f_x^T(\alpha^t) - f_x^T(\alpha^{t-1})\rp|\\
&\leq \sum_{t = 1}^{n \choose 2} \frac{\epsilon'}{n^2} \leq \epsilon' = 2n^2\epsilon
\end{align*}

\emph{Case 2:} Suppose $\epsilon' \geq 1$. This means that 
$$
TV(\mu,\hat{\mu}) \leq 1 \leq 2n^2\epsilon
$$
so the claim is trivial in that case. 

\end{proof}

We can now also conclude the proof of Theorem~\ref{thm:alg} for known topology, which we sketched in earlier Sections.

\begin{proof}[Proof of Theorem~\ref{thm:alg} (Known topology)]
Let $\alpha \in [-1,1]^n$ denote the vector of correlations of leaves in the model we are trying to learn and $\mu$ denote the distribution on the leaves for this model.
Let us consider the sample mean obtained from $m$ independent samples $x^{(1)},\ldots,x^{(m)}$.
$$
\hat{\alpha}_{ij} = \sum_{k=1}^m \frac{x^{(k)}_i x^{(k)}_j}{m}
$$

By standard Chernoff bounds, we know that with probability at least $1-\delta$, for all leaves $i,j$
\begin{equation}\label{eq:chernoff}
|\alpha_{ij} - \hat{\alpha}_{ij}| \leq \sqrt{\frac{2\log(n^2/\delta)}{m}} := \eta
\end{equation}

We run Algorithm~\ref{alg:fixed} with this $\eta$ parameter.

First, we claim that the LP that Algorithm~\ref{alg:fixed} solves has a feasible solution for this choice of $\eta$. 
To show that, we will construct a feasible solution of the program. If $\theta \in [-1,1]^{|E|}$ is the vector of the edge weights of the model we are trying to learn,
then for all leaves $i,j$ 
$$
\alpha_{ij} = \prod_{(k,l) \in P_{ij}} \theta_{kl}
$$
This implies that
$$
|\alpha_{ij}| = \prod_{(k,l) \in P_{ij}} |\theta_{kl}|
$$

Thus, if we set $w_{kl} = \ln |\theta_{kl}|$ for all edges $(k,l)$, we have that
$$
\sum_{(k,l) \in P_{i,j}} w_{kl} =  \ln \lp(\prod_{(k,l) \in P_{ij}} |\theta_{kl}|\rp) = \log |\alpha_{ij}|
$$
We know that with probability at least $1-\delta$
$$
||\alpha_{ij}| - |\hat{\alpha}_{ij}|| \leq |\alpha_{ij} - \hat{\alpha}_{ij}| \leq \eta
$$
which implies by the previous observations that
$$
\log (|\hat{\alpha}_{ij}| - \eta ) \leq \sum_{(k,l) \in P_{i,j}} w_{kl} \leq \log (|\hat{\alpha}_{ij}| + \eta )
$$

Hence, the inequality constraint for feasibility is satisfied.
Hence, this is a feasible solution for the program.

Let $\tilde{\theta}_{kl}$ be the edge weights that are returned by the LP and $\tilde{\alpha} \in [0,1]^{n \choose 2}$ be the pairwise correlations that are induced by these weights. We need to figure out the correct signs for each $\theta_{kl}$. Let $s_{kl} \in \{-1,1\}$ be a sign variable for each edge $(k,l)$. 
Also, let $s(\alpha)_{ij} \in \{-1,1\}$ be the sign of $\alpha_{ij}$ and likewise define $s(\alpha_{ij})$. 
If we find an assignment of the $s_{kl}$ variables such that
$$
\prod_{(k,l) \in P_{ij}} s_{kl} = s(\alpha_{ij})
$$
for all pairs $i,j$,
then it follows that
\begin{equation}\label{eq:correct_sign}
\lp|\prod_{(k,l)\in P_{ij}} s_{kl}\tilde{\theta_{kl}} - \alpha_{ij}\rp| = 
\lp|\prod_{(k,l)\in P_{ij}} s_{kl} \prod_{(k,l)\in P_{ij}} \tilde{\theta_{kl}} - s(\alpha_{ij})|\alpha_{ij}|\rp|
= \lp|\prod_{(k,l)\in P_{ij}} \tilde{\theta_{kl}} - |\alpha_{ij}|\rp| \leq \eta 
\end{equation}
Thus, we will now focus on finding such $s_{kl}$ and the output of the algorithm will be $\overline{\theta_{kl}} = s_{kl}\tilde{\theta_{kl}}$.
First of all, we need a way to figure out $s(\alpha_{ij})$ for all $i,j$. Let $U= \{(i,j) : |\hat{\alpha}_{ij}| > \eta\}$. By the approximation guarantee $|\alpha_{ij} - \hat{\alpha_{ij}}|< \eta$, we conclude that for all $(i,j) \in U$, $s(\alpha_{ij}) = s(\hat{\alpha_{ij}})$. Hence, we build up the system of equations
\begin{equation}\label{eq:linear_sys}
\prod_{(k,l) \in P_{ij}} s_{kl} = s(\hat{\alpha_{ij}})  \text{ for all } (i,j) \in U
\end{equation}
This can be viewed as a system of linear equations in $\F_2$, which is the field with $2$ elements.
Hence, we can use the standard Gaussian elimination algorithm to solve it. Since $s(\hat{\alpha_{ij}}) = s(\alpha_{ij})$ for all $(i,j) \in U$, we know that this system 
has at least one solution, namely setting $s_{kl}$ to be the sign of $\theta_{kl}$ in the true model. Let $\tilde{s}$ be the solution that is returned by the Gaussian elimination algorithm. There might be many solutions, since it's possible that the system is underdetermined, which could be cause by the absence of some equations for $(i,j)\notin U$. 
But in any case, we know that $\tilde{s}$ satisfies \eqref{eq:linear_sys}. 

Now, we set $\overline{\theta_{kl}}= \tilde{s_{kl}} \tilde{\theta_{kl}}$ and let $\overline{\alpha} \in [-1,1]^{n\choose 2}$ be the correlations that are induced by $\overline{\theta}$, namely $\overline{\alpha_{ij}} = \tilde{\alpha_{ij}}\prod_{(k,l)\in P_{ij}} s_{kl}$.
If $(i,j) \in U$, then \eqref{eq:linear_sys} holds, which means that by \eqref{eq:correct_sign} we have 
$|\overline{\alpha_{ij}} - \alpha_{ij}| \leq \eta$. 
If $(i,j) \notin U$, then 
$|\hat{\alpha}_{ij}|\leq \eta$ implies $|\alpha_{ij}|\leq 2\eta$ and $|\tilde{\alpha_{ij}}|= |\overline{\alpha_{ij}}| \leq 2\eta$. Thus,
$|\overline{\alpha_{ij}}-\alpha_{ij}| \leq 4\eta$.

Next, we argue about the TV distance between the distribution $\overline{\mu}$ that is induced on the leaves by the output $\overline{\alpha}$ of the algorithm and $\mu$. We just proved that for all $i,j$
$$
|\alpha_{ij} - \overline{\alpha_{ij}}| \leq 4\eta
$$
We can then apply Theorem~\ref{thm:probabilistic_fixed}, which gives
$$
TV(\mu,\tilde{\mu}) \leq 8n^2 \eta = 8n^2  \sqrt{\frac{2\log(n^2/\delta)}{m}}
$$
To make this quantity smaller than $\epsilon$, we need 
$$
m = \Theta\lp(\frac{n^4\log(n/\delta)}{\epsilon^2}\rp)
$$
samples.

\end{proof}

\section{Proof of Theorem~\ref{thm:probabilistic}(Different topologies)}

Let us start with some definitions. Let $T$ be a tree with $n$ leaves and where all non-leaf nodes have degree $3$. 
For a vector $\alpha \in [-1,1]^{n \choose 2}$, we say that $\alpha$ is \emph{induced} by a metric in $T$ if there exists an assignment $\theta_e \in [-1,1]$ of weights for each edge $e$ of $T$, such that
for all leaves $i,j$
$$
\alpha_{ij} = \prod_{e \in P_{ij}} \theta_e
$$
If $\mu$ is the distribution of the leaves of $T$ when the vector of pairwise distances between leaves is $\alpha$, we say that $\mu$ is specified by the pair $(T,\alpha)$.
For an even subset $S \subseteq [n]$, we denote $\alpha_S^T$ the coefficient of $\prod_{i\in S} x_i$ in the Fourier expansion of the probability distribution on a tree $T$. 
In this Section, we prove the following Theorem, which is a restatement of Theorem~\ref{thm:probabilistic} (different topologies).

\begin{theorem}\label{thm:probabilistic_unknown}
Let $T,\hat{T}$ denote the topologies of two trees with $n$ leaves, where each non-leaf node has degree $3$. Let $\alpha,\hat{\alpha} \in [-1,1]^{n \choose 2}$ denote vectors of pairwise distances that are induced by some tree metric on $T$ and $\hat{T}$ , respectively. 
Let also $\hat{\mu}$ be the
distribution that is induced on the leaves by $(\hat{T},\hat{\alpha})$ and 
$\mu$ the one induced by $(T,\alpha)$. Finally, let $D$ be the minimum diameter of $T,\hat{T}$.
Suppose that for all leaves $i \neq j$ we have that $|\alpha_{ij} - \hat{\alpha}_{ij}| \leq \epsilon$. 
Then, 
\begin{equation}\label{eq:tensorization}
TV(\mu,\hat{\mu}) \leq C n^5D \epsilon
\end{equation}
where $C$ is an absolute constant. 
\end{theorem}

For four leaves $i,j,k,l$, we will denote a quartet of leaves by $\{i,j,k,l\}$ when we do not wish to specify their relative placement. The following fact is folklore: if we 
contract all edges that do not belong to some path between two leaves in the quartet, then we might end up with one of three possible topologies. We call this the topology of the quartet $\{i,j,k,l\}$. The $3$ topologies are shown in Figure~\ref{fig:quartet-possibilities}. For example, if we are in the first topology, we write $\{(12)(34)\}$ to denote that fact. We might refer to the quartet as either $\{1,2,3,4\}$ of $\{(12)(34)\}$, depending on whether we want to highlight the topology of the quartet or not. 

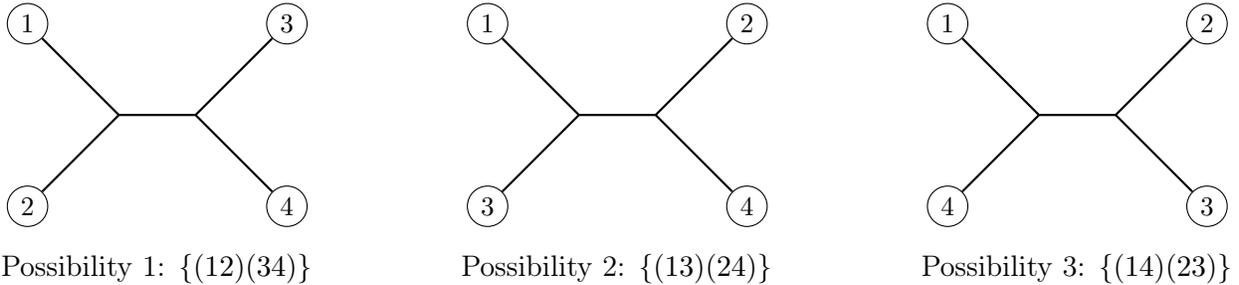
\begin{figure}[htbp]
\centering
\resizebox{\linewidth}{!}{%
\begin{tikzpicture}
%
%
\coordinate (midL1) at (0,0);
\coordinate (midR1) at (1,0);
\node[draw, circle, minimum size=15pt, inner sep=2pt, above left=of midL1] (i1) {\small $1$};
\node[draw, circle, minimum size=15pt, inner sep=2pt, below left=of midL1] (j1) {\small $2$};
\node[draw, circle, minimum size=15pt, inner sep=2pt, above right=of midR1] (k1) {\small $3$};
\node[draw, circle, minimum size=15pt, inner sep=2pt, below right=of midR1] (l1) {\small $4$};
\draw[thick] (midL1) -- (midR1);
\draw[thick] (midL1) -- (i1);
\draw[thick] (midL1) -- (j1);
\draw[thick] (midR1) -- (k1);
\draw[thick] (midR1) -- (l1);

%
%
\coordinate (midL2) at (6,0);
\coordinate (midR2) at (7,0);
\node[draw, circle, minimum size=15pt, inner sep=2pt, above left=of midL2] (i2) {\small $1$};
\node[draw, circle, minimum size=15pt, inner sep=2pt, below left=of midL2] (k2) {\small $3$};
\node[draw, circle, minimum size=15pt, inner sep=2pt, above right=of midR2] (j2) {\small $2$};
\node[draw, circle, minimum size=15pt, inner sep=2pt, below right=of midR2] (l2) {\small $4$};
\draw[thick] (midL2) -- (midR2);
\draw[thick] (midL2) -- (i2);
\draw[thick] (midL2) -- (k2);
\draw[thick] (midR2) -- (j2);
\draw[thick] (midR2) -- (l2);

%
%
\coordinate (midL3) at (12,0);
\coordinate (midR3) at (13,0);
\node[draw, circle, minimum size=15pt, inner sep=2pt, above left=of midL3] (i3) {\small $1$};
\node[draw, circle, minimum size=15pt, inner sep=2pt, below left=of midL3] (l3) {\small $4$};
\node[draw, circle, minimum size=15pt, inner sep=2pt, above right=of midR3] (j3) {\small $2$};
\node[draw, circle, minimum size=15pt, inner sep=2pt, below right=of midR3] (k3) {\small $3$};
\draw[thick] (midL3) -- (midR3);
\draw[thick] (midL3) -- (i3);
\draw[thick] (midL3) -- (l3);
\draw[thick] (midR3) -- (j3);
\draw[thick] (midR3) -- (k3);

%
%
\node[inner sep=0pt] at ($(midL1)!0.5!(midR1) + (0,-2)$) {Possibility 1: $\{(12)(34)\}$};
\node[inner sep=0pt] at ($(midL2)!0.5!(midR2) + (0,-2)$) {Possibility 2: $\{(13)(24)\}$};
\node[inner sep=0pt] at ($(midL3)!0.5!(midR3) + (0,-2)$) {Possibility 3: $\{(14)(23)\}$};
\end{tikzpicture}
}
\caption{The three possible topologies for a quartet. 
In Possibility 1, the path from $1$ to $2$ does not intersect the path from $3$ to $4$. Further, $\alpha_{12}\alpha_{34} \ge \alpha_{13}\alpha_{24} = \alpha_{14}\alpha_{23}$ (and similarly for Possibilities 2 and 3).}
\label{fig:quartet-possibilities}
\end{figure}

As explained in Figure~\ref{fig:quartet-possibilities}, what distinguishes the topology is the relative order between the products $\alpha_{12}\alpha_{34}, \alpha_{13}\alpha_{24}, \alpha_{14}\alpha_{23}$. When $\alpha$ is induced by some tree metric, then two of these products will always be equal and the third will be larger or equal. Depending on which of the products is larger, we get one of the tree possible topologies (if all products are equal then we will choose the topology arbitrarily). Hence, if for some reason we cannot distinguish which of the three products is larger, then we intuitively expect that it will be hard to find the correct topology for a quartet. We give
the relevant definitions below. 

\begin{definition}
Let $i,j,k,l$ be a quartet of four leaves of $T$. We define 
$$
\Delta_{ijkl}(\alpha) := \max(\alpha_{ij}\alpha_{kl}, \alpha_{ik}\alpha_{jl}, \alpha_{il}\alpha_{jk}) - \min(\alpha_{ij}\alpha_{kl}, \alpha_{ik}\alpha_{jl}, \alpha_{il}\alpha_{jk})
$$
\end{definition}
\begin{definition}
Let $i,j,k,l$ be a quartet of four leaves of $T$. We say that this is an \emph{$\epsilon$-good quartet} w.r.t. some vector $\alpha$ if 
$$
\Delta_{ijkl}(\alpha) > \epsilon
$$
\end{definition}

A quartet of leaves that is not an $\epsilon$-good quartet is called an \emph{$\epsilon$-bad quartet}. 
Intuitively, if a quartet is good, then it is easy for an algorithm to distinguish which is the correct topology of these four leaves out of the three possibilities. If it is bad, then the topology is very close to being a star and so all three possibilities are roughly equivalent.
One thing to note is that if $|\hat{\alpha} - \alpha^*| \leq \epsilon$, then
$$
|\Delta_{ijkl}(\hat{\alpha}) - \Delta_{ijkl}(\alpha^*)| \leq 2\epsilon
$$
for all $i,j,k,l$. Hence, it does not really make a difference whether a quartet is good or bad with respect to $\hat{\alpha}$ or $\alpha^*$, since there is only an $O(\epsilon)$-additive error. 

We now proceed with the proof of Theorem~\ref{thm:probabilistic_unknown}.
First, we need to formally define some notions of cutting and pasting nodes in different parts of the tree. This will prove useful in having a unified vocabulary when describing the process of interpolating between two trees.

\begin{definition}
Let $T = G(V,E)$ denote the topology of a tree where every node has degree at most $3$. We define $\textsc{Binary}(T)$ to be the tree that is obtained from $T$ by contracting all maximal paths of nodes of degree $2$ into a single edge. In other words, it is obtained if we succesively find a degree $2$ node $u$ with edges $(u,v), (u,w)$ and replace it with edge $(v,w)$, until we cannot find such a node. 
\end{definition}

Notice that the output of \textsc{Binary} is also described in Definition~\ref{def:contraction}. 
For an example of how \textsc{Binary} works, see Figure~\ref{fig:binary}. Clearly, \textsc{Binary}(T) satisfies the property that every non-leaf node has degree $3$. 

\begin{figure}[htbp]
\centering
\scalebox{0.6}{\resizebox{\linewidth}{!}{%
\begin{tikzpicture}
%
%
\node[draw, circle, minimum size=15pt, inner sep=2pt] at (0,0) (v1-T) {};
\node[draw, circle, minimum size=15pt, inner sep=2pt] at ($(v1-T) + (-1.75,-2)$) (v2-T) {};
\node[draw, circle, minimum size=15pt, inner sep=2pt] at ($(v1-T) + (1.75,-2)$) (v3-T) {};
\node[draw, circle, minimum size=15pt, inner sep=2pt] at ($(v2-T) + (-1,-2)$) (v4-T) {};
\node[draw, circle, minimum size=15pt, inner sep=2pt] at ($(v2-T) + (1,-2)$) (v5-T) {};
\node[draw, circle, minimum size=15pt, inner sep=2pt] at ($(v1-T) + (1.75,1)$) (v6-T) {};

\node[draw, circle, minimum size=15pt, inner sep=2pt] at ($(v1-T)!0.5!(v2-T)$) (v12-T) {};
\node[draw, circle, minimum size=15pt, inner sep=2pt] at ($(v1-T)!0.5!(v3-T)$) (v13-T) {};

\draw[thick, blue] (v1-T) -- (v12-T) -- (v2-T);
\draw[thick, red] (v1-T) -- (v13-T) -- (v3-T);
\draw[thick] (v1-T) -- (v6-T);
\draw[thick] (v2-T) -- (v4-T);
\draw[thick] (v2-T) -- (v5-T);

%
%
\node[draw, circle, minimum size=15pt, inner sep=2pt] at (7,0) (v1-BT) {};
\node[draw, circle, minimum size=15pt, inner sep=2pt] at ($(v1-BT) + (-1.75,-2)$) (v2-BT) {};
\node[draw, circle, minimum size=15pt, inner sep=2pt] at ($(v1-BT) + (1.75,-2)$) (v3-BT) {};
\node[draw, circle, minimum size=15pt, inner sep=2pt] at ($(v2-BT) + (-1,-2)$) (v4-BT) {};
\node[draw, circle, minimum size=15pt, inner sep=2pt] at ($(v2-BT) + (1,-2)$) (v5-BT) {};
\node[draw, circle, minimum size=15pt, inner sep=2pt] at ($(v1-BT) + (1.75,1)$) (v6-BT) {};

\draw[thick, blue] (v1-BT) -- (v2-BT);
\draw[thick, red] (v1-BT) -- (v3-BT);
\draw[thick] (v1-BT) -- (v6-BT);
\draw[thick] (v2-BT) -- (v4-BT);
\draw[thick] (v2-BT) -- (v5-BT);

%
\node[] at (-2,0.5) {\small $T$};
\node[] at (5,0.5) {\small $\textsc{BINARY}(T)$};
\end{tikzpicture}
}}
\caption{On the left we have the tree before the contraction. On the right, we have the tree after applying \textsc{Binary}. We have highlighted with similar colors the path that is contracted on the left and the final edge on the right.
}
\label{fig:binary}
\end{figure}
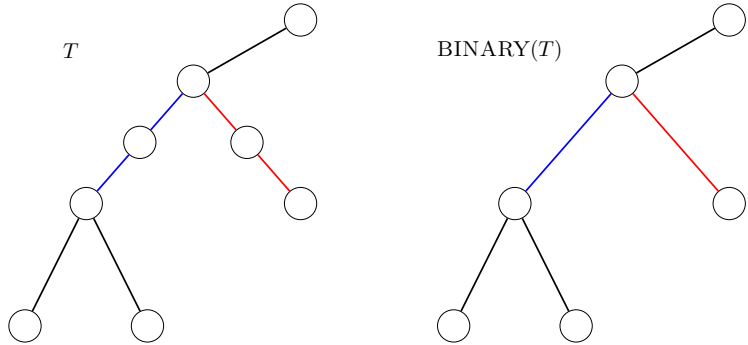

\begin{definition}
Let $T = G(V,E)$ denote the topology of a tree and suppose $(u,v) \in E$. Let $(r,s)\in E$ be some other edge of $T$, where we might have $\{r,s\}\cap \{u,v\}\neq \emptyset$. The only requirement is that $r,s$ belong to a different component than $u$ when we remove edge $(u,v)$ from $T$. We define \textsc{CutPaste}$(T,u,v,(r,s))$ to be the tree that is obtained from $T$ as follows: we delete edges $(u,v)$ and $(r,s)$, we add a node $t$ and we add the edges $(t,u),(t,r),(t,s)$. This produces a tree $T'$. We set \textsc{CutPaste}$(T,u,v,(r,s))$ = \textsc{Binary}$(T')$.
\end{definition}

An example of applying \textsc{CutPaste} can be seen in Figure~\ref{fig:cutpaste}.

\begin{figure}[htbp]
\centering
\resizebox{\linewidth}{!}{%
\begin{tikzpicture}
%
%
\node[draw, circle, minimum size=15pt, inner sep=2pt] at (0,0) (r-T) {\small $r$};
\node[draw, circle, minimum size=15pt, inner sep=2pt] at ($(r-T) + (2,0)$) (s-T) {\small $s$};
\node[draw, circle, minimum size=15pt, inner sep=2pt] at ($(r-T) + (-1,0)$) (v-T) {\small $v$};
\node[draw, circle, minimum size=15pt, inner sep=2pt] at ($(v-T) + (0,-1)$) (u-T) {\small $u$};
\node[draw, circle, minimum size=15pt, inner sep=2pt] at ($(v-T) + (-1,0)$) (v1-T) {};
\node[draw, circle, minimum size=15pt, inner sep=2pt] at ($(u-T) + (-1,-1)$) (v2-T) {};
\node[draw, circle, minimum size=15pt, inner sep=2pt] at ($(u-T) + (1,-1)$) (v3-T) {};
\node[draw, circle, minimum size=15pt, inner sep=2pt] at ($(r-T) + (0,-1)$) (v4-T) {};
\node[draw, circle, minimum size=15pt, inner sep=2pt] at ($(s-T) + (0,-1)$) (v5-T) {};
\node[draw, circle, minimum size=15pt, inner sep=2pt] at ($(v5-T) + (-1,-1)$) (v6-T) {};
\node[draw, circle, minimum size=15pt, inner sep=2pt] at ($(v5-T) + (1,-1)$) (v7-T) {};

\draw[thick, dashed] (v-T) -- (u-T);
\draw[thick] (v-T) -- (r-T);
\draw[thick, dashed] (r-T) -- (s-T);
\draw[thick] (v-T) -- (v1-T);
\draw[thick] (u-T) -- (v2-T);
\draw[thick] (u-T) -- (v3-T);
\draw[thick] (r-T) -- (v4-T);
\draw[thick] (s-T) -- (v5-T);
\draw[thick] (v5-T) -- (v6-T);
\draw[thick] (v5-T) -- (v7-T);

%
%
\node[draw, circle, minimum size=15pt, inner sep=2pt] at (7,0) (r-Tprime) {\small $r$};
\node[draw, red, circle, minimum size=15pt, inner sep=2pt] at ($(r-Tprime) + (1,0)$) (t-Tprime) {\small $t$};
\node[draw, circle, minimum size=15pt, inner sep=2pt] at ($(r-Tprime) + (2,0)$) (s-Tprime) {\small $s$};
\node[draw, circle, minimum size=15pt, inner sep=2pt] at ($(r-Tprime) + (-1,0)$) (v-Tprime) {\small $v$};
\node[draw, circle, minimum size=15pt, inner sep=2pt] at ($(t-Tprime) + (0,1)$) (u-Tprime) {\small $u$};
\node[draw, circle, minimum size=15pt, inner sep=2pt] at ($(v-Tprime) + (-1,0)$) (v1-Tprime) {};
\node[draw, circle, minimum size=15pt, inner sep=2pt] at ($(u-Tprime) + (-1,1)$) (v2-Tprime) {};
\node[draw, circle, minimum size=15pt, inner sep=2pt] at ($(u-Tprime) + (1,1)$) (v3-Tprime) {};
\node[draw, circle, minimum size=15pt, inner sep=2pt] at ($(r-Tprime) + (0,-1)$) (v4-Tprime) {};
\node[draw, circle, minimum size=15pt, inner sep=2pt] at ($(s-Tprime) + (0,-1)$) (v5-Tprime) {};
\node[draw, circle, minimum size=15pt, inner sep=2pt] at ($(v5-Tprime) + (-1,-1)$) (v6-Tprime) {};
\node[draw, circle, minimum size=15pt, inner sep=2pt] at ($(v5-Tprime) + (1,-1)$) (v7-Tprime) {};

\draw[thick] (v-Tprime) -- (r-Tprime);
\draw[thick, dashed] (r-Tprime) -- (t-Tprime);
\draw[thick, dashed] (t-Tprime) -- (u-Tprime);
\draw[thick, dashed] (t-Tprime) -- (s-Tprime);
\draw[thick] (v-Tprime) -- (v1-Tprime);
\draw[thick] (u-Tprime) -- (v2-Tprime);
\draw[thick] (u-Tprime) -- (v3-Tprime);
\draw[thick] (r-Tprime) -- (v4-Tprime);
\draw[thick] (s-Tprime) -- (v5-Tprime);
\draw[thick] (v5-Tprime) -- (v6-Tprime);
\draw[thick] (v5-Tprime) -- (v7-Tprime);

%
%
\node[draw, circle, minimum size=15pt, inner sep=2pt] at (14,0) (r-Tfinale) {\small $r$};
\node[draw, circle, minimum size=15pt, inner sep=2pt] at ($(r-Tfinale) + (1,0)$) (t-Tfinale) {\small $t$};
\node[draw, circle, minimum size=15pt, inner sep=2pt] at ($(r-Tfinale) + (2,0)$) (s-Tfinale) {\small $s$};
\node[draw, circle, minimum size=15pt, inner sep=2pt] at ($(r-Tfinale) + (-1,0)$) (v-Tfinale) {};
\node[draw, circle, minimum size=15pt, inner sep=2pt] at ($(t-Tfinale) + (0,1)$) (u-Tfinale) {\small $u$};
\node[draw, circle, minimum size=15pt, inner sep=2pt] at ($(u-Tfinale) + (-1,1)$) (v2-Tfinale) {};
\node[draw, circle, minimum size=15pt, inner sep=2pt] at ($(u-Tfinale) + (1,1)$) (v3-Tfinale) {};
\node[draw, circle, minimum size=15pt, inner sep=2pt] at ($(r-Tfinale) + (0,-1)$) (v4-Tfinale) {};
\node[draw, circle, minimum size=15pt, inner sep=2pt] at ($(s-Tfinale) + (0,-1)$) (v5-Tfinale) {};
\node[draw, circle, minimum size=15pt, inner sep=2pt] at ($(v5-Tfinale) + (-1,-1)$) (v6-Tfinale) {};
\node[draw, circle, minimum size=15pt, inner sep=2pt] at ($(v5-Tfinale) + (1,-1)$) (v7-Tfinale) {};

\draw[thick] (v-Tfinale) -- (r-Tfinale);
\draw[thick] (r-Tfinale) -- (t-Tfinale);
\draw[thick] (t-Tfinale) -- (u-Tfinale);
\draw[thick] (t-Tfinale) -- (s-Tfinale);
\draw[thick] (u-Tfinale) -- (v2-Tfinale);
\draw[thick] (u-Tfinale) -- (v3-Tfinale);
\draw[thick] (r-Tfinale) -- (v4-Tfinale);
\draw[thick] (s-Tfinale) -- (v5-Tfinale);
\draw[thick] (v5-Tfinale) -- (v6-Tfinale);
\draw[thick] (v5-Tfinale) -- (v7-Tfinale);

%
\node[] at (-2,1) {\small $T$};
\node[] at (5,1) {\small $T'$};
\node[] at (12,1) {\begin{tabular}{r}\small$\textsc{CutPaste}(T)$\\$= \textsc{Binary}(T')$\end{tabular}};
\end{tikzpicture}
}
\caption{The Figure shows the output of $\textsc{CutPaste}(T,u,v,(r,s))$. In the first step we cut $u$ from it's place and paste it in the middle of edge $(r,s)$. In the second step we contract paths of degree $2$ nodes.}
\label{fig:cutpaste}
\end{figure}
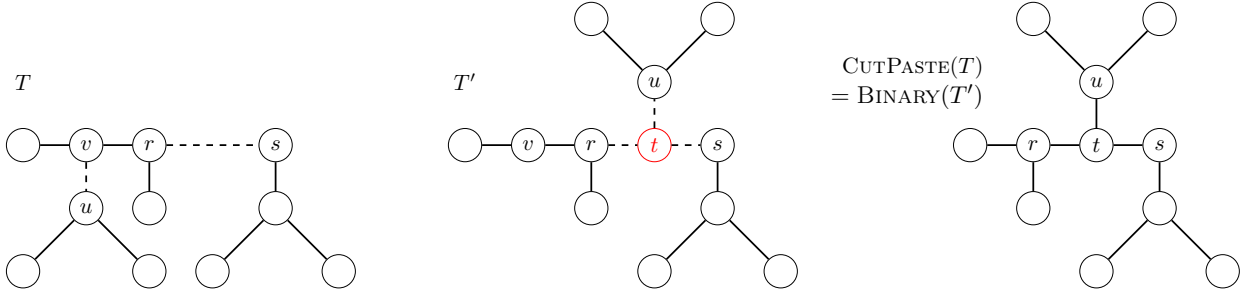

Intuitively, \textsc{CutPaste} encodes the following process: we delete edge $(u,v)$, we add a node $t$ in the middle of edge $(r,s)$ and we attach $u$ together with it's connected component to $t$. This is why the order of $u,v$ as arguments of \textsc{CutPaste} is important, while the order of $r,s$ is not. 

Before describing the interpolation process between the two trees, we prove a lemma about \textsc{CutPaste}. It shows what changes in the distribution if we change the tree according to \textsc{CutPaste}. For a tree $T$ and a quartet $\{w,z,y,u\}$, let's denote by  $T\setminus\{w,z,y,u\}$ the graph that we get if we remove all paths of the quartet $\{w,z,y,u\}$ from $T$. An example is given in Figure~\ref{fig:path-removal}(b). 

\begin{lemma}\label{l:lip_expression}
Let $T = (V,E)$ be a tree and $\alpha \in [-1,1]^{n \choose 2}$. Let $i, j$ be two nodes in $V$(leaf or non-leaf) and denote $i = v_0, v_1,\ldots,v_m = j $ to be the nodes in the path that connects them in the tree, with $m \geq 3$. 
Define $T^k = \textsc{CutPaste}(T, i,v_1,(v_k,v_{k+1}))$ for all $0< k < m$. Denote by $U_k$ the set of quartets where $T^k$ and $T^{k+1}$ differ. Lastly, let's define the vector $\alpha^{wzyu}\in [-1,1]^{n \choose 2}$ as
$$
\alpha^{wzyu}_{kl} = 
\left\{
\begin{array}{ll}
      \alpha_{kl} \text{, if the path $P_{kl}$ does not have common edges with any of the paths in the quartet $\{w,z,y,u\}$}\\
      0 \text{, otherwise}
\end{array} 
\right. 
$$

Then, 
$$
f_x^{T^k}(\alpha) - f_x^{T^{k+1}}(\alpha) =  \sum_{(w,z,y,u) \in U_k}  \lp(\alpha_{\{w,z,y,u\}}^{T^k} - \alpha_{\{w,z,y,u\}}^{T^{k+1}}\rp) x_wx_zx_yx_u \cdot f_x^{T^k\setminus\{w,z,y,u\}}(\alpha^{wzyu})
$$
where $f^T_x(\alpha)$ was defined in \eqref{eq:f-def}.

\end{lemma}

\begin{proof}
Let $I_i$ denote the subset of leaves that lie on the subtree where $i$ belongs to if we delete edge $(i,v_1)$ from the tree, and analogously we define $I_j$ as the set of leaves of the subtree where $j$ belongs to after this removal. 
By definition, we have 
\begin{align*}
f_x^{T_k}(\alpha) - f_x^{T^{k+1}}(\alpha) &= 
\frac{1}{2^n} \sum_{S \subseteq [n], |S| even} (\alpha_S^{T^k} - \alpha_S^{T^{k+1}})\prod_{u \in S} x_u\\
&= \frac{1}{2^n} \sum_{S \subseteq [n], |S| even, S \cap I_i \neq \emptyset } (\alpha_S^{T^k} - \alpha_S^{T^{k+1}})\prod_{u \in S} x_u
\end{align*}
The last equality follows because the relative topology of the leaves $[n]\setminus I_i$ does not change, which means the coefficients $\alpha_S$ for $S \subseteq [n]\setminus I_i$ also do not change. Now, for $0< k < m$ let us define $S_k$ to be the subset of leaves on the connected component that $v_k$ belongs to, if we remove all edges of the path $P_{ij}$ from the graph. 
Let's also define $L_k = \cup_{q=1}^{k}S_q, R_k = \cup_{q = k}^{m-1}S_q \cup I_j$ . We will characterize the set of quartets $U_k$ that change from $T_k$ to $T_{k+1}$. 
An illustration of all these concepts we just defined is given in Figure~\ref{fig:movement}.

\begin{figure}[htbp]
\centering
\resizebox{\linewidth}{!}{%
\begin{tikzpicture}
%
%
\node[draw, circle, minimum size=15pt, inner sep=2pt] at (0,0) (v0) {\small $v_0$};
\node[above=5pt of v0] {\small $v_0 = i$};
\node[draw, circle, minimum size=15pt, inner sep=2pt, right=50pt of v0] (v1) {\small $v_1$};
\node[draw, circle, minimum size=15pt, inner sep=2pt, right=50pt of v1] (v2) {\small $v_2$};
\node[above=5pt of v2] {\small $v_2 = v_k$};
\node[draw, circle, minimum size=15pt, inner sep=2pt, right=50pt of v2] (v3) {\small $v_3$};
\node[draw, circle, minimum size=15pt, inner sep=2pt, right=50pt of v3] (v4) {\small $v_4$};
\node[draw, circle, minimum size=15pt, inner sep=2pt, right=50pt of v4] (v5) {\small $v_5$};
\node[above=5pt of v5] {\small $v_5 = j$};

\draw[thick] (v0) -- (v1);
\draw[thick] (v1) -- (v2);
\draw[thick] (v2) -- (v3);
\draw[thick] (v3) -- (v4);
\draw[thick] (v4) -- (v5);

%
%

\node[draw, circle, minimum size=15pt, inner sep=2pt, below=50pt of v0] (S0) {\small $w$};
\node[draw, white, circle, minimum size=15pt, inner sep=2pt, below=50pt of v1] (S1) {};
\node[draw, circle, minimum size=15pt, inner sep=2pt, below=50pt of v2] (S2) {\small $z$};
\node[draw, circle, minimum size=15pt, inner sep=2pt, below=50pt of v3] (S3) {\small $y$};
\node[draw, circle, minimum size=15pt, inner sep=2pt, below=50pt of v4] (S4) {\small $u$};
\node[draw, white, circle, minimum size=15pt, inner sep=2pt, below=50pt of v5] (S5) {};
\node[fit=(S0), draw, inner sep=5pt, isosceles triangle, rotate=90, anchor=center] (S0-tree) {};
\node[fit=(S1), draw, inner sep=5pt, isosceles triangle, rotate=90, anchor=center] (S1-tree) {};
\node[fit=(S2), draw, inner sep=5pt, isosceles triangle, rotate=90, anchor=center] (S2-tree) {};
\node[fit=(S3), draw, inner sep=5pt, isosceles triangle, rotate=90, anchor=center] (S3-tree) {};
\node[fit=(S4), draw, inner sep=5pt, isosceles triangle, rotate=90, anchor=center] (S4-tree) {};
\node[fit=(S5), draw, inner sep=5pt, isosceles triangle, rotate=90, anchor=center] (S5-tree) {};

\draw[thick] (v0) -- (S0-tree);
\draw[thick] (v1) -- (S1-tree);
\draw[thick] (v2) -- (S2-tree);
\draw[thick] (v3) -- (S3-tree);
\draw[thick] (v4) -- (S4-tree);
\draw[thick] (v5) -- (S5-tree);

%
%
\node[draw, blue, circle, minimum size=15pt, inner sep=2pt] at ($(v1)!0.5!(v2) + (0,1)$) (i1) {\small $v_0$};
\node[draw, orange, circle, minimum size=15pt, inner sep=2pt] at ($(v2)!0.5!(v3) + (0,1)$) (i2) {\small $v_0$};
\node[draw, red, circle, minimum size=15pt, inner sep=2pt] at ($(v3)!0.5!(v4) + (0,1)$) (i3) {\small $v_0$};

\node[draw, circle, minimum size=15pt, inner sep=2pt, above=50pt of i1] (S0-1) {\small $w$};
\node[draw, circle, minimum size=15pt, inner sep=2pt, above=50pt of i2] (S0-2) {\small $w$};
\node[draw, circle, minimum size=15pt, inner sep=2pt, above=50pt of i3] (S0-3) {\small $w$};
\node[fit=(S0-1), draw, blue, inner sep=5pt, isosceles triangle, rotate=270, anchor=center] (S0-1-tree) {};
\node[fit=(S0-2), draw, orange, inner sep=5pt, isosceles triangle, rotate=270, anchor=center] (S0-2-tree) {};
\node[fit=(S0-3), draw, red, inner sep=5pt, isosceles triangle, rotate=270, anchor=center] (S0-3-tree) {};

\draw[thick, blue] ($(v1)!0.5!(v2)$) -- (i1) -- (S0-1-tree);
\draw[thick, orange] ($(v2)!0.5!(v3)$) -- (i2) -- (S0-2-tree);
\draw[thick, red] ($(v3)!0.5!(v4)$) -- (i3) -- (S0-3-tree);

%
%
\node[below=5pt of S0] {\small $S_0=I_i$};
\node[below=5pt of S1] (S1-label) {\small $S_1$};
\node[below=5pt of S2] (S2-label) {\small $S_2$};
\node[below=5pt of S3] {\small $S_3=S_{k+1}$};
\node[below=5pt of S4] (S4-label) {\small $S_4$};
\node[below=5pt of S5] (S5-label) {\small $S_5=I_j$};
\node[above=5pt of S0-1] {\small $S_0$};
\node[above=5pt of S0-2] {\small $S_0$ (in $T^k$)};
\node[above=5pt of S0-3] {\small $S_0$ (in $T^{k+1}$)};

\draw[decorate, decoration={mirror,brace,raise=5pt,amplitude=5pt}] ($(S1-label) + (-1,0)$) -- ($(S2-label) + (1,0)$) node[midway,yshift=-20pt]{$L_2 = L_k$};
\draw[decorate, decoration={mirror,brace,raise=5pt,amplitude=5pt}] ($(S4-label) + (-1,0)$) -- ($(S5-label) + (1,0)$) node[midway,yshift=-20pt]{$R_4 = R_{k+2}$};
\end{tikzpicture}
}
\caption{This is an illustration of how \textsc{CutPaste} cuts $i$ from a place and moves it along the path to $v_5 = j$, one step at a time. This is also exactly the same movement that is done by Algorithm~\ref{alg:interpolation}, where one move corresponds to moving $v_0$ one step to the right. Note that quartet $\{w,z,y,u\}$ is changed when we move $v_0$ from the left of $v_3$ to the right of $v_3$. For illustration we denote $k=2$ and we depict the movement from tree $T^k$ to $T^{k+1}$.}
\label{fig:movement}
\end{figure}
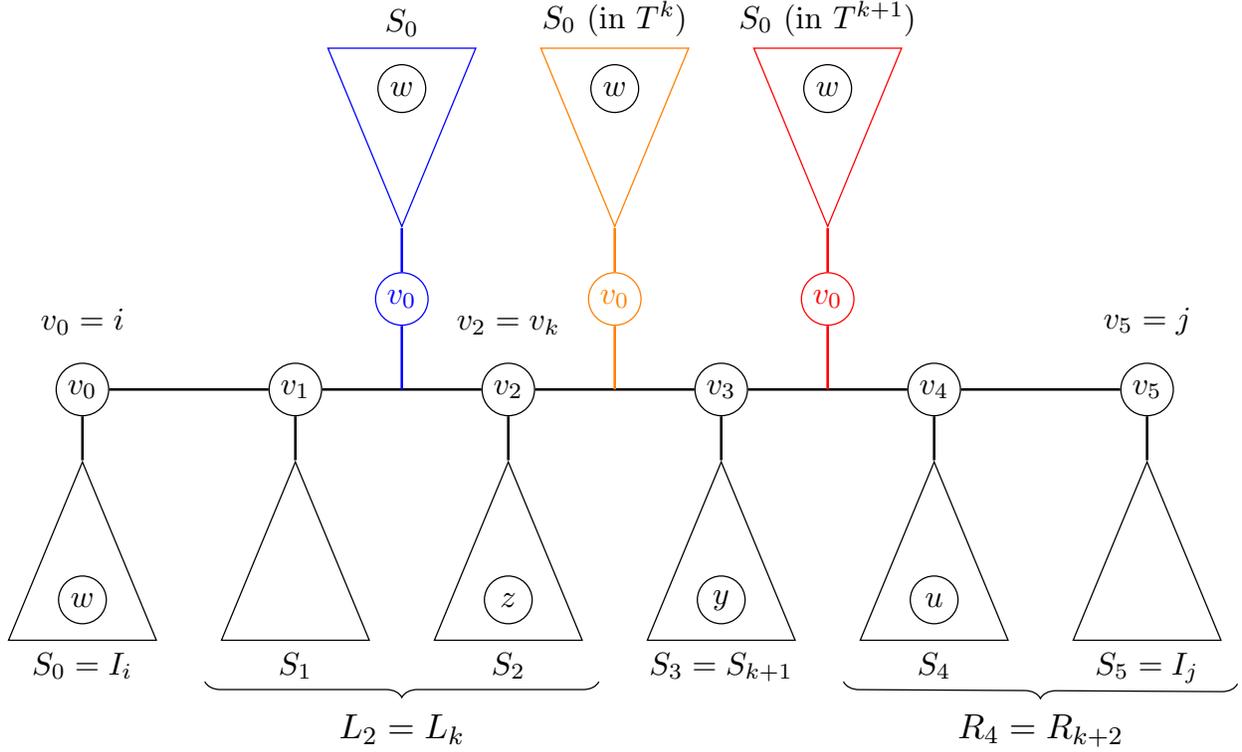

First of all, notice that $I_i,I_j, \{S_q\}_{q=1}^{m-1}$ partitions the set of leaves. It is straightforward to see that
$$
U_k = \{\{w,z,y,u\}: w \in I_i, z \in L_k, y \in S_{k+1}, u \in R_{k+2}\}
$$
Now let's fix a quartet $\{w,z,y,u\} \in U_k$.
We would like to characterize the even subsets $S \supset \{w,z,y,u\}$ such that
$$
\alpha_S^{T^k} = \alpha_{\{w,z,y,u\}}^{T^k} \alpha^{T^k}_{S \setminus \{w,z,y,u\}}
$$
Denote by $\mathcal{S}_{w,z,y,u}$ this collection of subsets. 
Essentially, these are the subsets were the matchings happen so that $w,z$ and $y,u$ are matched together.
The reason we are interested in these subsets is that these are exactly the subsets where $\alpha_S^{T^k}$ and $\alpha_S^{T^{k+1}}$ will be different (once we enumerate over all quartets $\{w,z,y,u\}$ in $U_k$) 
Our strategy to understand how these sets look like will be similar to the one employed in the proof of Lemma~\ref{l:difference_fixed}.
In particular, 
let us consider removing the paths of the quartet $\{w,z,y,u\}$ from $T^k$. This leaves us with a collection of connected subtrees, each with a leaf set $A_r$. Here, $r$ ranges from $1$ to $l$ where $l$ is the number of these components. The set of leaves can be partitioned as
$$
[n] = \{w,z,y,u\}\cup \lp(\cup_{r \leq l} A_r\rp)
$$

It should then be clear from the figure that $S \in \mathcal{S}_{w,z,y,u}$ if and only if $|S\cap A_i|$ is even, for all $i$. To justify that, let's see what happens if for some $r$ $|S\cap A_r|$ was odd. Then, there would be a leaf $b \in A_r$ that would be left unmatched in $A_r$. As we can see from Figure~\ref{fig:quartet_relative}, there are $5$ different possible positions that $b$ can lie in the relative topology of the quartet $\{w,z,y,u\}$. However, from these, only $4$ are possible, since $b$ cannot lie in the middle of the quartet. The reason is that by definition of $\{w,z,y,u\}$ there is no node in the middle edge of that quartet, so there is no subtree that is hanging from there. 

\begin{figure}[htbp]
\centering
\scalebox{0.5}{\resizebox{\linewidth}{!}{%
\begin{tikzpicture}
\coordinate (midL) at (0,0);
\coordinate (midR) at (1,0);
\node[draw, circle, minimum size=15pt, inner sep=2pt] at ($(midL) + (-2,1)$) (w) {\small $w$};
\node[draw, circle, minimum size=15pt, inner sep=2pt] at ($(midL) + (-2,-1)$) (z) {\small $z$};
\node[draw, circle, minimum size=15pt, inner sep=2pt] at ($(midR) + (2,1)$) (y) {\small $y$};
\node[draw, circle, minimum size=15pt, inner sep=2pt] at ($(midR) + (2,-1)$) (u) {\small $u$};

\node[draw, circle, minimum size=15pt, inner sep=2pt] at ($(w)!0.5!(midL) + (0.25,0.75)$) (b1) {\small $b$};
\node[draw, circle, minimum size=15pt, inner sep=2pt] at ($(z)!0.5!(midL) + (0.25,-0.75)$) (b2) {\small $b$};
\node[draw, circle, minimum size=15pt, inner sep=2pt] at ($(u)!0.5!(midR) + (-0.25,-0.75)$) (b3) {\small $b$};
\node[draw, circle, minimum size=15pt, inner sep=2pt] at ($(y)!0.5!(midR) + (-0.25,0.75)$) (b4) {\small $b$};
\node[draw, circle, minimum size=15pt, inner sep=2pt] at ($(midL)!0.5!(midR) + (0,0.75)$) (b5) {\small $b$};
\node[] at ($(b1)+(0.5,0)$) {\scriptsize $(1)$};
\node[] at ($(b2)+(0.5,0)$) {\scriptsize $(2)$};
\node[] at ($(b3)+(0.5,0)$) {\scriptsize $(3)$};
\node[] at ($(b4)+(0.5,0)$) {\scriptsize $(4)$};
\node[] at ($(b5)+(0.5,0)$) {\scriptsize $(5)$};

\draw[thick, decorate, decoration={snake, amplitude=0.35mm}] (midL1) -- (midR1);
\draw[thick, decorate, decoration={snake, amplitude=0.35mm}] (midL1) -- (w);
\draw[thick, decorate, decoration={snake, amplitude=0.35mm}] (midL1) -- (z);
\draw[thick, decorate, decoration={snake, amplitude=0.35mm}] (midR1) -- (y);
\draw[thick, decorate, decoration={snake, amplitude=0.35mm}] (midR1) -- (u);

\draw[thick, dashed] ($(w)!0.5!(midL)$) -- (b1);
\draw[thick, dashed] ($(z)!0.5!(midL)$) -- (b2);
\draw[thick, dashed] ($(u)!0.5!(midR)$) -- (b3);
\draw[thick, dashed] ($(y)!0.5!(midR)$) -- (b4);
\draw[thick, dashed] ($(midL)!0.5!(midR)$) -- (b5);
\end{tikzpicture}
}}
\caption{The $5$ different placings of $b$ relative to the quartet $\{w,z,y,u\}$. Notice that position (5) is actually not possible, as there is no node in the middle of the quartet (see Figure~\ref{fig:movement})}
\label{fig:quartet_relative}
\end{figure}
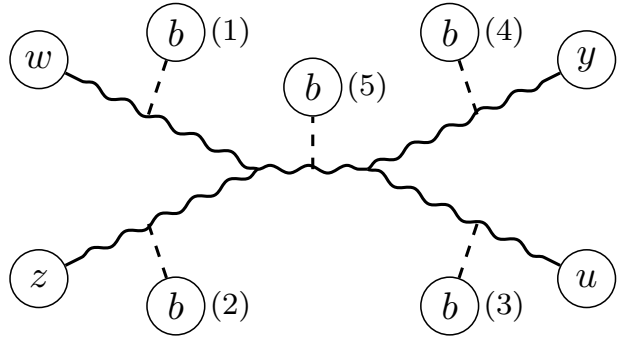

Hence,
$b$ should lie closer to one of the $4$ leaves. Let's assume w.l.o.g. that it lies closer to $w$. Then, a similar argument as in Lemma~\ref{l:difference_fixed} applies. In particular, we can also assume w.l.o.g. that $b$ is the closest leaf in $w$ that is left unmatched by it's subtree $A_r$ (otherwise we consider the closest one instead of $b$). 
Then, $b$ has to be matched with $w$ in $\alpha_S^{T^k}$, which means that the term $\alpha_{wz}$ will not appear in that expression.
This means that $S \notin \mathcal{S}_{w,z,y,u}$, a contradiction. Thus, we established that 
$$
\mathcal{S}_{w,z,y,u} =  \{\{w,z,y,u\}\cup \lp(\cup_{r \leq l} S_r\rp) : S_r \subseteq A_r , |S_r| \text{ even }\} 
$$
Note that the sets $\mathcal{S}_{w,z,y,u}$ are disjoint for different quartets $\{w,z,y,u\}$. Also, it is easy to see that $S \in \mathcal{S}_{w,z,y,u}$ if and only if 
$$
\alpha_S^{T^{k+1}} = \alpha_{\{w,z,y,u\}}^{T^{k+1}} \alpha^{T^k}_{S \setminus \{w,z,y,u\}}
$$
Hence, the sets $S$ such that $\alpha_S^{T^k} \neq \alpha_S^{T^{k+1}}$ are precisely the union $\cup_{\{w,z,y,u\} \in U_k} \mathcal{S}_{w,z,y,u}$.
Now, notice that 
\begin{multline*}
\frac{1}{2^n}\sum_{S \in \mathcal{S}_{w,z,y,u}} (\alpha_S^{T^{k+1}} - \alpha_S^{T^{k}})\prod_{c \in S}x_u \\
= 
\lp(\alpha_{\{w,z,y,u\}}^{T^k} - \alpha_{\{w,z,y,u\}}^{T^{k+1}}\rp) x_wx_zx_yx_u \underbrace{\frac{1}{2^n}\prod_{r}  \lp(\sum_{S \subseteq A_r, |S| even}\alpha_{S } \prod_{c \in S \cap A_r} x_c\rp)}_{f_x^{T\setminus\{w,z,y,u\}}(\alpha^{wzyu})}
\end{multline*}
The last equality is true, since it has the form of a product distribution over the subsets $A_i$, which is exactly the distribution of the topology $T\setminus\{w,z,y,u\}$. 
The weights in each subtree remain the same, but across subtrees the correlations are $0$, which is why the argument is $\alpha^{wzyu}$ now.
Summing over all $\{w,z,y,u\} \in U_k$ gives us the desired claim.

\end{proof}

We now describe the process of interpolating between $T$ and $\hat{T}$. 
We first give the pseudocode, which is Algorithm~\ref{alg:interpolation}. We note that even though we call this process an algorithm, it will only be used as part of the Analysis of the TV distance between two trees. Hence, we are not concerned with its computational complexity.

\begin{algorithm}
\SetAlgoLined\DontPrintSemicolon
\caption{Interpolation between two tree topologies}
\KwInp{Unweighted trees $T = (V,E), \hat{T} = (\hat{V}, \hat{E})$ whose leaves are labeled $1,\dots,n$}
\KwInp{Correlations $(\alpha_{ij})_{i,j \in \{1,\dots,n\},\ i \ne j}$ between leaves of $T$ and $(\hat{\alpha}_{ij})_{i,j \in \{1,\dots,n\},\ i \ne j}$ between leaves of $\hat{T}$}
$T_1  \gets T$\\
$S \gets \{T_1\}$\\
$L \gets \{1,\ldots,n\}$\\
\While(\tcp*[h]{a new round starts}){$|L| \geq 4$ }{ 
\For{$i,j \in L$}{
\If(\tcp*[h]{new epoch}){$\textsc{Cherry}(i,j,T_1) == \textsc{False}$ \textbf{and} $\textsc{Cherry}(i,j,\hat{T}) == \textsc{True}$}{
$p \gets$ common neighbor of $i,j$ in $\hat{T}$\\
$I \gets $ set of leaves in the same component at $i$, if we remove $(i,p)$ from $\hat{T}$\\
$J \gets $ set of leaves in the same component at $j$, if we remove $(j,p)$ from $\hat{T}$\\
$z \gets \arg\max_{u \in I}\hat{\alpha}_{iu}$\\
$w \gets \arg\max_{u \in J}\hat{\alpha}_{ju}$\\
\uIf{$\hat{\alpha}_{zp} > \hat{\alpha}_{wp}$}{Switch $i,j$}
$k \gets $neighbor of $i$ on the path $P_{ij}$\\
$l \gets$ neighbor of $j$ on the path $P_{ij}$\\
$S_2 \gets \textsc{Sequence}(T_1,i,j)$\tcp*{sequence of moves}
$T_1 \gets$ \textsc{CutPaste}$(T_1, i,k,(j,l))$ \tcp*{make $i,j$ a cherry}
$S \gets S \cup S_2$
}
\If(\tcp*[h]{remove cherries where $T_1,\hat{T}$ agree}){ $\textsc{Cherry}(i,j,\hat{T}) == \textsc{True}$}{
$p \gets$ common neighbor of $i,j$ in $\hat{T}$\\
$L \gets L\setminus \{i,j\}$\\
$L \gets L \cup \{p\}$\\
}
}
}
\Return{\rm The list $L$ of topologies that were generated during interpolation}
\label{alg:interpolation}
\end{algorithm}

The interpolation will be carried away in \emph{rounds}. Each round corresponds to a run of the outer \texttt{While} loop. 
In the first round ($q=1)$, we make sure that any two leaves that form a cherry in $\hat{T}$ will also form a cherry in $T$. At the end of the first round, we update the set of leaves by removing leaves that are cherries and adding their parents.
Hence, in the second round, we make sure that parents of leaves that are cherries in $\hat{T}$ become also cherries in $T$ and so on. 

Let us now describe in a bit more detail what happens in each round. 
First of all, notice that the $L$ in the for loop condition is evaluated at the start of the loop. This means that if we change it during the run of the loop, the number of iterations will not be affected.
In the first round, this set $L$ corresponds to the leaf set $[n]$. 
We proceed to search for a pair $i,j$ that is a cherry in $\hat{T}$ but not in $T$. 
If such a pair $i,j$ is found, the we have to move one of them towards the other to make them a cherry. This sequence of moves is called an \emph{epoch} and corresponds to a run of the first \texttt{If} statement inside the \texttt{For}. 
We include an extra \texttt{If} statement since we want to choose the \emph{weakest} of $i,j$ to move (we will see why this is important later). 
To move $i$ towards $j$, we use the fuction \textsc{Sequence}. This gives us all the intermediate topologies that are needed to move $i$ to $j$. 
Each of these topologies corresponds to a \emph{move}. Hence, an epoch consists of moves.
The movement is by cutting $i$ from it's current placement and pasting it in all the edges of the path to $j$ consecutively, similarly to what is shown in Figure~\ref{fig:movement}. 
After this movement is made, $T_3$ is updated to store the new topology.

We now explain the significance of the second \texttt{If} statement. 
If $i,j$ was not a cherry in $T$ but was in $\hat{T}$, then the previous \texttt{If} fixed that. Now, the second \texttt{If} locates all these cherries that are common in $T_3$ and $\hat{T}$ and removes them from the leaf set $L$, while adding their parent. 
This means that the subtree rooted in the parent will not be changed after that point, since it has the same topology in $T_3,\hat{T}$ and instead will be moved around with it's parent in subsequent steps. Hence, in the second round, $L$ will contain some parents of leaves and possibly some leaves that were not matched into cherries in the first round. We give an example run of Algorithm~\ref{alg:interpolation} in Figure~\ref{fig:algorun}.

\begin{figure}[htbp]
\centering
\input{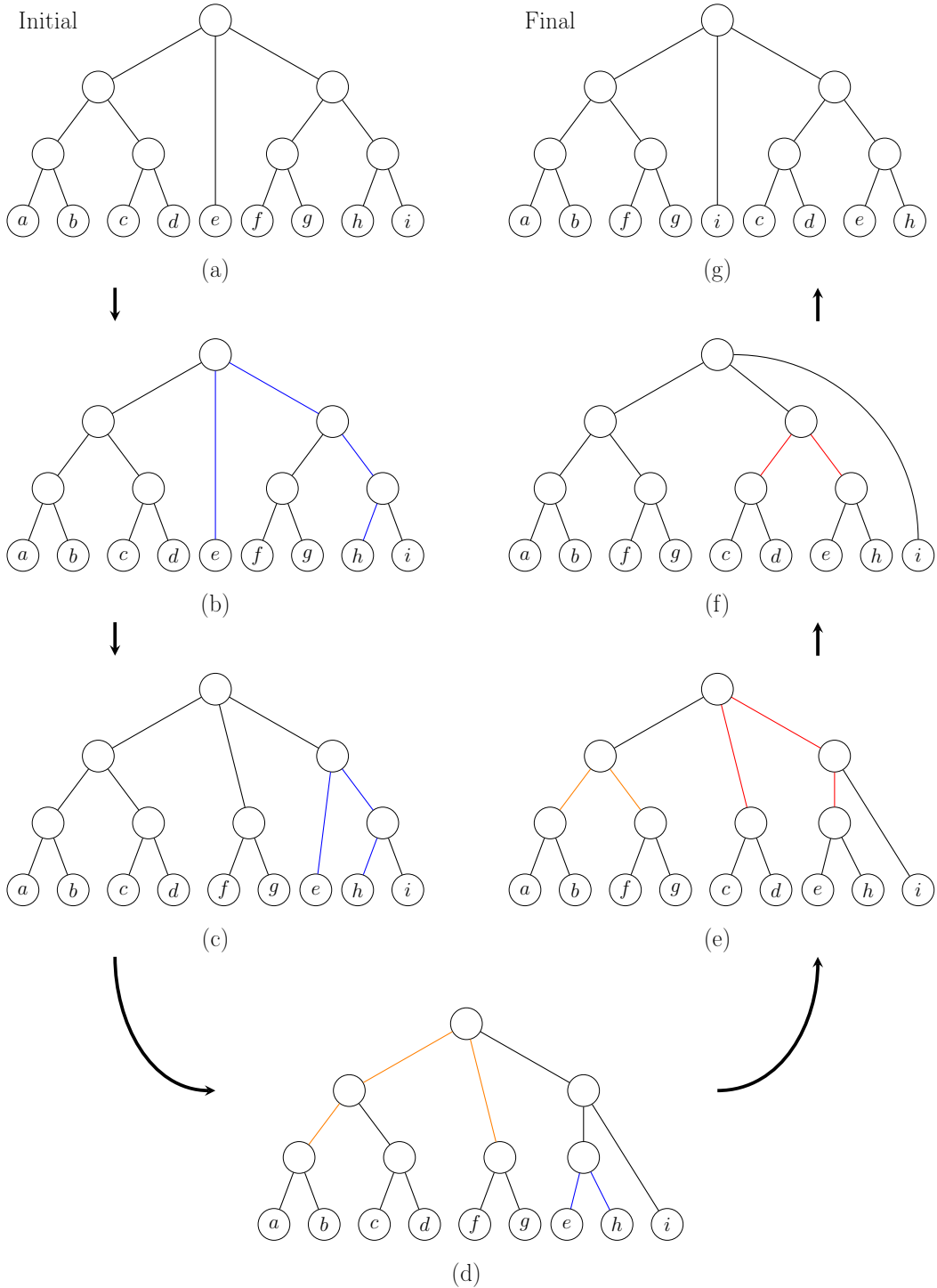}
\caption{
Example run of Algorithm~\ref{alg:interpolation}. In the first round, we have epoch $1$. In epoch $1$, $e$ becomes a cherry with $h$. In the second round, we have epochs number $2$ and $3$. In epoch $2$, the parent of $f,g$ becomes a cherry with the parent of $a,b$. In epoch $3$, the parent of $c,d$ becomes a cherry with the parent of $e,h$. 
After that, we have reached the final topology.}
\label{fig:algorun}
\end{figure}

\begin{algorithm}[htp]
  \SetAlgoLined\DontPrintSemicolon
  \SetKwFunction{algo}{algo}\SetKwFunction{proc}{Sequence}
  \SetKwFunction{fcherry}{Cherry}
  \SetKwProg{myproc}{Function}{}{}
  \myproc{\proc{$T,i,j$}}{
  $S \gets \{T\}$\\
  Let $\{v_0,\ldots, v_m\}$ be path from $i$ to $j$\tcp*{$v_0 = i, v_m = j$}
  \For{$r\gets1$ \KwTo $m-1$}{
  $T_2 \gets \textsc{CutPaste}(T,i,v_1,(v_r,v_{r+1})$ \tcp*{a move happens here}
  $S \gets S \cup \{T_2\}$
  }
  \Return{$S$}
  }
  
  \myproc{\fcherry{i,j,T}}{
  \uIf{$i,j$ have a common neighbor in $T$}{
    \Return \textsc{True}
  }
  \Else{
  \Return \textsc{False}}
  }
  \caption{Functions \textsc{Sequence} and \textsc{Cherry}}
\end{algorithm}


Let's introduce a bit of notation about this process. Suppose $q$ is a round, $t$ is some epoch of this round, and $m$ is some move in epoch $t$.  We denote $(i_t,j_t)$ the pair of leaves from $L$ that is selected during epoch $t$ of the algorithm. Suppose the length of the path $P_{i_tj_t}$ is $l_t$. Then, we denote by $v_0^{t} = i_{t} , v_1^{t},\ldots, v_{l_t}^{t} = j_{t}$ be the nodes in the path from $i_{t}$ to $j_{t}$, which has length $l_t$. 
We denote by $T^{m}$ the topology that we get before move $m$ and $T^{m+1}$ the one we get after the move. We also define $T^0 = T$.
Formally, if $m'$ is the first move of epoch $t$, we have, 
$$
T^{m+1} = \textsc{CutPaste}(T^{m'}, i_{t},v_1^{t}, (v_m^{t},v_{m+1}^{t}))
$$
 It is implied that in the definition of $T^{qrs}$ we do not delete the leaves that have already been fixed into cherries. Note that $i_{t},j_{t}$ might correspond to some internal nodes. Let $I_t, J_t$ be the set of leaves in the same component as $i_t, j_t$ respectively, if we remove all the edges in the path from $i_t$ to $j_t$.

We collect here some observations about Algorithm~\ref{alg:interpolation} that will prove useful in the sequel.

\begin{observation}\label{ob:epoch}
The total number of epochs for a single run of Algorithm~\ref{alg:interpolation} is at most $n$.
\end{observation}
\begin{proof}
Each time an epoch is complete, 
we build a subtree of strictly larger size than before (size stands for number of leaves here) which agrees with $\hat{T}$. Since there are at most $n$ leaves, we need at most $n$ steps until we reach $\hat{T}$. Hence, there are at most $n$ epochs in the whole process.
\end{proof}
\begin{observation}\label{ob:quartet-epoch}
Any quartet $\{i,j,k,l\}$ changes topology at most once per epoch. 
\end{observation}
\begin{proof}
Follows by inspecting the set of quartets that change during a move, which was described in Lemma~\ref{l:lip_expression}. It is trivial to see that these sets are disjoint for different moves in a single epoch. 
\end{proof}
\begin{observation}\label{ob:subtree}
For any epoch $t$ and any move $m$ in $t$, the subtree induced by the leaves in $I_t$ and $J_t$ is identical in $T^m$ and in $\hat{T}$.
\end{observation}
\begin{proof}
This follows inductively by the construction of the Algorithm. When a node $i_t$ is selected, it is either a leaf or some node that was added in $L$ after it's two children $i_{t'},j_{t'}$ became cherries in $T_1$ during a previous epoch $t' < t$. Inductively, the subtrees rooted in $i_{t'},j_{t'}$ have the same topology in $T_1$ and $\hat{T}$. Since $i_{t'},j_{t'}$ are siblings in $T_1$ and in $T'$ in epoch $t$, we conclude that the subtrees rooted at $i_t$ are also identical in $T_1$ and $\hat{T}$ for epoch $t$. Same reasoning applies for $J_t$. 
\end{proof}

\begin{observation}\label{ob:final}
At the end of the last $M$ of Algorithm~\ref{alg:interpolation} we have $T^{M+1} = \hat{T}$.
\end{observation}
\begin{proof}
Let us define the graph $H_t$ as follows: it is obtained by running Algorithm~\ref{alg:interpolation} until epoch $t$ and each time $t' \leq t$ we make a cherry with $i_{t'},j_{t'}$, we remove the subtrees $I_{t'},J_{t'}$, so that $i_{t'},j_{t'}$ become leaves. 
By Observation~\ref{ob:subtree} we know that the subtrees we remove in each epoch have the same topology as in $\hat{T}$. 
Obviously, the number of leaves in $H_t$ shrink with each epoch, until we have $3$ leaves $u,v,w$, for which there is only one possible topology. 
At that point, topology $T^{M+1}$ is obtained by placing the subtrees for $u,v,w$ back. By Observation~\ref{ob:subtree}, we know that these subtrees have the topology of $\hat{T}$, hence $T^{M+1}$ should also have the topology of $\hat{T}$.
\end{proof}

\begin{observation}\label{ob:diam}
If $D$ is the diameter of $\hat{T}$, there are at most $\lceil D/2\rceil$ rounds when we run Algorithm~\ref{alg:interpolation}. 
Furthermore, each leaf is moved in at most one epoch per round.
\end{observation}
\begin{proof}
Consider the graph $H_t$ that was defined in the proof of Observation~\ref{ob:final}. 
We will show that the largest path in $H_t$ shrinks by at least $2$ edges in each round. 
It then follows that there will be at most $\lceil D/2\rceil$ rounds in total.

Let $u,v$ be two leaves of $H_t$ such that $P_{uv}$ in $H_t$ has length equal to the diameter $D$ of $H_t$.
Let $p$ be the only neighbor of $u$ in $H_t$.
Clearly, $u$ is also part of the path $P_{uv}$. Let $w$ be the neighbor of $p$ that does not lie on the path $P_{uv}$ (since $p$ has degree $3$, such a neighbor should exist). We claim that $w$ should be a leaf, otherwise we could extend the path $P_{up}$ into one with larger length than $P_{uv}$. 
Thus, $u,w$ should be siblings, which means they will be selected in the current round to be paired into a cherry, which will remove them from the graph and will leave $p$ as a leaf. Thus, $P_{uv}$ will shring by one edge on the side of $u$ and for the same reason will also shring by one edge on the side of $v$. This proves our claim. 

For the second claim, any leaf $u$ is moved only when some subtree with root $i_t$ is moved and $u$ belongs in this subtree. 
Suppose that this happens during a round, resulting in $i_t,j_t$ becoming a cherry. 
Then, we can see that Algorithm~\ref{alg:interpolation} then removes $i_t,j_t$ from the list of leaves $L$, which means that $i_t$ will not move again for the remainder of that round (it's parent $p_t$ is not considered in the \texttt{For} loop of the current round). 
 Hence, 
$u$ remains fixed for the remaining of that round. 

\end{proof}

We first argue that during this interpolation process, only bad quartets change topology. This is crucial, since good quartets should be maintained if we wish to lose only a little in TV. 

\begin{lemma}\label{l:bad_quartet}
Let $T = G(V,E)$ and $\hat{T} = G(\hat{V},\hat{E})$ be two trees with tree metrics $\alpha,\hat{\alpha}$ respectively. We assume that $\|\alpha-\hat{\alpha}\|_\infty \leq \epsilon$.
Suppose we run the procedure~\ref{alg:interpolation} with input $T,\hat{T},\alpha,\hat{\alpha}$.
Let $T^{m},T^{m+1}$ be two arbitrary consecutive steps in this process.
Let $U_{m}$ be the set of quartets where $T^{m},T^{m+1}$ disagree. Then, for all $(w,z,y,u) \in U_{m}$, we have that 
$$
\Delta_{w,z,y,u}(\alpha) \leq 20 n \epsilon
$$
\end{lemma}

\begin{proof}
We will denote by $t$ the epoch where move $m$ belongs to. 
We will prove the claim inductively over $t$.
We will prove that if a quartet $\{w,z,y,u\}$ is changed during the $t$-th epoch, then 
$$
\Delta_{w,z,y,u}(\alpha) \leq 20 t \epsilon
$$
This obviously implies the final claim since $t \leq n$ by Observation~\ref{ob:epoch}. 
Since the base case is the same as the inductive step, we give the inductive step proof only. 

Suppose we are at epoch $t$. 
First of all, we know that node $i_{t}$ is selected to be moved towards $j_{t}$, where the intermediate nodes are $v_0^{t} = i_{t} , v_1^{t},\ldots, v_{m_{t}}^{t} = j_{t}$. Similarly to the proof of Lemma~\ref{l:lip_expression}, let $I_t$ be the set of leaves on the same component with $i_{t}$ if we remove edge $(v_0^{t},v_1^{t})$, $S_i^{t}$ the set of leaves on the same component as $v_i^{t}$, if we remove edges $(v_{i-1}^{t},v_i^{t}),(v_{i}^{t},v_{i+1}^{t})$ from the tree, and $J_t$ be the set of leaves on the same component with $j_{t}$ if we remove edge $(v_{m_{t}-1}^{t},v_{m_{t}}^{t})$. 
Also, let us define $L_{ts} = \cup_{k \leq s} S_k^{t}, R_{ts} = \cup_{k \geq s} S_k^{t} \cup J_t$ for all $s \leq m_t$. 
The situation is similar to the one presented in Figure~\ref{fig:movement}. 

Suppose $T^m$ corresponds to node $i_t$ being pasted in the middle of edge $(v_s^t,v_{s+1}^t)$ for some fixed $s$. 
As we saw in the proof of Lemma~\ref{l:lip_expression}, the set $U_{m}$ of quartets that differ in $T^m,T^{m+1}$ can be written as
$$
U_{m} = \{\{w,z,y,u\}: w \in I_t, z \in L_s^{t}, y \in S_{s+1}^{t}, u \in R_{s+2}^{t}\}
$$
Note that all the quartets in $U_m$ are considered to change at epoch $t$. 
Suppose there exists a quartet $\{w,z,y,u\} \in U_{m}$ such that 
$$
\Delta_{w,z,y,u}(\alpha) > 20 t \epsilon
$$
First, we will assume that $u \in \cup_{k \geq s} S_k^{t}$. Afterwards, we will deal with the case $u \in J_t$, which will actually prove to be easier. 
The first thing we observe is that we can assume without loss of generality that the topology of $\{w,z,y,u\}$ has not been altered in 
any previous epoch. 
The reason is that if it was altered at some epoch $t' < t$, then by the inductive assumption, we already have
$$
\Delta_{w,z,y,u}(\alpha) \leq 20 t' \epsilon < 20 t \epsilon
$$
and we have nothing to prove. 
Hence, we can assume w.l.o.g. that it is the first time that it is changing topology. Note also that by Observation~\ref{ob:quartet-epoch} a quartet changes topology at most once per epoch.
Since it has not changed topology before, it follows that it's topology in $T$ is $\{(wz)(yu)\}$.
It is straightforward to notice that
$$
\Delta_{w,z,y,u}(\alpha) = \Delta_{w,z,y,u}(|\alpha|)
$$
where $|\alpha|$ is the vector of absolute values of $\alpha$. 
This implies that
$$
\Delta_{w,z,y,u}(\alpha)  = |\alpha_{wz}||\alpha_{yu}| - |\alpha_{zy}||\alpha_{wu}| > 20 t \epsilon
$$
Our assumption about $\alpha,\hat{\alpha}$ implies that
$$
|\alpha_{wz}||\alpha_{yu}| - |\hat\alpha_{wz}||\hat\alpha_{yu}| \leq 2\epsilon \quad,\quad |\alpha_{zy}||\alpha_{wu}| - |\hat\alpha_{zy}||\hat\alpha_{wu}| \leq 2\epsilon
$$
Hence, we have
$$
|\hat\alpha_{wz}||\hat\alpha_{yu}| - 
|\hat\alpha_{zy}||\hat\alpha_{wu}| > 20 t \epsilon - 4\epsilon
$$
Let $p_{t}$ be the common parent of $i_{t},j_{t}$ in the tree $\hat{T}$. 
Now, by the construction of procedure~\ref{alg:interpolation}(first \texttt{If} statement), we know that
\begin{equation}\label{eq:comparison}
\max_{f \in J_t} |\hat{\alpha}_{f,p_{t}}| \geq \max_{f \in I_t} |\hat{\alpha}_{f,p_{t}}|
\end{equation}
Now, we know by Observation~\ref{ob:subtree}  that the subtrees rooted at $i_{t}$ and $j_{t}$ with leaf sets $I_t$ and $J_t$ respectively have the same topology in $T^{m}$ and $\hat{T}$. Since $z,u \notin I_t$, we can write (see Figure~\ref{fig:siblings}).
$$
|\hat\alpha_{wz}||\hat\alpha_{yu}| - 
|\hat\alpha_{zy}||\hat\alpha_{wu}| = 
|\hat{\alpha}_{w,p_{t}}|(|\hat{\alpha}_{z,p_{t}}| |\hat\alpha_{yu}| - |\hat\alpha_{zy}||\hat{\alpha}_{u,p_{t}}|) 
$$

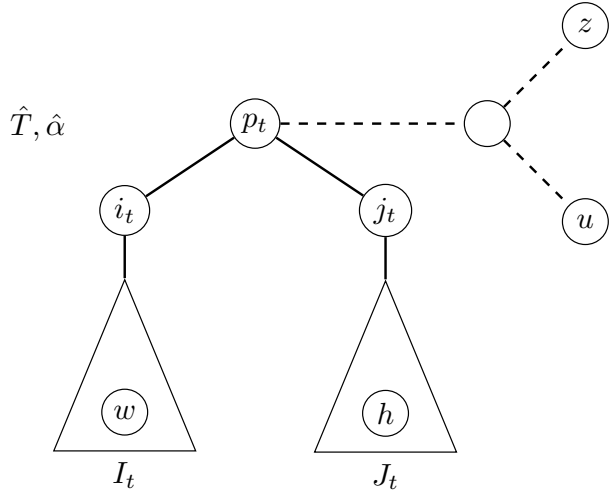
\begin{figure}[htbp]
\centering
\scalebox{0.5}{\resizebox{\linewidth}{!}{%
\begin{tikzpicture}
%
%
\node[draw, circle, minimum size=15pt, inner sep=2pt] at (0,0) (v0) {\small $p_t$};
\node[draw, circle, minimum size=15pt, inner sep=2pt] at ($(v0) + (-1.5,-1)$) (v1) {\small $i_t$};
\node[draw, circle, minimum size=15pt, inner sep=2pt] at ($(v0) + (1.5,-1)$) (v2) {\small $j_t$};

\node[draw, circle, minimum size=15pt, inner sep=2pt, right=60pt of v0] (v3) {};
\node[draw, circle, minimum size=15pt, inner sep=2pt, above right=30pt of v3] (v4) {\small $z$};
\node[draw, circle, minimum size=15pt, inner sep=2pt, below right=30pt of v3] (v5) {\small $u$};

\draw[thick] (v0) -- (v1);
\draw[thick] (v0) -- (v2);
\draw[thick, dashed] (v0) -- (v3);
\draw[thick, dashed] (v3) -- (v4);
\draw[thick, dashed] (v3) -- (v5);

%
%

\node[draw, circle, minimum size=15pt, inner sep=2pt, below=50pt of v1] (iqr) {\small $w$};
\node[draw, circle, minimum size=15pt, inner sep=2pt, below=50pt of v2] (jqr) {\small $h$};
\node[fit=(iqr), draw, inner sep=5pt, isosceles triangle, rotate=90, anchor=center] (iqr-tree) {};
\node[fit=(jqr), draw, inner sep=5pt, isosceles triangle, rotate=90, anchor=center] (jqr-tree) {};

\draw[thick] (v1) -- (iqr-tree);
\draw[thick] (v2) -- (jqr-tree);

%
%
\node[below=5pt of iqr] {\small $I_t$};
\node[below=5pt of jqr] {\small $J_t$};
\node[left=50pt of v0] {\small $\hat{T}, \hat{\alpha}$};
\end{tikzpicture}
}}
\caption{From the picture, it is clear that $\hat{\alpha}_{wz} = \hat{\alpha}_{wp_t}\hat{\alpha}_{p_tz}$}
\label{fig:siblings}
\end{figure}

Let $h = \arg\max_{f \in J_t} |\hat{\alpha}_{f,p_{t}}|$. Then, by\eqref{eq:comparison} we have

$$
|\hat{\alpha}_{h,p_{t}}|(|\hat{\alpha}_{z,p_{t}}| |\hat\alpha_{yu}| - |\hat\alpha_{zy}||\hat{\alpha}_{u,p_{t}}|)  \geq 
|\hat{\alpha}_{w,p_{t}}|(|\hat{\alpha}_{z,p_{t}}| |\hat\alpha_{yu}| - |\hat\alpha_{zy}||\hat{\alpha}_{u,p_{t}}|)
$$
Since we have assumed that $z,u \notin I_t$, figure~\ref{fig:siblings} implies that
$$
|\hat\alpha_{hz}||\hat\alpha_{yu}| -|\hat\alpha_{zy}||\hat\alpha_{hu}| = 
|\hat{\alpha}_{h,p_{t}}|(|\hat{\alpha}_{z,p_{t}}| |\hat\alpha_{yu}| - |\hat\alpha_{zy}||\hat{\alpha}_{u,p_{t}}|)  \geq 
|\hat{\alpha}_{w,p_{t}}|(|\hat{\alpha}_{z,p_{t}}| |\hat\alpha_{yu}| - |\hat\alpha_{zy}||\hat{\alpha}_{u,p_{t}}|) > 20 t\epsilon - 6 \epsilon   
$$
By the closeness of $\alpha,\hat\alpha$, this in turn implies that
\begin{equation}\label{eq:contradiction}
|\alpha_{hz}||\alpha_{yu}| -|\alpha_{zy}||\alpha_{hu}| > 20 t \epsilon - 10 \epsilon > 20 (t-1)\epsilon
\end{equation}
Clearly, leaves $h,z,y,u$ do not change position during epoch $t$, hence the quartet $\{h,z,y,u\}$ does not change topology during epoch $t$. 
Now, there are two possibilities:

\emph{Case 1:} Suppose $\{h,z,y,u\}$ has changed topology at least once in some previous epoch $t' < t$. 
Then, by the inductive hypothesis, we should have 
$$
\Delta_{h,z,y,u}(\alpha) \leq 20 t'\epsilon
$$
Since $t' \leq t-1$, this contradicts \eqref{eq:contradiction}.

\emph{Case 2:} Suppose $\{h,z,y,u\}$ has not changed topology until epoch $t$. This means that the topology of $\{h,z,y,u\}$ in $T$ is $\{(zy)(hu)\}$, since that is the topology in $T^{qrs}$. This implies that
$$
|\alpha_{hz}||\alpha_{yu}| -|\alpha_{zy}||\alpha_{hu}| < 0
$$
which again contradicts \eqref{eq:contradiction}. 

Hence, in all cases we obtain a contradiction and the inductive step is proved. 
Now, let's consider the case $u \in I_t$. 
Then, we can assume w.l.o.g. that this is the first epoch where $\{w,z,y,u\}$ changes topology, otherwise the inductive step applies.
This means that the topology of this quartet in $T$ is $\{(wz)(yu)\}$. 
Hence,
$$
\Delta_{w,z,y,u}(\alpha)= |\alpha_{wz}||\alpha_{yu}| - |\alpha_{zy}||\alpha_{wu}| 
$$
Clearly, the topology of this quartet is $\{(zy)(wu)\}$ in $\hat{T}$. This implies that
$$
|\hat\alpha_{wz}||\hat\alpha_{yu}| - 
|\hat\alpha_{zy}||\hat\alpha_{wu}| < 0
$$
In turn, this means
$$
|\alpha_{wz}||\alpha_{yu}| - |\alpha_{zy}||\alpha_{wu}| < 4\epsilon < 20 t \epsilon
$$
and this concludes the claim in that case as well.

\end{proof}

We now formulate our main result, which bounds the Lipschitzness of the function $f_x^T$ in terms of local changes in the topology of $T$. We will use it to relate the changes in TV to the changes of quartet topologies along this interpolation process.

\begin{lemma}\label{l:interpolate_bound}
Let $T = G(V,E)$ and $\hat{T} = G(\hat{V},\hat{E})$ be two trees and suppose that the diameter of $\hat{T}$ is $D$. Let $\alpha$ be some tree metric induced by $T$. 
We assume that $\|\alpha-\hat{\alpha}\|_\infty \leq \epsilon/(40Dn^5)$ for some $\epsilon < 1$.
Suppose we run the procedure~\ref{alg:interpolation} with input $T,\hat{T},\alpha,\hat{\alpha}$.
Let $T^{m},T^{m+1}$ be two arbitrary consecutive topologies in this process, corresponding to move $m$.
Let $U_{m}$ be the set of quartets where $T^{m},T^{m+1}$ disagree.
Then,

\begin{equation}\label{eq:graph_difference}
\sum_{x \in \{-1,1\}^n} \lp|f_x^{T^{m+1}}(\alpha) - f_x^{T^{m}}(\alpha)\rp| \leq \frac{|U_{m}|\epsilon}{Dn^4}
\end{equation}

\end{lemma}

\begin{proof}
We are going to prove this inductively on the total number of moves $m$.
Suppose we are in the first move, $m = 1$ of round $q = 1$ and epoch $t = 1$. By applying
Lemma~\ref{l:lip_expression},
we have that 
\begin{align*}
\sum_{x \in \{-1,1\}^n} \lp|f_x^{T^{1}}(\alpha) - f_x^{T^{0}}(\alpha)\rp| &= 
\sum_{x \in \{-1,1\}^n} \lp|\sum_{\{w,z,y,u\} \in U_{1}}x_wx_zx_yx_u f_x^{T^{1}\setminus\{w,z,y,u\}}(\alpha)\rp|\\
&\leq \sum_{\{w,z,y,u\} \in U_{111}} \Delta_{w,z,y,u}(\alpha)
 \sum_{x \in \{-1,1\}^n}
 \lp|f_x^{T^{1}\setminus\{w,z,y,u\}}(\alpha)\rp|
\end{align*}
Now, notice that since there have not been any other changes to the topology of $T$ except for the first move, $T^{0} = T$ and furthermore, the expression $f_x^{T^{1}\setminus\{w,z,y,u\}}(\alpha)$
is actually the probability distribution on the leaves of a tree that has edge weights that agree with $\alpha$, except for edges that belong to some path of the quartet $\{w,z,y,u\}$, which have weight $0$. 
Hence, we can remove the absolute value and this gives us
$$
\sum_{x \in \{-1,1\}^n} \lp|f_x^{T^{1}}(\alpha) - f_x^{T^{0}}(\alpha)\rp| \leq \sum_{\{w,z,y,u\} \in U_{1}} \Delta_{w,z,y,u}(\alpha)
$$
Finally, by applying Lemma~\ref{l:bad_quartet} we get
$$
\sum_{x \in \{-1,1\}^n} \lp|f_x^{T^{1}}(\alpha) - f_x^{T^{0}}(\alpha)\rp| \leq |U_{1}|
20 n \cdot \frac{\epsilon}{40 Dn^5} = \frac{\epsilon |U_{1}|}{2Dn^4}
$$
hence the base case is true. 

Now suppose the claim holds for all moves $m' < m$.  
First, we define the vector $\alpha^{wzyu}\in [0,1]^{n \choose 2}$ as
$$
\alpha^{wzyu}_{kl} = 
\left\{
\begin{array}{ll}
      \alpha_{kl} \text{, if the path $P_{kl}$ does not have common edges with any of the paths in the quartet $\{w,z,y,u\}$}\\
      0 \text{, otherwise}
\end{array} 
\right. 
$$
Clearly, $\alpha^{wzyu}$ is also a metric on $T$, which is induced by the same weights as $\alpha$, except that all edges on paths of the quartet $\{w,z,y,u\}$ have weight $0$. 
By again applying Lemma~\ref{l:lip_expression} and Lemma~\ref{l:bad_quartet}, we can get 
\begin{align}\label{eq:difference}
\sum_{x \in \{-1,1\}^n} \lp|f_x^{T^{m+1}}(\alpha) - f_x^{T^{m}}(\alpha)\rp| &\leq \sum_{\{w,z,y,u\} \in U_{m}} \Delta_{w,z,y,u}(\alpha) \sum_{x \in \{-1,1\}^n} \lp|f_x^{T^{m}\setminus\{w,z,y,u\}}(\alpha^{wzyu})\rp| \nonumber \\
&\leq \frac{\epsilon}{2Dn^4} \sum_{\{w,z,y,u\} \in U_{m}} \lp|f_x^{T^{m}\setminus\{w,z,y,u\}}(\alpha^{wzyu})\rp| 
\end{align}
Now, let's fix a quartet $\{w,z,y,u\} \in U_{qrs}$. The graph $T^{m}\setminus\{w,z,y,u\}$ is a tree where all edges that belong to some path of the quartet $\{w,z,y,u\}$ have been removed. 
The problem is that other changes have happened in the topology of $T$ before it reaches the current state $T^{m}$. Therefore, the quantity $f_x^{T^{m}\setminus\{w,z,y,u\}}(\alpha^{wzyu})$ is no longer a distribution over a tree, since $\alpha^{wzyu}$ corresponds to the initial tree metric on $T$, with some edges set to $0$. 
Hence, we cannot get rid of the absolute value and claim that this quantity sums up to $1$. 
Instead,
our strategy will be to relate this quantity to some other quantity that is a probability distribution. 
To describe this probability distribution, consider the collection of subtrees that are obtained from $T^{m}$ by removing all paths of the quartet $\{w,z,y,u\}$. 
These partition the set of leaves into subsets $S_i$, one for each subtree. 
Let $G_{m}$ be the forest that is obtained by taking for each subset on leaves $S_i$ the subtree induced by $T$(when we say induced, it is implicit that the function \textsc{Binary} is applied to make the subtree have all non-leaves with degree $3$). We will show how to relate $f_x^{T^{m}\setminus\{w,z,y,u\}}(\alpha^{wzyu})$ with
$f_x^{G_{m}}(\alpha^{wzyu})$. 
Notice that by definition, $\alpha^{wzyu}$ is clearly a metric induced from $G_{m}$ and so the latter quantity is a probability distribution.

Our strategy for relating these two quantities will be to interpolate between $T^{m}\setminus\{w,z,y,u\}$ to $G_{m}$. 
The way to do this interpolation is using Algorithm~\ref{alg:interpolation}.
In particular, the following Lemma shows that in order for Algorithm~\ref{alg:interpolation} to transform $G_{m}$ to $T^{m}\setminus\{w,z,y,u\}$, it will need strictly less moves than the ones needed to transform $T$ to $T^{m+1}$. 

\begin{lemma}\label{l:epoch_transform}
Let $M$ be the number of moves needed for Algorithm~\ref{alg:interpolation} to transform $T$ to $T^{m+1}$. Then, it is possible to transform $G_{m}$ to $T^{m}\setminus\{w,z,y,u\}$ using a number of moves that is strictly smaller than $M$. 
\end{lemma}
\begin{proof}
The idea of the proof is very simple and relies on the fact that we can simply "copy" the moves made from $T$ to $T^{m+1}$, except when these moves aim at making a cherry with two leaves that belong to different subtrees of $G_m$, in which case no move is necessary. 
To be more formal, let $R$ be the number of epochs that Algorithm~\ref{alg:interpolation} needs to reach $T^{m+1}$ starting from $T$. Then, we will show that we can reach $T^m \setminus\{w,z,y,u\}$ using a number of epochs $R'$ such that $R' < R$. Furthermore, we will argue that each of the $R'$ epochs has at most the same number of moves as the corresponding one starting from $R$. 
We do this by examining one by one the $R$ epochs from $T$ to $T^{m+1}$ and deciding how to potentially change it.
First of all, let's remember that at the start of each epoch $t$, Algorithm~\ref{alg:interpolation} chooses two nodes $i_{t},j_{t}$ and makes them siblings with parent $p_{t}$, thus making a larger subtree that agrees with $\hat{T}$. The relative topology inside this subtree will never be altered by the algorithm again. We call this process \emph{fixing} the subtree with root $p_{t}$. 
When we refer to the subtree of $i_t$ and $j_t$ we mean the connected component that results when we remove path $P_{i_tj_t}$ from the graph at epoch $t$. 
We denote $\{i_t,j_t\}_t$ the sequence of epochs produced from $T$ to $T^{qrs}$ and $\{i_t',j_t'\}_t$ the sequence of epochs that transforms $G_{m}$ to $T^{m}\setminus\{w,z,y,u\}$. 

We will inductively prove that at any epoch  in the sequence $\{i_t,j_t\}$, if a subtree with root $p_{t}$ has been fixed after that epoch and if this subtree is contained in some component of $G_m$, then
this subtree will also be fixed under sequence $\{i_t',j_t'\}$. 
In proving the inductive step, we will also describe how to define the sequence of epochs $\{i_t',j_t'\}$. 
After proving this claim, we will explain why it implies the statement of the Lemma. 

Suppose the claim holds for all epochs prior to $t$(for $t=1$ the claim is trivial). 
Suppose then that at epoch $t$ $i_t$ and $j_t$ become cherries with parent $p_t$. Let $T^t$ be the topology at the start of the epoch and let $G_{m}^t$ be the corresponding topology at the start of epoch $t$ under the sequence $\{i_t',j_t'\}$. Suppose first that the subtrees of $i_t,j_t$ belong to the same component of $T^{qrs}$. Then, inductively, we know that $i_t,j_t$ exist also in $G_{m}^t$ and their subtrees have already been fixed by the sequence $\{i'_t,j'_t\}$.  In that case, we set $i_t' = i_t,j_t' = j_t$ and set the movement of $i_t'$ to $j_t'$ to be the same as the one from $i_t$ to $j_t$, but on the induced component of $G_m^t$ that $i_t,j_t$ belong to. We call this a \emph{true} epoch. Since the paths in an induced subtree can only stay the same or become smaller than the ones in the original tree (after applying operation \textsc{Binary}), the number of moves required to move $i_t'$ to $j_t'$ is at most the number of moves required to move $i_t$ to $j_t$. Once we move $i_t'$ to become sibling with $j_t'$, the new subtree with root $p_t'$ has also been fixed for the sequence $\{i_t',j_t'\}$, proving the inductive hypothersis in that case. Now, suppose that $i_t,j_t$ belong in different subtrees of $G_m$. There are two cases: either both $i_t,j_t$ exists as nodes in $G_m^t$, or at least one of them does not exist. If they both exist, then again by the inductive hypothesis, it follows that the subtrees  $i_t,j_t$ have also been fixed in $G_m^t$. In that case, it must be the case that the entire component of $i_t$ in $G_m^t$ is equal to that subtree (otherwise we would be able to connect $i_t$ to some other sibling and enlarge it). Hence, in that case no movement takes place and we trivially set $i_t' = j_t' = i_t$ to denote that this is not a true epoch. 
The point is that there is no need to move them again until we reach $T^m\setminus\{w,z,y,u\}$, so our choice not to move them is correct.
Now, let's examine the case that either $i_t$ or $j_t$ does not exist in $G_m^t$. Suppose $i_t$ does not exist w.l.o.g. Then, 
this means that it is a parent of two subtrees that do not belong to the same component of $G_m^t$. Thus,
this means that there is no reason to connect these subtrees, hence we also trivially set $i_t' = j_t' = i_t$. We call this epoch \emph{fake}.
The inductive step is now complete.

The induction we just proved shows that the sequence $\{i_t',j_t'\}$ leads to $T^{m}\setminus\{w,z,y,u\}$ when started from $G_{m}$. 
Also, it is clear that the true number of epochs in $\{i_t',j_t'\}$ at any given time is at most the ones in $\{i_t,j_t\}$, since some epochs might be fake. In fact, if $\{i_t,j_t\}$ reaches $T^{m+1}$ at epoch $R$, the number of true epochs $R'$ in $\{i_t',j_t'\}$ should be strictly smaller than $R$. The reason is that at epoch $R$, $i_R$ and $j_R$ belong to different components of $T^{qrs}\setminus\{w,z,y,u\}$ by definition (since we remove the path from $w$ to $u$). 
Hence, the last epoch $R$ will not be a true epoch for $\{i_t',j_t'\}$. 
Since we have also argued that the number of moves in epochs of $\{i_t',j_t'\}$ is at most
the corresponding number for epochs in $\{i_t,j_t\}$,this concludes the proof of the Lemma.

\end{proof}

Let $M$ be the total number of moves required by Algorithm~\ref{alg:interpolation} to transform $T$ to $T^{m+1}$ and $M'$ the moves to transform $G_{m}$ to $T^{m}\setminus \{w,z,y,u\}$. 
The point of Lemma~\ref{l:epoch_transform} is that $M' < M$. 
Let $G_{m} = G_{m}^0, G_{m}^1,\ldots, G_{m}^{M'} = T^{m}\setminus \{w,z,y,u\}$ be the sequence of graphs in the interpolation. 
By triangle inequality, we have
\begin{align*}
\lp|f_x^{T^{m}\setminus\{w,z,y,u\}}(\alpha^{wzyu})\rp| &\leq \lp|f_x^{G_{m}^0}(\alpha^{wzyu})\rp| + \sum_{s= 1}^{M'}  \lp|f_x^{G_{m}^{s}}(\alpha) - f_x^{G_{m}^{s-1}}(\alpha)\rp|
\end{align*}
As we have already explained, the first term on the right hand side corresponds to a distribution, hence we can remove the absolute values. The remaining terms have the form of 
the left hand side of \eqref{eq:graph_difference}, which is what we want to bound in general. However, these differences are applied to graphs that are obtained after at most $M'$ moves of Algorithm~\ref{alg:interpolation}. Hence, we can apply the inductive hypothesis \eqref{eq:graph_difference}. If $U_{s}'$ is the set of quartets that change from $G^{s-1}_{m}$ to $G^{s}_{m}$, then, 
$$
\sum_{x \in \{-1,1\}^n} \sum_{s= 1}^{M'}  \lp|f_x^{G_{m}^{s}}(\alpha) - f_x^{G_{m}^{s-1}}(\alpha)\rp| \leq
\frac{\epsilon}{Dn^4}\sum_{s= 1}^{M'}  |U_{s}'|
$$
It remains to bound the sum $\sum_{s= 1}^{M'}  |U_{s}'|$. This is equal to the total number of quartets that have changed topology until move $M'$, starting from $G_{m}$ (if a quartet has changed multiple times, we count the number of times it has changed in this sum). 
We argue that 
$$
\sum_{s= 1}^{M'}  |U_{s}'| \leq Dn^4
$$
The reason is the following: there is a total of $n \choose 4$ quartets, so it suffices to bound the number of times that any specific quartet $\{w,z,y,u\}$ changes topology. 
First of all, we have already argued that a quartet changes topology at most once every epoch.
In order for a quartet to change topology during some epoch, at least one of it's leaves should be moved to some different position.
By Observation~\ref{ob:diam} we know that a leaf is moved at most $\lceil D/2\rceil$ times in total. Hence, a quartet changes topology at most $4\lceil D/2\rceil$ times in total. 
Hence,
$$
\sum_{s= 1}^{M'}  |U_{s}|  \leq {n \choose 4} 2D \leq Dn^4
$$
Combining everything together, we get
$$
\sum_{x \in \{-1,1\}^n} \lp|f_x^{T^{m}\setminus\{w,z,y,u\}}(\alpha^{wzyu})\rp| \leq 
\sum_{x \in \{-1,1\}^n} f_x^{G_{m}^0}(\alpha^{wzyu}) + \frac{\epsilon}{Dn^4} Dn^4 = 1 + \epsilon
$$
This holds for all $\{w,z,y,u\} \in U_{m}$. Hence, by using \eqref{eq:difference} we get
\begin{align*}
\sum_{x \in \{-1,1\}^n} \lp|f_x^{T^{m+1}}(\alpha) - f_x^{T^{m}}(\alpha)\rp| 
&\leq  \frac{\epsilon}{2Dn^4} \sum_{\{w,z,y,u\} \in U_{m}} \lp|f_x^{T^{m}\setminus\{w,z,y,u\}}(\alpha^{wzyu})\rp| \\ 
&\leq \frac{\epsilon}{2Dn^4} |U_{m}|(1 + \epsilon) \leq \frac{|U_{m}|\epsilon}{Dn^4}
\end{align*}
since $\epsilon \leq 1$. 
This is the inductive claim that we wanted to prove.

\end{proof}

We are now ready to conclude the proof of Theorem~\ref{thm:probabilistic_unknown}. To do it, we simply use Lemma~\ref{l:interpolate_bound} to transition from $T$ to $\hat{T}$. Then, we use the bound for the fixed topology to change $\alpha$ to $\hat{\alpha}$. 


\begin{proof}[Proof of Theorem~\ref{thm:probabilistic_unknown}]
We can assume without loss of generality that
$\hat{T}$ has a smaller diameter than $T$, otherwise we just reverse the roles of $T,\hat{T}$. We run Algorithm~\ref{alg:interpolation} with input $T,\hat{T},\alpha$, which produces a sequence $T = T^0 , T^2, \ldots, T^M = \hat{T}$, where each element of the sequence corresponds to some move. 
We have that 
\begin{align}\label{eq:tv_bound}
TV(\mu,\hat{\mu}) &= \sum_{x \in \{-1,1\}^n} \lp|f_x^T(\alpha) - f_x^{\hat{T}}(\hat{\alpha})\rp|\nonumber\\
&\leq \sum_{x \in \{-1,1\}^n} \lp|f_x^T(\alpha) - f_x^{\hat{T}}(\alpha)\rp| + 
\sum_{x \in \{-1,1\}^n} \lp|f_x^{\hat{T}}(\alpha) - f_x^{\hat{T}}(\hat{\alpha})\rp| 
\end{align}
Define $\epsilon' = 40 Dn^5 \epsilon$. 
Let us divide into cases.

\emph{Case 1:} Suppose $\epsilon' < 1$. 
Then, the first term of the RHS of \eqref{eq:tv_bound} can be bounded using the successive steps of the interpolation process. In particular, since $\epsilon = \epsilon'/(40Dn^5)$, we can apply Lemma~\ref{l:interpolate_bound} to get
\begin{align*}
\sum_{x \in \{-1,1\}^n} \lp|f_x^T(\alpha) - f_x^{\hat{T}}(\alpha)\rp| \leq 
\sum_{m = 1}^M \sum_{x \in \{-1,1\}^n}  \lp|f_x^{T^m}(\alpha) - f_x^{T^{m-1}}(\alpha)\rp|
\leq \sum_{m = 1}^M \frac{|U_m|\epsilon'}{Dn^4}
\end{align*}

By the proof of Lemma~\ref{l:interpolate_bound}, this implies that
$$
\sum_{x \in \{-1,1\}^n} \lp|f_x^T(\alpha) - f_x^{\hat{T}}(\alpha)\rp| \leq \frac{\epsilon'}{Dn^4} Dn^4 = \epsilon' = 40Dn^5\epsilon
$$
As for the second term of the RHS of \eqref{eq:tv_bound}, it is essentially the difference when we substitute $\hat{\alpha}$ with $\alpha$ in the fixed topology $\hat{T}$. Hence, we can directly apply Theorem~\ref{thm:probabilistic_fixed} to get 
$$
\sum_{x \in \{-1,1\}^n} \lp|f_x^{\hat{T}}(\alpha) - f_x^{\hat{T}}(\hat{\alpha})\rp|  \leq 2n^2\epsilon
$$
Overall, this gives us 
$$
TV(\mu,\hat{\mu}) \leq 42Dn^5 \epsilon
$$
which proves inequality \eqref{eq:tensorization} in that case.

\emph{Case 2:} Assume $\epsilon' \geq 1$. Then, 
$$
TV(\mu,\hat{\mu}) \leq 1 \leq 40Dn^5 \epsilon
$$
which means that \eqref{eq:tensorization} trivially holds in that case too.
The proof is now complete.

\end{proof}

\section{Proof of Theorem~\ref{thm:alg} (unknown topology)} \label{sec:pr-alg-unknown}

\subsection{Outline}
In this section, we will present an algorithm that takes samples from the leaves of some tree Ising model with tree $T^*$ and weight $\theta^*$ and estimates a topology $\hat T$ together with weights on the edges, so that the distributions on the leaves of $T$ and $\hat T$ are $\epsilon$-close in TV distance.
The number of samples will be polynomial in $n$ and $1/\epsilon$.

We will use the results in \cite{daskalakis2009phylogenies} about learning phylogenetic trees without assuming any upper/lower bounds on the edge weights.
We will show how we can use the guarantees of this prior work to obtain an algorithm for finding a good enough topology. Before explaining the result formally, we first want to give an intuition. We discuss how the output might be different from the original tree.
Firstly, without sufficient samples, it is impossible to determine the existence of edges that are far away from leaves. In such cases, the algorithm of \cite{daskalakis2009phylogenies} will omit those edges and output a forest, as shown in Figure~\ref{fig:appendix1}.

\begin{figure}[H]
\centering
\resizebox{\linewidth}{!}{%
\begin{tikzpicture}
\node[draw, circle, minimum size=15pt, inner sep=2pt] at (0,0) (v1) {};
\node[draw, circle, minimum size=15pt, inner sep=2pt] at ($(v1) + (-1.75,-2)$) (v2) {};
\node[draw, circle, minimum size=15pt, inner sep=2pt] at ($(v1) + (1.75,-2)$) (v3) {};
\node[draw, circle, minimum size=15pt, inner sep=2pt] at ($(v2) + (-1,-2)$) (v4) {\small $1$};
\node[draw, circle, minimum size=15pt, inner sep=2pt] at ($(v2) + (1,-2)$) (v5) {\small $2$};
\node[draw, circle, minimum size=15pt, inner sep=2pt] at ($(v3) + (-1,-2)$) (v6) {\small $3$};
\node[draw, circle, minimum size=15pt, inner sep=2pt] at ($(v3) + (1,-2)$) (v7) {\small $4$};
\node[draw, circle, minimum size=15pt, inner sep=2pt] at (8,0) (v8) {};
\node[draw, circle, minimum size=15pt, inner sep=2pt] at ($(v8) + (-1.75,-2)$) (v9) {};
\node[draw, circle, minimum size=15pt, inner sep=2pt] at ($(v8) + (1.75,-2)$) (v10) {};
\node[draw, circle, minimum size=15pt, inner sep=2pt] at ($(v9) + (-1,-2)$) (v11) {\small $5$};
\node[draw, circle, minimum size=15pt, inner sep=2pt] at ($(v9) + (1,-2)$) (v12) {\small $6$};
\node[draw, circle, minimum size=15pt, inner sep=2pt] at ($(v10) + (-1,-2)$) (v13) {\small $7$};
\node[draw, circle, minimum size=15pt, inner sep=2pt] at ($(v10) + (1,-2)$) (v14) {\small $8$};

\draw[thick, red] (v1) -- (v2);
\draw[thick, red] (v1) -- (v3);
\draw[thick, red] (v2) -- (v4);
\draw[thick, red] (v2) -- (v5);
\draw[thick, red] (v3) -- (v6);
\draw[thick, red] (v3) -- (v7);
\draw[thick, dashed] (v1) -- (v8);
\draw[thick, dashed] (v8) -- (v9);
\draw[thick, dashed] (v8) -- (v10);
\draw[thick, blue] (v9) -- (v11);
\draw[thick, blue] (v9) -- (v12);
\draw[thick, black!20!green] (v10) -- (v13);
\draw[thick, black!20!green] (v10) -- (v14);
\end{tikzpicture}
}
\caption{A forest that is created by deleting edges: the true tree contains all the solid and dashed edges. Yet, the algorithm of \cite{daskalakis2009phylogenies} might not be able to identify some of the edges because they do not sufficiently correlate with the leaves, and so it will return a forest. In the example here, the forest is obtained from the original tree by removing the dashed edges. The three connected components of the output forest are colored red, blue and green.}
\label{fig:appendix1}
\end{figure}
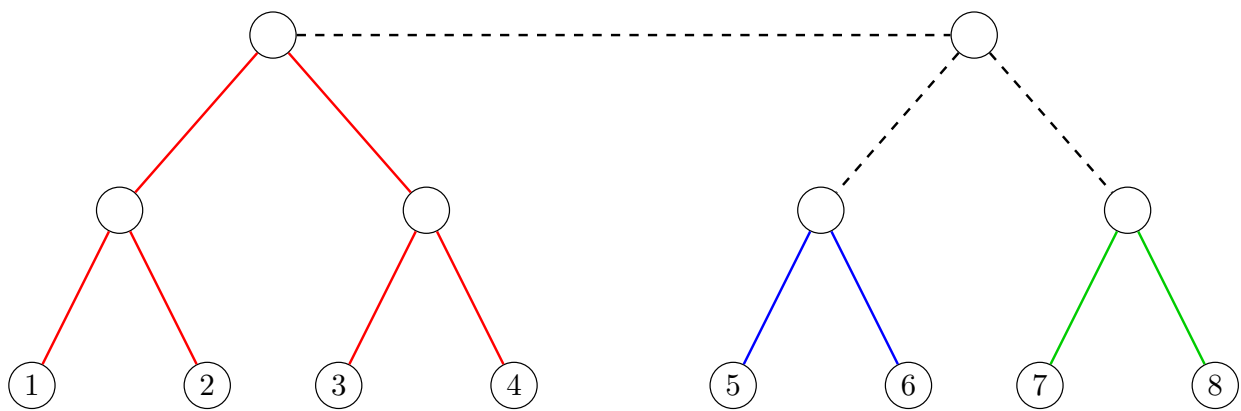

Secondly, the algorithm may fail to split the tree in a topologically sensible way. This means that it will return a forest, yet, in contrast with the example in Figure~\ref{fig:appendix1}, it is \emph{impossible} to obtain this forest by cutting some edges in the ground truth tree.
Still, the topology \emph{within each connected component} in this forest is preserved,
as illustrated in Figure~\ref{fig:appendix2} (a)-(b).
After splitting the tree into two subtrees, those subtrees might contain some internal nodes of degree $2$. Such nodes cannot be identified from the leaves, and they will be contracted, as shown in Figure~\ref{fig:appendix2} (c).

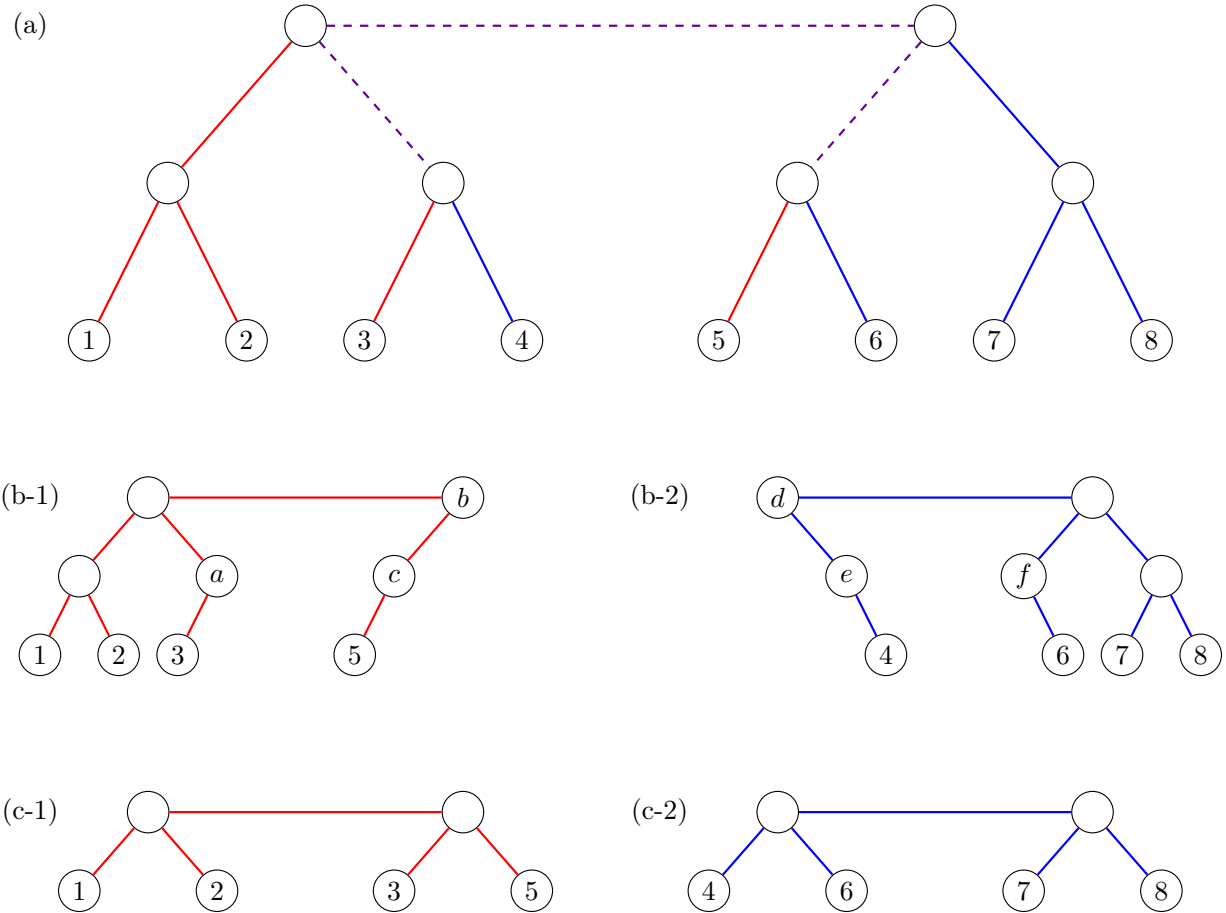
\begin{figure}[H]
\centering
\resizebox{\linewidth}{!}{%
\begin{tikzpicture}
%
%
\node[draw, circle, minimum size=15pt, inner sep=2pt] at (0,0) (v1-a) {};
\node[draw, circle, minimum size=15pt, inner sep=2pt] at ($(v1-a) + (-1.75,-2)$) (v2-a) {};
\node[draw, circle, minimum size=15pt, inner sep=2pt] at ($(v1-a) + (1.75,-2)$) (v3-a) {};
\node[draw, circle, minimum size=15pt, inner sep=2pt] at ($(v2-a) + (-1,-2)$) (v4-a) {\small $1$};
\node[draw, circle, minimum size=15pt, inner sep=2pt] at ($(v2-a) + (1,-2)$) (v5-a) {\small $2$};
\node[draw, circle, minimum size=15pt, inner sep=2pt] at ($(v3-a) + (-1,-2)$) (v6-a) {\small $3$};
\node[draw, circle, minimum size=15pt, inner sep=2pt] at ($(v3-a) + (1,-2)$) (v7-a) {\small $4$};
\node[draw, circle, minimum size=15pt, inner sep=2pt] at (8,0) (v8-a) {};
\node[draw, circle, minimum size=15pt, inner sep=2pt] at ($(v8-a) + (-1.75,-2)$) (v9-a) {};
\node[draw, circle, minimum size=15pt, inner sep=2pt] at ($(v8-a) + (1.75,-2)$) (v10-a) {};
\node[draw, circle, minimum size=15pt, inner sep=2pt] at ($(v9-a) + (-1,-2)$) (v11-a) {\small $5$};
\node[draw, circle, minimum size=15pt, inner sep=2pt] at ($(v9-a) + (1,-2)$) (v12-a) {\small $6$};
\node[draw, circle, minimum size=15pt, inner sep=2pt] at ($(v10-a) + (-1,-2)$) (v13-a) {\small $7$};
\node[draw, circle, minimum size=15pt, inner sep=2pt] at ($(v10-a) + (1,-2)$) (v14-a) {\small $8$};

\draw[thick, red] (v1-a) -- (v2-a);
\draw[thick, blue!60!red, dashed] (v1-a) -- (v3-a);
\draw[thick, red] (v2-a) -- (v4-a);
\draw[thick, red] (v2-a) -- (v5-a);
\draw[thick, red] (v3-a) -- (v6-a);
\draw[thick, blue] (v3-a) -- (v7-a);
\draw[thick, blue!60!red, dashed] (v1-a) -- (v8-a);
\draw[thick, blue!60!red, dashed] (v8-a) -- (v9-a);
\draw[thick, blue] (v8-a) -- (v10-a);
\draw[thick, red] (v9-a) -- (v11-a);
\draw[thick, blue] (v9-a) -- (v12-a);
\draw[thick, blue] (v10-a) -- (v13-a);
\draw[thick, blue] (v10-a) -- (v14-a);

%
%
\node[draw, circle, minimum size=15pt, inner sep=2pt] at (-2,-6) (v1-b1) {};
\node[draw, circle, minimum size=15pt, inner sep=2pt] at ($(v1-b1) + 0.5*(-1.75,-2)$) (v2-b1) {};
\node[draw, circle, minimum size=15pt, inner sep=2pt] at ($(v1-b1) + 0.5*(1.75,-2)$) (v3-b1) {\small $a$};
\node[draw, circle, minimum size=15pt, inner sep=2pt] at ($(v2-b1) + 0.5*(-1,-2)$) (v4-b1) {\small $1$};
\node[draw, circle, minimum size=15pt, inner sep=2pt] at ($(v2-b1) + 0.5*(1,-2)$) (v5-b1) {\small $2$};
\node[draw, circle, minimum size=15pt, inner sep=2pt] at ($(v3-b1) + 0.5*(-1,-2)$) (v6-b1) {\small $3$};
\node[draw, circle, minimum size=15pt, inner sep=2pt] at (2,-6) (v8-b1) {\small $b$};
\node[draw, circle, minimum size=15pt, inner sep=2pt] at ($(v8-b1) + 0.5*(-1.75,-2)$) (v9-b1) {\small $c$};
\node[draw, circle, minimum size=15pt, inner sep=2pt] at ($(v9-b1) + 0.5*(-1,-2)$) (v11-b1) {\small $5$};

\draw[thick, red] (v1-b1) -- (v2-b1);
\draw[thick, red] (v1-b1) -- (v3-b1);
\draw[thick, red] (v2-b1) -- (v4-b1);
\draw[thick, red] (v2-b1) -- (v5-b1);
\draw[thick, red] (v3-b1) -- (v6-b1);
\draw[thick, red] (v1-b1) -- (v8-b1);
\draw[thick, red] (v8-b1) -- (v9-b1);
\draw[thick, red] (v9-b1) -- (v11-b1);

\node[draw, circle, minimum size=15pt, inner sep=2pt] at (6,-6) (v1-b2) {\small $d$};
\node[draw, circle, minimum size=15pt, inner sep=2pt] at ($(v1-b2) + 0.5*(1.75,-2)$) (v3-b2) {\small $e$};
\node[draw, circle, minimum size=15pt, inner sep=2pt] at ($(v3-b2) + 0.5*(1,-2)$) (v7-b2) {\small $4$};
\node[draw, circle, minimum size=15pt, inner sep=2pt] at (10,-6) (v8-b2) {};
\node[draw, circle, minimum size=15pt, inner sep=2pt] at ($(v8-b2) + 0.5*(-1.75,-2)$) (v9-b2) {\small $f$};
\node[draw, circle, minimum size=15pt, inner sep=2pt] at ($(v8-b2) + 0.5*(1.75,-2)$) (v10-b2) {};
\node[draw, circle, minimum size=15pt, inner sep=2pt] at ($(v9-b2) + 0.5*(1,-2)$) (v12-b2) {\small $6$};
\node[draw, circle, minimum size=15pt, inner sep=2pt] at ($(v10-b2) + 0.5*(-1,-2)$) (v13-b2) {\small $7$};
\node[draw, circle, minimum size=15pt, inner sep=2pt] at ($(v10-b2) + 0.5*(1,-2)$) (v14-b2) {\small $8$};

\draw[thick, blue] (v1-b2) -- (v3-b2);
\draw[thick, blue] (v3-b2) -- (v7-b2);
\draw[thick, blue] (v1-b2) -- (v8-b2);
\draw[thick, blue] (v8-b2) -- (v9-b2);
\draw[thick, blue] (v8-b2) -- (v10-b2);
\draw[thick, blue] (v9-b2) -- (v12-b2);
\draw[thick, blue] (v10-b2) -- (v13-b2);
\draw[thick, blue] (v10-b2) -- (v14-b2);

%
%
\node[draw, circle, minimum size=15pt, inner sep=2pt] at (-2,-10) (v1-c1) {};
\node[draw, circle, minimum size=15pt, inner sep=2pt] at ($(v1-c1) + 0.5*(-1.75,-2)$) (v2-c1) {\small $1$};
\node[draw, circle, minimum size=15pt, inner sep=2pt] at ($(v1-c1) + 0.5*(1.75,-2)$) (v3-c1) {\small $2$};
\node[draw, circle, minimum size=15pt, inner sep=2pt] at (2,-10) (v4-c1) {};
\node[draw, circle, minimum size=15pt, inner sep=2pt] at ($(v4-c1) + 0.5*(-1.75,-2)$) (v5-c1) {\small $3$};
\node[draw, circle, minimum size=15pt, inner sep=2pt] at ($(v4-c1) + 0.5*(1.75,-2)$) (v6-c1) {\small $5$};

\draw[thick, red] (v1-c1) -- (v2-c1);
\draw[thick, red] (v1-c1) -- (v3-c1);
\draw[thick, red] (v1-c1) -- (v4-c1);
\draw[thick, red] (v4-c1) -- (v5-c1);
\draw[thick, red] (v4-c1) -- (v6-c1);

\node[draw, circle, minimum size=15pt, inner sep=2pt] at (6,-10) (v1-c2) {};
\node[draw, circle, minimum size=15pt, inner sep=2pt] at ($(v1-c2) + 0.5*(-1.75,-2)$) (v2-c2) {\small $4$};
\node[draw, circle, minimum size=15pt, inner sep=2pt] at ($(v1-c2) + 0.5*(1.75,-2)$) (v3-c2) {\small $6$};
\node[draw, circle, minimum size=15pt, inner sep=2pt] at (10,-10) (v4-c2) {};
\node[draw, circle, minimum size=15pt, inner sep=2pt] at ($(v4-c2) + 0.5*(-1.75,-2)$) (v5-c2) {\small $7$};
\node[draw, circle, minimum size=15pt, inner sep=2pt] at ($(v4-c2) + 0.5*(1.75,-2)$) (v6-c2) {\small $8$};

\draw[thick, blue] (v1-c2) -- (v2-c2);
\draw[thick, blue] (v1-c2) -- (v3-c2);
\draw[thick, blue] (v1-c2) -- (v4-c2);
\draw[thick, blue] (v4-c2) -- (v5-c2);
\draw[thick, blue] (v4-c2) -- (v6-c2);

%
%
\node[] at (-3.5,0) {\small (a)};
\node[] at (-3.5,-6) {\small (b-1)};
\node[] at (4.5,-6) {\small (b-2)};
\node[] at (-3.5,-10) {\small (c-1)};
\node[] at (4.5,-10) {\small (c-2)};
\end{tikzpicture}
}
\caption{Splitting a tree into subtrees in a more complicated  fashion: In (a), the true tree contains both the solid and dashed edges. Yet, the algorithm might not identify completely the topology. Instead it will return two subtrees. Yet, the subtrees are not topologically sensible: one tree will contain the leaves $\{1,2,3,5\}$ and the other will contain $\{4,6,7,8\}$. The red edges correspond to the first tree, the blue edges to the second tree, while the dashed edges are shared by both trees. In figure (b), we split the original tree into the two subtrees, one containing the red and dashed edges and the other the blue and dashed edges. Notice that in Figure (b), the nodes labeled $a$-$f$ have degree $2$. Information theoretically, it is impossible to identify hidden nodes of degree $2$. Indeed, the same leaf distribution is obtained by removing each of these nodes, connecting its neighbors, and adjusting the weight of the new edges. Hence, the output of the algorithm will not have degree-$2$ nodes: instead, the transformation described above, that removes degree-2 nodes, will be performed. In (c) we demonstrate the result of removing such degree-$2$ nodes.}
\label{fig:appendix2}
\end{figure}

The two transformations that are applied to a tree in Figure~\ref{fig:appendix1} and Figure~\ref{fig:appendix2} can be viewed as a single transformation: separating the tree into subtrees, while preserving the topology within each subtree. This can be seen in Figure~\ref{fig:appendix3}.

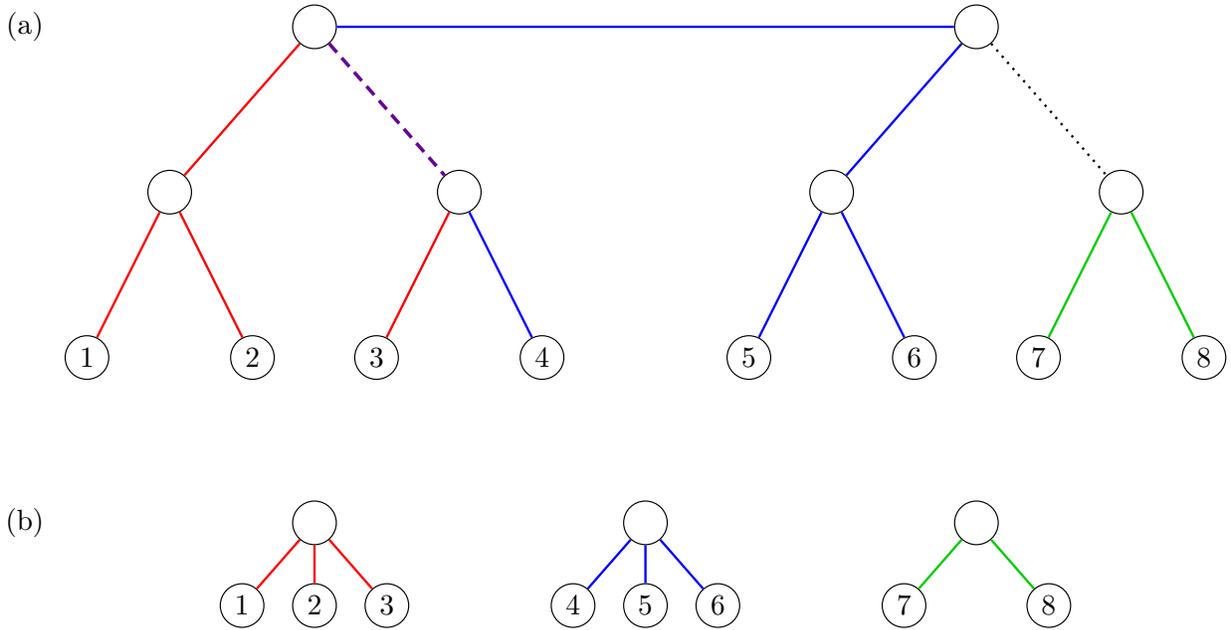
\begin{figure}[H]
\centering
\resizebox{\linewidth}{!}{%
\begin{tikzpicture}
%
%
\node[draw, circle, minimum size=15pt, inner sep=2pt] at (0,0) (v1-a) {};
\node[draw, circle, minimum size=15pt, inner sep=2pt] at ($(v1-a) + (-1.75,-2)$) (v2-a) {};
\node[draw, circle, minimum size=15pt, inner sep=2pt] at ($(v1-a) + (1.75,-2)$) (v3-a) {};
\node[draw, circle, minimum size=15pt, inner sep=2pt] at ($(v2-a) + (-1,-2)$) (v4-a) {\small $1$};
\node[draw, circle, minimum size=15pt, inner sep=2pt] at ($(v2-a) + (1,-2)$) (v5-a) {\small $2$};
\node[draw, circle, minimum size=15pt, inner sep=2pt] at ($(v3-a) + (-1,-2)$) (v6-a) {\small $3$};
\node[draw, circle, minimum size=15pt, inner sep=2pt] at ($(v3-a) + (1,-2)$) (v7-a) {\small $4$};
\node[draw, circle, minimum size=15pt, inner sep=2pt] at (8,0) (v8-a) {};
\node[draw, circle, minimum size=15pt, inner sep=2pt] at ($(v8-a) + (-1.75,-2)$) (v9-a) {};
\node[draw, circle, minimum size=15pt, inner sep=2pt] at ($(v8-a) + (1.75,-2)$) (v10-a) {};
\node[draw, circle, minimum size=15pt, inner sep=2pt] at ($(v9-a) + (-1,-2)$) (v11-a) {\small $5$};
\node[draw, circle, minimum size=15pt, inner sep=2pt] at ($(v9-a) + (1,-2)$) (v12-a) {\small $6$};
\node[draw, circle, minimum size=15pt, inner sep=2pt] at ($(v10-a) + (-1,-2)$) (v13-a) {\small $7$};
\node[draw, circle, minimum size=15pt, inner sep=2pt] at ($(v10-a) + (1,-2)$) (v14-a) {\small $8$};

\draw[thick, red] (v1-a) -- (v2-a);
\draw[very thick, blue!60!red, dash pattern=on 4pt off 2.5pt] (v1-a) -- (v3-a);
\draw[thick, red] (v2-a) -- (v4-a);
\draw[thick, red] (v2-a) -- (v5-a);
\draw[thick, red] (v3-a) -- (v6-a);
\draw[thick, blue] (v3-a) -- (v7-a);
\draw[thick, blue] (v1-a) -- (v8-a);
\draw[thick, blue] (v8-a) -- (v9-a);
\draw[thick, dotted] (v8-a) -- (v10-a);
\draw[thick, blue] (v9-a) -- (v11-a);
\draw[thick, blue] (v9-a) -- (v12-a);
\draw[thick, black!20!green] (v10-a) -- (v13-a);
\draw[thick, black!20!green] (v10-a) -- (v14-a);

%
%
\node[draw, circle, minimum size=15pt, inner sep=2pt] at (0,-6) (v1-b) {};
\node[draw, circle, minimum size=15pt, inner sep=2pt] at ($(v1-b) + 0.5*(-1.75,-2)$) (v2-b) {\small $1$};
\node[draw, circle, minimum size=15pt, inner sep=2pt] at ($(v1-b) + 0.5*(0,-2)$) (v3-b) {\small $2$};
\node[draw, circle, minimum size=15pt, inner sep=2pt] at ($(v1-b) + 0.5*(1.75,-2)$) (v4-b) {\small $3$};
\node[draw, circle, minimum size=15pt, inner sep=2pt] at (4,-6) (v5-b) {};
\node[draw, circle, minimum size=15pt, inner sep=2pt] at ($(v5-b) + 0.5*(-1.75,-2)$) (v6-b) {\small $4$};
\node[draw, circle, minimum size=15pt, inner sep=2pt] at ($(v5-b) + 0.5*(0,-2)$) (v7-b) {\small $5$};
\node[draw, circle, minimum size=15pt, inner sep=2pt] at ($(v5-b) + 0.5*(1.75,-2)$) (v8-b) {\small $6$};
\node[draw, circle, minimum size=15pt, inner sep=2pt] at (8,-6) (v9-b) {};
\node[draw, circle, minimum size=15pt, inner sep=2pt] at ($(v9-b) + 0.5*(-1.75,-2)$) (v10-b) {\small $7$};
\node[draw, circle, minimum size=15pt, inner sep=2pt] at ($(v9-b) + 0.5*(1.75,-2)$) (v11-b) {\small $8$};

\draw[thick, red] (v1-b) -- (v2-b);
\draw[thick, red] (v1-b) -- (v3-b);
\draw[thick, red] (v1-b) -- (v4-b);
\draw[thick, blue] (v5-b) -- (v6-b);
\draw[thick, blue] (v5-b) -- (v7-b);
\draw[thick, blue] (v5-b) -- (v8-b);
\draw[thick, black!20!green] (v9-b) -- (v10-b);
\draw[thick, black!20!green] (v9-b) -- (v11-b);

%
\node[] at (-3.5,0) {\small (a)};
\node[] at (-3.5,-6) {\small (b)};
\end{tikzpicture}
}
\caption{Separating a tree into multiple subtrees: In (a), the original tree contains all solid, dashed and dotted edges. It is split by the algorithm to three subtrees. The first subtree contains leaves $\{1,2,3\}$ and the red and dashed edges. The second subtree contains $\{4,5,6\}$ and the blue and dashed edges. The third contains $\{7,8\}$ and the green edges. Notice that the dashed purple edge is contained in two trees, while the dotted black edge is contained in no tree. Notice that the first and second subtrees intersect while the third subtree is disjoint from the other subtrees. In (b) we see the output of the algorithm.}
\label{fig:appendix3}
\end{figure}

Yet, the algorithm of \cite{daskalakis2009phylogenies} might not be able to tell the exact topology within each subtree. In this case, the output will just contract some of the internal edges of this subtree. See Figure~\ref{fig:appendix4}~(a)-(c) for an example of a contracted subtree. In Figure~\ref{fig:appendix4}~(d) we depict some of the possible topologies that the algorithm could have confused between, which led to contracting a specific edge.

\begin{figure}[H]
\centering
\resizebox{\linewidth}{!}{%
\begin{tikzpicture}
%
%

\node[draw, circle, minimum size=15pt, inner sep=2pt] at (0,0) (v2-a) {};
\node[draw, circle, minimum size=15pt, inner sep=2pt] at (3,0) (v3-a) {\small $a$};
\node[draw, circle, minimum size=15pt, inner sep=2pt] at ($(v2-a) + (-1,-2)$) (v4-a) {\small $1$};
\node[draw, circle, minimum size=15pt, inner sep=2pt] at ($(v2-a) + (1,-2)$) (v5-a) {\small $2$};
\node[draw, circle, minimum size=15pt, inner sep=2pt] at ($(v3-a) + (-1,-2)$) (v6-a) {\small $3$};
\node[draw, circle, minimum size=15pt, inner sep=2pt] at ($(v3-a) + (1,-2)$) (v7-a) {\small $4$};
\node[draw, circle, minimum size=15pt, inner sep=2pt] at (6,0) (v8-a) {\small $b$};
\node[draw, circle, minimum size=15pt, inner sep=2pt] at ($(v8-a) + (-1,-2)$) (v9-a) {\small $5$};
\node[draw, circle, minimum size=15pt, inner sep=2pt] at ($(v8-a) + (1,-2)$) (v10-a) {\small $6$};
\node[draw, circle, minimum size=15pt, inner sep=2pt] at (9,0) (v11-a) {};
\node[draw, circle, minimum size=15pt, inner sep=2pt] at ($(v11-a) + (-1,-2)$) (v12-a) {\small $7$};
\node[draw, circle, minimum size=15pt, inner sep=2pt] at ($(v11-a) + (1,-2)$) (v13-a) {\small $8$};

\draw[thick] (v2-a) -- (v4-a);
\draw[thick] (v2-a) -- (v5-a);
\draw[thick] (v2-a) -- (v3-a);
\draw[thick] (v3-a) -- (v6-a);
\draw[thick] (v3-a) -- (v7-a);
\draw[thick] (v3-a) -- (v8-a);
\draw[thick] (v8-a) -- (v9-a);
\draw[thick] (v8-a) -- (v10-a);
\draw[thick] (v11-a) -- (v12-a);
\draw[thick] (v11-a) -- (v13-a);
\draw[very thick, dotted] (v8-a) -- (v11-a);

%
%
\node[draw, circle, minimum size=15pt, inner sep=2pt] at (0,-4) (v2-b) {};
\node[draw, circle, minimum size=15pt, inner sep=2pt] at ($(v2-b) + (3,0)$) (v3-b) {\small $a$};
\node[draw, circle, minimum size=15pt, inner sep=2pt] at ($(v2-b) + (-1,-2)$) (v4-b) {\small $1$};
\node[draw, circle, minimum size=15pt, inner sep=2pt] at ($(v2-b) + (1,-2)$) (v5-b) {\small $2$};
\node[draw, circle, minimum size=15pt, inner sep=2pt] at ($(v3-b) + (-1,-2)$) (v6-b) {\small $3$};
\node[draw, circle, minimum size=15pt, inner sep=2pt] at ($(v3-b) + (1,-2)$) (v7-b) {\small $4$};
\node[draw, circle, minimum size=15pt, inner sep=2pt] at ($(v3-b) + (3,0)$) (v8-b) {\small $b$};
\node[draw, circle, minimum size=15pt, inner sep=2pt] at ($(v8-b) + (-1,-2)$) (v9-b) {\small $5$};
\node[draw, circle, minimum size=15pt, inner sep=2pt] at ($(v8-b) + (1,-2)$) (v10-b) {\small $6$};
\node[draw, circle, minimum size=15pt, inner sep=2pt] at ($(v8-b) + (3,0)$) (v11-b) {};
\node[draw, circle, minimum size=15pt, inner sep=2pt] at ($(v11-b) + (-1,-2)$) (v12-b) {\small $7$};
\node[draw, circle, minimum size=15pt, inner sep=2pt] at ($(v11-b) + (1,-2)$) (v13-b) {\small $8$};

\draw[thick] (v2-b) -- (v4-b);
\draw[thick] (v2-b) -- (v5-b);
\draw[thick] (v2-b) -- (v3-b);
\draw[thick] (v3-b) -- (v6-b);
\draw[thick] (v3-b) -- (v7-b);
\draw[very thick, dashed] (v3-b) -- (v8-b);
\draw[thick] (v8-b) -- (v9-b);
\draw[thick] (v8-b) -- (v10-b);
\draw[thick] (v11-b) -- (v12-b);
\draw[thick] (v11-b) -- (v13-b);

%
%
\node[draw, circle, minimum size=15pt, inner sep=2pt] at (0,-8) (v2-c) {};
\node[draw, circle, minimum size=15pt, inner sep=2pt] at ($(v2-c) + (-1,-2)$) (v4-c) {\small $1$};
\node[draw, circle, minimum size=15pt, inner sep=2pt] at ($(v2-c) + (1,-2)$) (v5-c) {\small $2$};
\node[draw, circle, minimum size=15pt, inner sep=2pt] at ($(v2-c)+(6,0)$) (v8-c) {};
\node[draw, circle, minimum size=15pt, inner sep=2pt] at ($(v8-c) + (-1.5,-2)$) (v6-c) {\small $3$};
\node[draw, circle, minimum size=15pt, inner sep=2pt] at ($(v8-c) + (-0.5,-2)$) (v7-c) {\small $4$};
\node[draw, circle, minimum size=15pt, inner sep=2pt] at ($(v8-c) + (0.5,-2)$) (v9-c) {\small $5$};
\node[draw, circle, minimum size=15pt, inner sep=2pt] at ($(v8-c) + (1.5,-2)$) (v10-c) {\small $6$};
\node[draw, circle, minimum size=15pt, inner sep=2pt] at ($(v10-c) + (1,0)$) (v11-c) {\small $7$};
\node[draw, circle, minimum size=15pt, inner sep=2pt] at ($(v11-c) + (2,0)$) (v12-c) {\small $8$};
\node[draw, circle, minimum size=15pt, inner sep=2pt] at ($(v12-c) + (-1,2)$) (v13-c) {};

\draw[thick] (v2-c) -- (v4-c);
\draw[thick] (v2-c) -- (v5-c);
\draw[thick] (v2-c) -- (v8-c);
\draw[thick] (v8-c) -- (v6-c);
\draw[thick] (v8-c) -- (v7-c);
\draw[thick] (v8-c) -- (v9-c);
\draw[thick] (v8-c) -- (v10-c);
\draw[thick] (v11-c) -- (v13-c);
\draw[thick] (v12-c) -- (v13-c);
%
%
\node[draw, circle, minimum size=15pt, inner sep=2pt] at (0,-12) (v2-d) {};
\node[draw, circle, minimum size=15pt, inner sep=2pt] at ($(v2-d)+(3,0)$) (v3-d) {};
\node[draw, circle, minimum size=15pt, inner sep=2pt] at ($(v2-d) + (-1,-2)$) (v4-d) {\small $1$};
\node[draw, circle, minimum size=15pt, inner sep=2pt] at ($(v2-d) + (1,-2)$) (v5-d) {\small $2$};
\node[draw, circle, minimum size=15pt, inner sep=2pt] at ($(v3-d) + (-1,-2)$) (v6-d) {};
\node[draw, circle, minimum size=15pt, inner sep=2pt] at ($(v3-d) + (1,-2)$) (v7-d) {};
\node[draw, circle, minimum size=15pt, inner sep=2pt] at ($(v2-d)+(6,0)$) (v8-d) {};
\node[draw, circle, minimum size=15pt, inner sep=2pt] at ($(v8-d) + (-1,-2)$) (v9-d) {};
\node[draw, circle, minimum size=15pt, inner sep=2pt] at ($(v8-d) + (1,-2)$) (v10-d) {};

\draw[thick] (v2-d) -- (v4-d);
\draw[thick] (v2-d) -- (v5-d);
\draw[thick] (v2-d) -- (v3-d);
\draw[thick] (v3-d) -- (v6-d);
\draw[thick] (v3-d) -- (v7-d);
\draw[thick] (v3-d) -- (v8-d);
\draw[thick] (v8-d) -- (v9-d);
\draw[thick] (v8-d) -- (v10-d);

%
%
\node[] at (-3.5,0) {\small (a)};
\node[] at ($(v2-b) + (-3.5,0)$) {\small (b)};
\node[] at ($(v2-c) + (-3.5,0)$) {\small (c)};
\node[] at ($(v2-d) + (-3.5,0)$) {\small (d)};

\node[] at ($(v2-d)+(-3,-2.5)$) {Option 1};
\node[] at ($(v2-d)+(2,-2.5)$) {\small $3$};
\node[] at ($(v2-d)+(4,-2.5)$) {\small $4$};
\node[] at ($(v2-d)+(5,-2.5)$) {\small $5$};
\node[] at ($(v2-d)+(7,-2.5)$) {\small $6$};

\node[] at ($(v2-d)+(-3,-3)$) {Option 2};
\node[] at ($(v2-d)+(2,-3)$) {\small $3$};
\node[] at ($(v2-d)+(4,-3)$) {\small $6$};
\node[] at ($(v2-d)+(5,-3)$) {\small $4$};
\node[] at ($(v2-d)+(7,-3)$) {\small $5$};

\node[] at ($(v2-d)+(-3,-3.5)$) {Option 3};
\node[] at ($(v2-d)+(2,-3.5)$) {\small $5$};
\node[] at ($(v2-d)+(4,-3.5)$) {\small $6$};
\node[] at ($(v2-d)+(5,-3.5)$) {\small $3$};
\node[] at ($(v2-d)+(7,-3.5)$) {\small $4$};

\end{tikzpicture}
}
\caption{An example of edge contraction. In (a), we see the original tree. In (b), the algorithm splits the tree into two subtrees. Yet, the algorithm could not figure exactly the topology of the left subtree. Instead, it contracts the dashed edge, and the result is shown in (c). The reason for the contraction is: the algorithm could not tell the true topology. In (d), we show multiple topologies that the algorithm might confuse between, leading it to contract the edge. These are labeled Option~1-3.}
\label{fig:appendix4}
\end{figure}
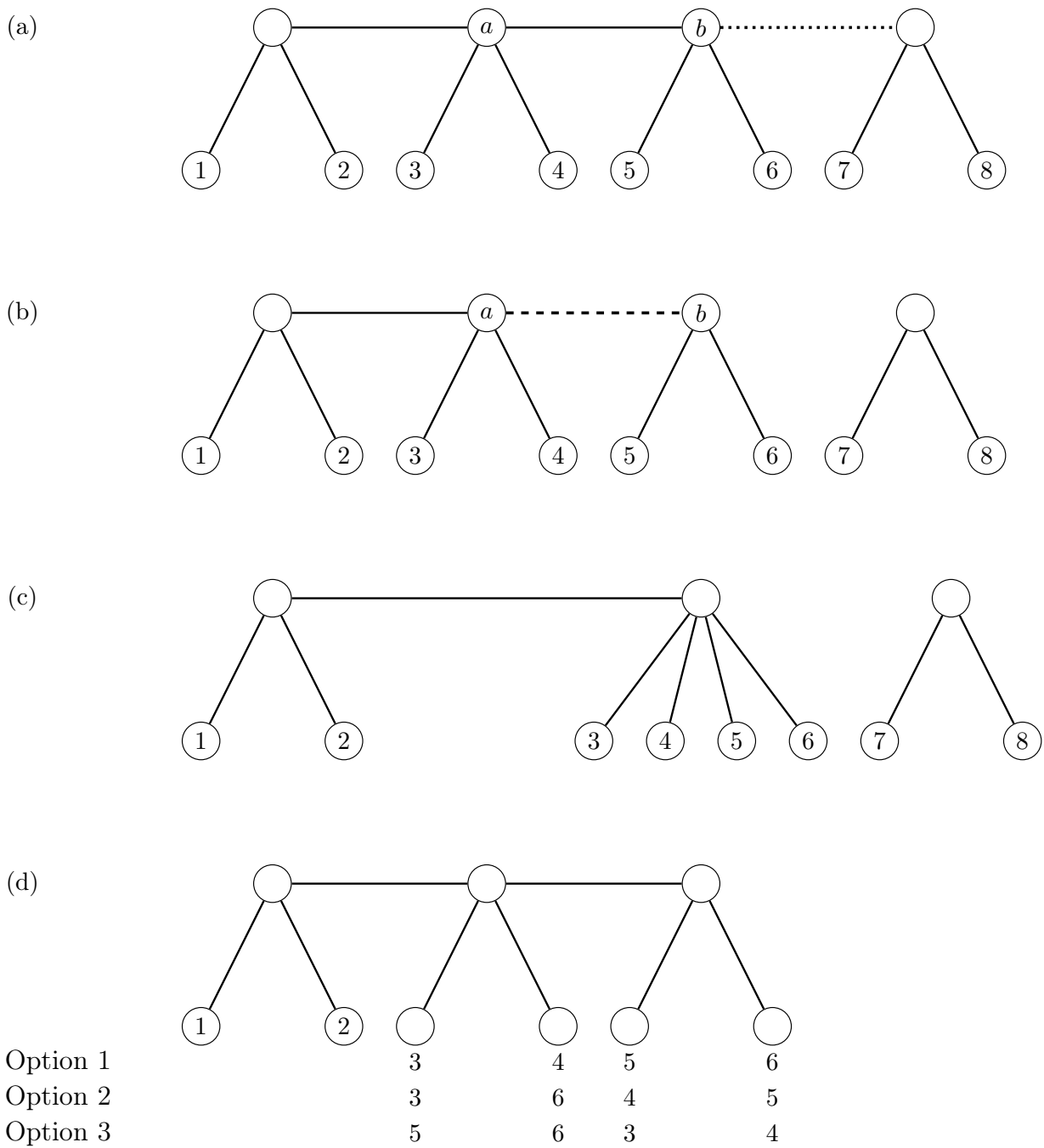

Before we give the formal proof, let us give an imprecise intuition of the approach.
In particular, we show how to compare between the output of the algorithm and the true topology. For the simple case where the output is obtained from the true tree by only deleting edges, these edges are ``distant'' from the leaves, and so they cannot influence the leaf distribution significantly. Thus, it suffices to study the more complicated case where the subtrees are \emph{not} obtained by simply cutting edges from the true tree.
Here, we first compare the true tree to a tree where all edges that appear in two different subtrees have been contracted, as shown in Figure~\ref{fig:appendix5}~(b). In the next step, each edge corresponds only to one subtree, and we can detach the different subtrees, as shown in Figure~\ref{fig:appendix5}~(c). Lastly, we can reconstruct the edges that were previously contracted, as shown in Figure~\ref{fig:appendix5}~(d).

\begin{figure}[H]
\centering
\resizebox{\linewidth}{!}{%
\begin{tikzpicture}
%
%
\node[draw, circle, minimum size=15pt, inner sep=2pt] at (0,0) (v1-a) {};
\node[draw, circle, minimum size=15pt, inner sep=2pt] at (6,0) (v2-a) {};
\node[draw, circle, minimum size=15pt, inner sep=2pt] at ($(v1-a) + (-1.5,-2)$) (v3-a) {\small $1$};
\node[draw, circle, minimum size=15pt, inner sep=2pt] at ($(v1-a) + (-0.5,-2)$) (v4-a) {\small $2$};
\node[draw, circle, minimum size=15pt, inner sep=2pt] at ($(v1-a) + (0.5,-2)$) (v5-a) {\small $3$};
\node[draw, circle, minimum size=15pt, inner sep=2pt] at ($(v1-a) + (1.5,-2)$) (v6-a) {\small $4$};
\node[draw, circle, minimum size=15pt, inner sep=2pt] at ($(v2-a) + (-1.5,-2)$) (v7-a) {\small $5$};
\node[draw, circle, minimum size=15pt, inner sep=2pt] at ($(v2-a) + (-0.5,-2)$) (v8-a) {\small $6$};
\node[draw, circle, minimum size=15pt, inner sep=2pt] at ($(v2-a) + (0.5,-2)$) (v9-a) {\small $7$};
\node[draw, circle, minimum size=15pt, inner sep=2pt] at ($(v2-a) + (1.5,-2)$) (v10-a) {\small $8$};

\draw[thick, red] (v1-a) -- (v3-a);
\draw[thick, red] (v1-a) -- (v4-a);
\draw[thick, blue] (v1-a) -- (v5-a);
\draw[thick, blue] (v1-a) -- (v6-a);
\draw[very thick, dashed, blue!50!red] (v1-a) -- (v2-a);
\draw[thick, red] (v2-a) -- (v7-a);
\draw[thick, red] (v2-a) -- (v8-a);
\draw[thick, blue] (v2-a) -- (v9-a);
\draw[thick, blue] (v2-a) -- (v10-a);

%
%
\node[draw, circle, minimum size=15pt, inner sep=2pt] at (3,-4) (v1-b) {};
\node[draw, circle, minimum size=15pt, inner sep=2pt] at ($(v1-b) + (-3.5,-2)$) (v2-b) {\small $1$};
\node[draw, circle, minimum size=15pt, inner sep=2pt] at ($(v1-b) + (-2.5,-2)$) (v3-b) {\small $2$};
\node[draw, circle, minimum size=15pt, inner sep=2pt] at ($(v1-b) + (-1.5,-2)$) (v4-b) {\small $3$};
\node[draw, circle, minimum size=15pt, inner sep=2pt] at ($(v1-b) + (-0.5,-2)$) (v5-b) {\small $4$};
\node[draw, circle, minimum size=15pt, inner sep=2pt] at ($(v1-b) + (0.5,-2)$) (v6-b) {\small $5$};
\node[draw, circle, minimum size=15pt, inner sep=2pt] at ($(v1-b) + (1.5,-2)$) (v7-b) {\small $6$};
\node[draw, circle, minimum size=15pt, inner sep=2pt] at ($(v1-b) + (2.5,-2)$) (v8-b) {\small $7$};
\node[draw, circle, minimum size=15pt, inner sep=2pt] at ($(v1-b) + (3.5,-2)$) (v9-b) {\small $8$};

\draw[thick, red] (v1-b) -- (v2-b);
\draw[thick, red] (v1-b) -- (v3-b);
\draw[thick, blue] (v1-b) -- (v4-b);
\draw[thick, blue] (v1-b) -- (v5-b);
\draw[thick, red] (v1-b) -- (v6-b);
\draw[thick, red] (v1-b) -- (v7-b);
\draw[thick, blue] (v1-b) -- (v8-b);
\draw[thick, blue] (v1-b) -- (v9-b);

%
%
\node[draw, circle, minimum size=15pt, inner sep=2pt] at (0,-8) (v1-c) {};
\node[draw, circle, minimum size=15pt, inner sep=2pt] at ($(v1-c) + (-1.5,-2)$) (v2-c) {\small $1$};
\node[draw, circle, minimum size=15pt, inner sep=2pt] at ($(v1-c) + (-0.5,-2)$) (v3-c) {\small $2$};
\node[draw, circle, minimum size=15pt, inner sep=2pt] at ($(v1-c) + (0.5,-2)$) (v4-c) {\small $5$};
\node[draw, circle, minimum size=15pt, inner sep=2pt] at ($(v1-c) + (1.5,-2)$) (v5-c) {\small $6$};
\node[draw, circle, minimum size=15pt, inner sep=2pt] at (6,-8) (v6-c) {};
\node[draw, circle, minimum size=15pt, inner sep=2pt] at ($(v6-c) + (-1.5,-2)$) (v7-c) {\small $3$};
\node[draw, circle, minimum size=15pt, inner sep=2pt] at ($(v6-c) + (-0.5,-2)$) (v8-c) {\small $4$};
\node[draw, circle, minimum size=15pt, inner sep=2pt] at ($(v6-c) + (0.5,-2)$) (v9-c) {\small $7$};
\node[draw, circle, minimum size=15pt, inner sep=2pt] at ($(v6-c) + (1.5,-2)$) (v10-c) {\small $8$};

\draw[thick, red] (v1-c) -- (v2-c);
\draw[thick, red] (v1-c) -- (v3-c);
\draw[thick, red] (v1-c) -- (v4-c);
\draw[thick, red] (v1-c) -- (v5-c);
\draw[thick, blue] (v6-c) -- (v7-c);
\draw[thick, blue] (v6-c) -- (v8-c);
\draw[thick, blue] (v6-c) -- (v9-c);
\draw[thick, blue] (v6-c) -- (v10-c);

%
%
\node[draw, circle, minimum size=15pt, inner sep=2pt] at (-1.5,-12) (v2-d) {};
\node[draw, circle, minimum size=15pt, inner sep=2pt] at (1.5,-12) (v3-d) {};
\node[draw, circle, minimum size=15pt, inner sep=2pt] at ($(v2-d) + (-1,-2)$) (v4-d) {\small $1$};
\node[draw, circle, minimum size=15pt, inner sep=2pt] at ($(v2-d) + (1,-2)$) (v5-d) {\small $2$};
\node[draw, circle, minimum size=15pt, inner sep=2pt] at ($(v3-d) + (-1,-2)$) (v6-d) {\small $5$};
\node[draw, circle, minimum size=15pt, inner sep=2pt] at ($(v3-d) + (1,-2)$) (v7-d) {\small $6$};
\node[draw, circle, minimum size=15pt, inner sep=2pt] at (4.5,-12) (v8-d) {};
\node[draw, circle, minimum size=15pt, inner sep=2pt] at ($(v8-d) + (-1,-2)$) (v9-d) {\small $3$};
\node[draw, circle, minimum size=15pt, inner sep=2pt] at ($(v8-d) + (1,-2)$) (v10-d) {\small $4$};
\node[draw, circle, minimum size=15pt, inner sep=2pt] at (7.5,-12) (v11-d) {};
\node[draw, circle, minimum size=15pt, inner sep=2pt] at ($(v11-d) + (-1,-2)$) (v12-d) {\small $7$};
\node[draw, circle, minimum size=15pt, inner sep=2pt] at ($(v11-d) + (1,-2)$) (v13-d) {\small $8$};

\draw[thick, red] (v2-d) -- (v4-d);
\draw[thick, red] (v2-d) -- (v5-d);
\draw[thick, red] (v2-d) -- (v3-d);
\draw[thick, red] (v3-d) -- (v6-d);
\draw[thick, red] (v3-d) -- (v7-d);
\draw[thick, blue] (v8-d) -- (v9-d);
\draw[thick, blue] (v8-d) -- (v10-d);
\draw[thick, blue] (v8-d) -- (v11-d);
\draw[thick, blue] (v11-d) -- (v12-d);
\draw[thick, blue] (v11-d) -- (v13-d);

%
%
\node[] at (-3.5,0) {\small (a)};
\node[] at (-3.5,-4) {\small (b)};
\node[] at (-3.5,-8) {\small (c)};
\node[] at (-3.5,-12) {\small (d)};
\end{tikzpicture}
}
\caption{Analyzing the difference between the true topology and the output of the algorithm. In (a), the true tree contains the union of solid and dashed lines. Yet, the algorithm splits the tree into two subtrees, with leaves $\{1,2,5,6\}$ and $\{3,4,7,8\}$. In (d), we can see the output of the algorithm. In order to compare between the true and output topologies, we construct two auxiliary forests: In (b), we contract all the edges that are shared between the two subtree. In this example, this amounts to contracting the dashed edge. Then, in (c), we split the different subtrees. In (d), we reconstruct the edge that was previously contracted. Yet, we reconstruct this edge separately for each subtree.}
\label{fig:appendix5}
\end{figure}
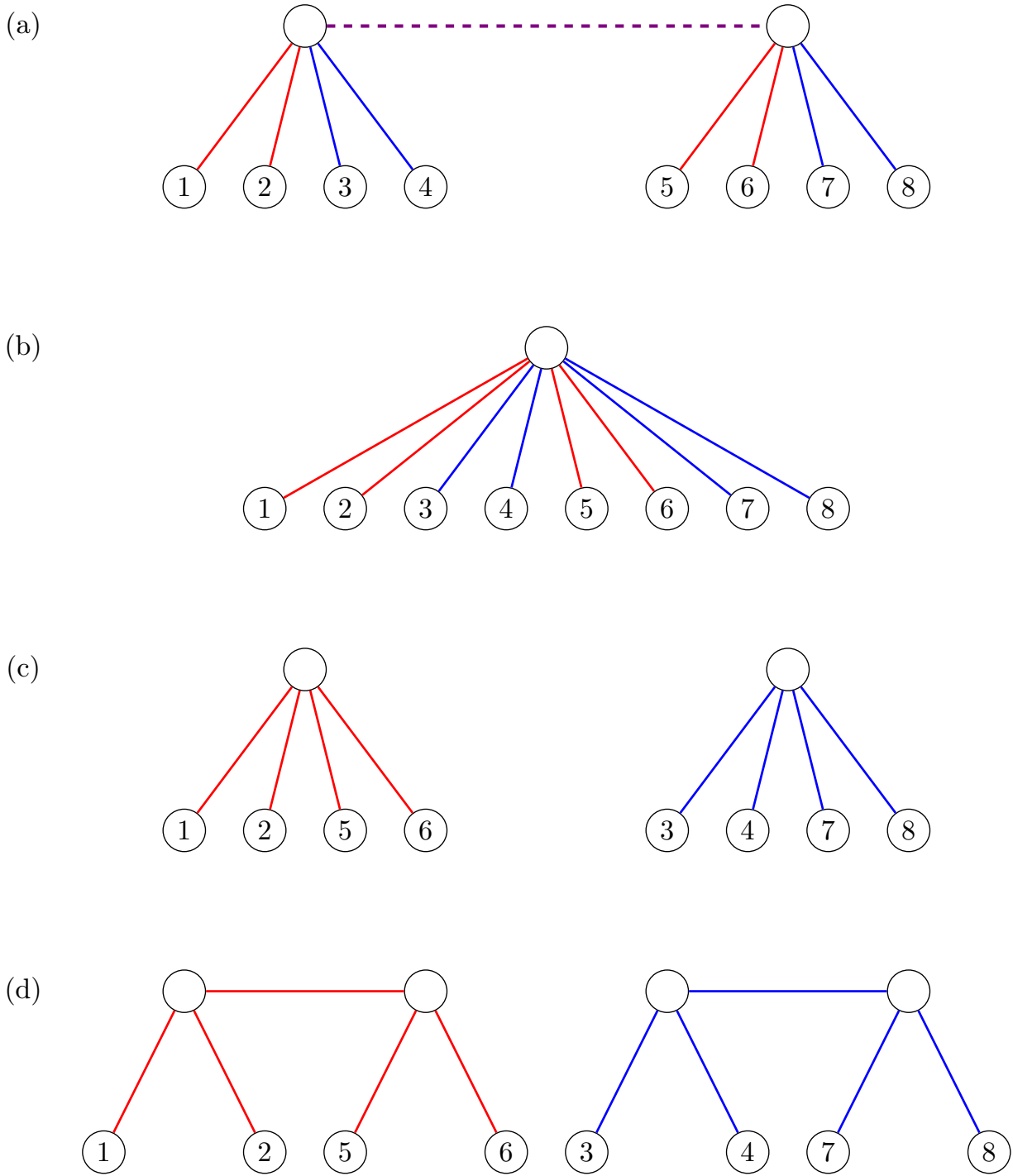

We can bound the total variation distance for each of these steps using the following guarantees:
(1)  edges that appear in multiple subtrees have weight close to $1$. Hence, contracting them, which is equivalent to changing their weight to $1$, does not influence significantly the total variation distance. (2) Leaves of different trees have a small correlation. Hence, detaching the different subtrees does not incur a high cost in total variation. 

\subsection{Definitions and probabilistic lemmas}

We begin with some notations that are specific for this section: given a tree $T=(V,E)$ with weights $\theta_e$ on the edges $e\in E$, denote the resulting distribution over the values $x=(x_1,\dots,x_n)$ on the leaves by $\Pr_{T,\theta}[x]$. The total variation distance between two leaf-distributions of two models, $(T,\theta)$ and $(T',\theta')$ is denoted by $\TV(\Pr_{T,\theta}[x],\Pr_{T',\theta'}[x])$. The values on the internal nodes are denoted by $y_v$ for each internal node $v \in V$, and $\Pr_{T,\theta}[x,y]$ denotes the joint distribution over the leaves and internal nodes. We continue with some definitions:

\begin{definition}[Edge contraction]
Given a graph $G = (V,E)$ and an edge $e = (u,v)$ in the graph, we say that a graph $G'=(V',E')$ is obtained from $G$ by contracting $e$ if $G'$ is the result of removing $e$ from the graph and identifying $u$ and $v$ as a single vertex. Namely, if the new edge is denoted $z$, then 
\[
V' = V \setminus \{u,v\}\cup \{z\}
\]
and
\[
E' = E \setminus \{(u,w) \colon w \in V\}
\setminus \{(v,w) \colon w \in V\}
\cup \{(z,w) \colon (u,w) \in V\}\enspace.
\]
\end{definition}
For an example of an edge-contraction, see Figure~\ref{fig:appendix4}~(b)-(c): The dashed edge in (b) is contracted, resulting in the graph shown in (c).

We now define the contraction of all degree-2 nodes:
\begin{definition}\label{def:contraction}
Given a graph $G=(V,E)$, we say that $G'=(V',E')$ is obtained from $G$ by contracting all the degree-2 nodes, if $G'$ is obtained from $G$ using the following process:
\begin{itemize}
    \item $G'\gets G$.
    \item While $G'$ contains a node $w$ of degree $2$:
    \begin{itemize}
        \item Contract one of the edges incident with $w$.
    \end{itemize}
\end{itemize}
\end{definition}
For an example of a contraction of degree-2 nodes, see Figure~\ref{fig:appendix2}, (b)-(c): In (b), there are some nodes of degree $2$, labeled $a$-$f$. In (c), we can see the result of contracting these nodes.

\begin{definition}[Subtrees induced by a set of leaves] \label{def:induced-subtree}
Given a tree $T$ and a subset $S$ of the leaves of $T$, the \emph{subtree of $T$ induced by $S$} is the tree that is obtained from $T$ by removing all the edges and all the nodes that are not in any path between two leaves $i,j \in S$.
\end{definition}
In other words, the subtree of $T$ induced by $S$ is the minimal subtree of $T$ that contains $S$. For an example, in Figure~\ref{fig:appendix2}~(a), a tree is depicted, and its subtrees induced by $\{1,2,3,5\}$ and $\{4,6,7,8\}$ are depicted in Figure~\ref{fig:appendix2}~(b-1) and (b-2), respectively.

Throughout the proof, we will modify graphs by contracting edges. Whenever we contract an edge, we identify its two endpoints as a single vertex. If we contract multiple edges, the resulting graph may identify even more than two edges of the original as one edge. If a graph $T'$ is the result of multiple edge-contractions applied on a graph $T$, then, for each vertex $v'$ of $T'$, \emph{the set of preimages of $v'$ under the transformation from $T$ to $T'$} is defined as the set of all vertices of $T$ that were identified into $v'$. To be more formal, we provide the following definition:
\begin{definition}\label{def:preimages}
Let $T'=(V',E')$ be obtained from $T=(V,E)$ via a sequence of edge contractions. For any vertex $v' \in T'$, \emph{the set of preimages of $v$ under the transformation from $T$ to $T'$} is defined as the following set, which we denote here by $A_{v'}$:
\begin{itemize}
    \item Start with the tree $T$, and define $A_v = \{v\}$ for each $v \in V$.
    \item For any contraction of an edge $(u,v)$ into a single vertex $w$:
    \begin{itemize}
        \item Define $A_w = A_u \cup A_v$.
    \end{itemize}
    \item Return $A_{v'}$ for each vertex $v' \in V'$.
\end{itemize}
\end{definition}

Lastly, notice that if we contract edges then some of the nodes might change names. Hence, an edge that was connecting between two nodes $(u,v)$ in the original tree, might connect two other nodes in the contracted tree. Yet, the edge's function remain the same. Hence, we define \emph{the analogue of an edge $e$} in the contracted graph:
\begin{definition}\label{def:edge-analogue}
Let $T'$ be a tree that results from another tree $T$ via a sequence of edge contractions. Let $(u,v)$ be an edge in $T$ that is not contracted. Then, \emph{the analogue of $(u,v)$ in $T'$} is the obtained from $(u,v)$ in the following fashion:
\begin{itemize}
    \item Set $e' \gets (u,v)$.
    \item For any contraction of edge $(z,w)$ into a node $q$ that $T$ undergoes:
    \begin{itemize}
        \item If $u \in \{z,w\}$ then $e' \gets (q,v)$.
        \item Otherwise, if $v \in \{z,w\}$ then $e'\gets (u,q)$,
    \end{itemize}
    \item Return $e'$.
\end{itemize}
\end{definition}

We continue with presenting some auxiliary lemmas that will be used for the proof.

\begin{lemma}\label{lem:change-one-edge}
Let $T=(V,E)$ and let $(\theta_e)_{e\in E}$ denote some weight-vector on the edges. Let $\theta'$ denote a weight vector that differs only on one edge $e'$. Then, $\TV(\Pr_{T,\theta}[x],\Pr_{T,\theta'}[x]) \le |\theta'_{e'}-\theta_{e'}|/2$.
\end{lemma}
\begin{proof}
Due to the equivalent definition of total variation in terms of coupling, it is sufficient to produce a coupling between $x \sim \Pr_{T,\theta}$ and $x'\sim \Pr_{T',\theta'}$ such that $\Pr[x\ne x'] \le |\theta'_{e'}-\theta_{e'}|/2$.
While $x$ and $x'$ denote the values of the leaves, we will use $y$ and $y'$ to denote the values on the internal nodes, such that $(x,y)$ and $(x',y')$ are jointly sampled from $\Pr_{T,\theta}$ and $\Pr_{T',\theta'}$, respectively. We will produce a coupling between $(x,y)$ and $(x',y')$ such that $\Pr[(x,y)\ne (x',y')] \le |\theta'_{e'}-\theta_{e'}|/2$ and this suffices to conclude the proof.

Let $e' = (u,v)$ denote the \emph{single} edge where $\theta$ and $\theta'$ differ. We produce the coupling as follows:
\begin{itemize}
    \item We start by sampling $y_u$ uniformly from $\{-1,1\}$ and $y'_u = y_u$
    \item Then, we sample $y_v$ such that $\Pr[y_v=y_u] = (1+\theta_e)/2$. Similarly, we sample $y'_v$ such that $\Pr[y'_v = y'_u] = \Pr[y'_v = y_u] = (1+\theta'_e)/2$. Note that we can couple $y_v$ and $y'_v$ such that $\Pr[y_v\ne y'_v] =\TV(y_v\mid y_u,~ y'_v\mid y_u) = |\theta_e - \theta'_e|/2$.
    \item Next, we will sample the remaining values of $x,y$ conditioned on $y_u$ and $y_v$, and the remaining values of $x',y'$ conditioned on $y'_u$ and $y'_v$. If $y_u=y'_u$ and $y_v=y'_v$, then these two conditional distributions are the same, hence, we can sample such that $(x,y)=(x',y')$. Otherwise, we will sample $x,y,x'$ and $y'$ arbitrarily.
\end{itemize}
 Notice that with probability $1-|\theta_{e'}-\theta'_{e'}|/2$, $y_u=y'_u$ and $y_v=y'_v$. Hence, $\Pr[(x,y)=(x',y')] \ge 1- |\theta_{e'}-\theta'_{e'}|/2$, as required to complete the proof.
\end{proof}

\begin{lemma}\label{lem:contract-no-TV}
Let $T$ be a tree and $\theta$ a weight-function on its edges. Let $T'$ be obtained from $T$ by contracting an edge $e$ with $\theta_e=1$ and let $\theta'$ denote the restriction of $\theta$ to the edges of $T'$. Then, for any values $x=(x_1,\dots,x_n)$ on the leaves, $\Pr_{T,\theta}[x]=\Pr_{T',\theta'}[x]$.
\end{lemma}
\begin{proof}
Notice that by contracting the edge, the pairwise correlations $\alpha_{ij}$ between any two leaves do not change, hence the leaf distributions are identical (this is a known fact and it also follows directly from Lemma~\ref{l:basic_extend}).
\end{proof}

\subsection{Proof body}

We are ready to present the results of \cite{daskalakis2009phylogenies}. They are written in a slightly different way than was originally present, but we translate their guarantees to our notation. (See Section~\ref{sec:translate-DMR} for translating their guarantees).

\begin{theorem}\label{thm:DMR}
There is a polynomial-time algorithm, for learning some unknown tree $T^* = (V^*,E^*)$, whose properties are presented below.
Its inputs are:
\begin{itemize}
    \item Approximate correlations, $\hat{\alpha}_{ij}$, for any two leaves $i,j \in [n]$. These satisfy the guarantee that there exists an Ising model $\Pr_{T^*,\theta^*}$, whose correlations $\alpha^*_{ij}$ satisfy: $|\alpha^*_{ij} - \hat{\alpha}_{ij}| \le \eta$, for any two leaves $i,j$ and for some $\eta\in (0,1/2]$.
    \item Parameters $\xi,\delta>0$ such that $\xi\delta \ge \eta$.
\end{itemize}
The algorithm outputs a forest, whose connected components are trees, $\tilde{T}_1=(\tilde{V}_1,\tilde{E}_1), \dots, \tilde{T}_R=(\tilde{V}_R,\tilde{E}_R)$ with the following guarantees: (below, $C>0$ is a universal constant)
\begin{itemize}
    \item Let $S_r$ denote the set of leaves of $\tilde T_r$ for any $r\in [R]$. Then, $\{S_1,\dots,S_R\}$ is a partition of the set of leaves of $T^*$.
    \item For all $r=1,\dots,R$, denote by $T_r=(V_r,E_r)$ the subtree of $T^*$ induced by $S_r$, as defined in Definition~\ref{def:induced-subtree}. Then, each tree $\tilde{T}_r$ is obtained from $T_r$ using the following operations:
    \begin{itemize}
        \item Contract a subset of the edges. Only edges $e$ of weight $\theta^*_e \ge 1-C\xi$ can be contracted.
        \item Contract all the nodes of degree-2 from the resulting tree.
    \end{itemize}
    \item Any edge $e$ that is common to more than one of the trees $\{T_1,\dots,T_R\}$, satisfies $\theta^*_e \ge 1-C\xi$.
    \item Any leaves $i,j$ that belongs to different sets from $\{S_1,\dots,S_R\}$, satisfy $|\alpha^*_{ij}| \le C\sqrt{\delta}$.
\end{itemize}
\end{theorem}

For example, in Figure~\ref{fig:appendix4}~(a) a tree is depicted, whereas the output of the algorithm is given in Figure~\ref{fig:appendix4}~(c). Since the dashed edge $e$ connecting nodes $a$ and $b$ is contracted in the output, its weight must satisfy $\theta^*_e \ge 1 - \Omega(\xi)$.  Further, since nodes $1$ and $7$ reside in different subtrees in the output of the algorithm, their correlation must satisfy $|\alpha^*_{17}| \le O(\sqrt\delta)$. For another example, see Figure~\ref{fig:appendix2}: in (a), the original tree is depicted, whereas, the output is the forest in (c). The induced trees $T_1$ and $T_2$ are shown in (b). Since the dashed edges in (a) are shared by both induced subtrees, their weight satisfies $\theta^*_e \ge 1-\Omega(\xi)$. 

Below, we will prove the following central Lemma:
\begin{lemma}\label{lem:main-alg}
Let $\eta > 0$ be a parameter and $T^*$ be an unknown tree with weight vector $\theta^*$ and pairwise correlations $\alpha^*$ between the leaves.
Suppose we execute Algorithm~\ref{alg:unkonwn-top} with the following inputs:
\begin{itemize}
    \item Pairwise correlations $\hat{\alpha}_{ij}$ that satisfy $|\hat{\alpha}_{ij}-\alpha^*_{ij}| \le \eta$.
    \item Parameters $\delta,\xi$ such that $\xi \delta \ge \eta$ and $\xi = C_1/n$, for some universal constant $C_1 > 0$.
    \item Parameter $\hat{\eta}$ that satisfies $\hat{\eta} = C_2 n \xi + \eta$, for some universal constant $C_2 > 0$.
\end{itemize}
Recall that the algorithm outputs a weighted forest, and denote the forest by $\tilde{F}$ and the weights by $\hat{\theta}$.
Then, $\TV(\Pr_{T^*,\theta^*}[x],\Pr_{\tilde{F},\tilde{\theta}}[x]) \in O(n^3\xi + n^2 \sqrt{\delta} + \eta)$. (We note that an Ising model over a forest is defined by taking the different tree components to be independent.)
\end{lemma}

The remainder of this section is dedicated to the proof of Lemma~\ref{lem:main-alg}. In Section~\ref{sec:alg-conclusion} we conclude the proof of Theorem~\ref{thm:alg} (unknown topology), by substituting the parameters $\xi,\delta$ and $\eta$ appropriately using the finite-sample estimates.

To analyze the algorithm, we will create auxiliary trees $T^{(i)} = (V^{(i)},E^{(i)})$ with weight function $\theta^{(i)}$ and pairwise correlations $\alpha^{(i)}_{ij}$ (for $i = \{1,2,3\}$) that interpolate between the true parameters $T^*, \theta^*$, and the Algorithm~\ref{alg:unkonwn-top}'s output $(\tilde F, \tilde\theta)$.
To bound the total variation distance between $(T^*, \theta^*)$ and $(\tilde F, \tilde\theta)$, we apply triangle inequality after individually bounding the total variation of the leaf distributions (i) between $(T^*, \theta^*)$ and $(T^{(1)}, \theta^{(1)})$, (ii) between $(T^{(1)}, \theta^{(1)})$ and $(T^{(2)}, \theta^{(2)})$, (iii) between $(T^{(2)}, \theta^{(2)})$ and $(T^{(3)}, \theta^{(3)})$, and (iv) between $(T^{(3)}, \theta^{(3)})$ and $(\tilde F, \tilde\theta)$.

Before defining the first intermediate distribution, $(T^{(1)},\theta^{(1)})$, we recall some definitions. First, recall that the algorithm of \cite{daskalakis2009phylogenies} returns a forest whose connected components are $(\tilde{T}_1,\dots,\tilde{T}_R)$. Each $\tilde{T}_r$ is a modification of $T_r$, which is defined as the subtree of $T^*$ that is induced by the set of leaves of $\tilde{T}_r$. As an intermediate step, we start by modifying $T^*$ according to the induced subtrees $T_1,\dots,T_R$. Note that initially, we consider $T_r$ instead of $\tilde{T}_r$, as $T_r$ is \emph{closer} to $T^*$ than $\tilde{T}_r$.

We start by defining $T^{(1)}$ as the tree that is obtained from $T^*$ by contracting all the edges that appear in more than one induced tree $T_r$ Further, the edge-weight for $\theta^{(1)}$ equals $\theta^*$ on all the remaining (non-contracted) edges. Note that $T^{(1)}$ still has a single connected component. (For example, in Figure~\ref{fig:appendix13}, we contract the purple-dashed lines in (a) because they appear both in the induces red and blue trees. This results in the tree depicted in (b).)

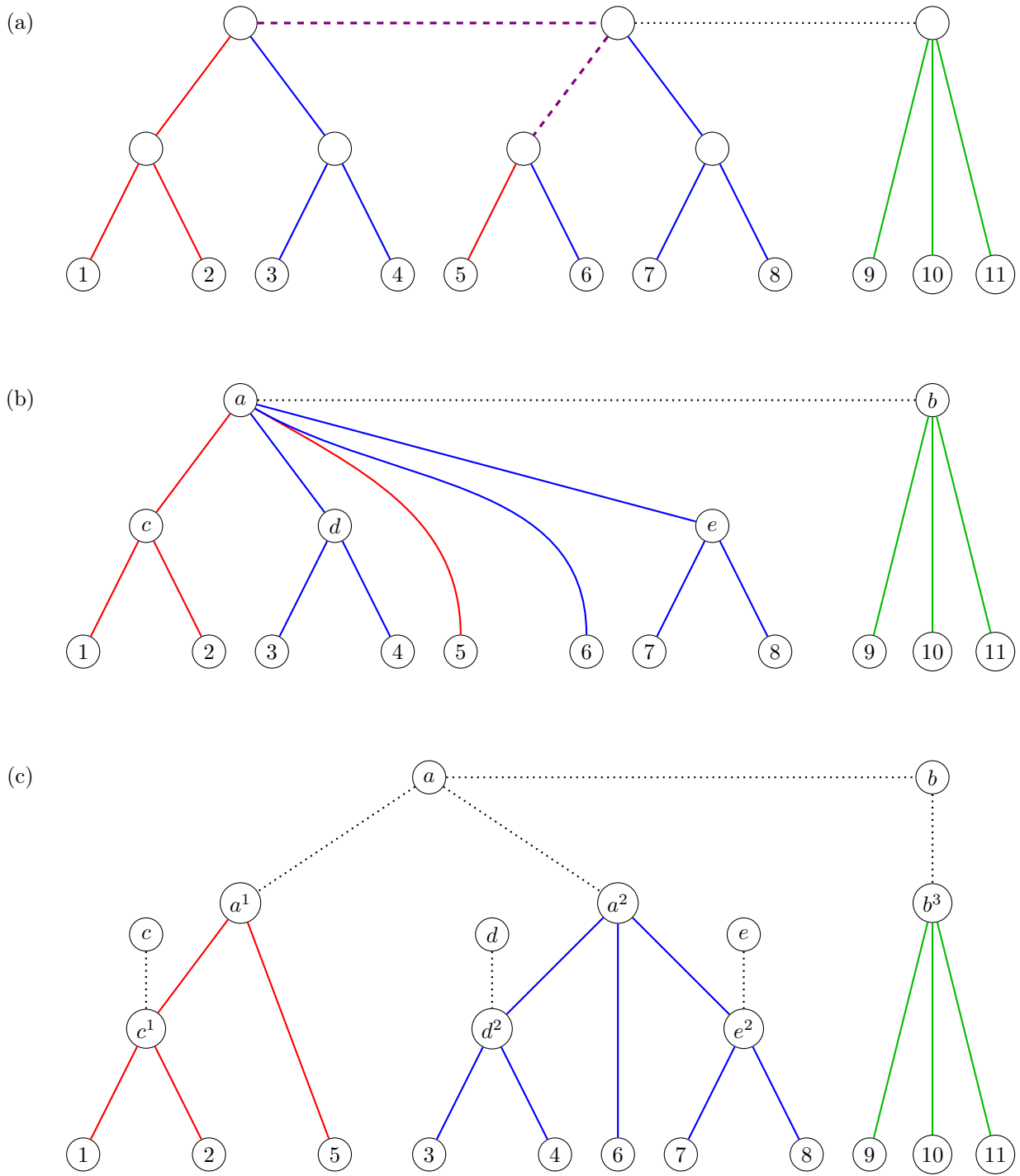
\begin{figure}[htbp]
\centering
\resizebox{\linewidth}{!}{%
\begin{tikzpicture}
%
%
\node[draw, circle, minimum size=15pt, inner sep=2pt] at (0,0) (v1-a) {};
\node[draw, circle, minimum size=15pt, inner sep=2pt] at ($(v1-a) + (-1.5,-2)$) (v2-a) {};
\node[draw, circle, minimum size=15pt, inner sep=2pt] at ($(v1-a) + (1.5,-2)$) (v3-a) {};
\node[draw, circle, minimum size=15pt, inner sep=2pt] at ($(v2-a) + (-1,-2)$) (v4-a) {\small $1$};
\node[draw, circle, minimum size=15pt, inner sep=2pt] at ($(v2-a) + (1,-2)$) (v5-a) {\small $2$};
\node[draw, circle, minimum size=15pt, inner sep=2pt] at ($(v3-a) + (-1,-2)$) (v6-a) {\small $3$};
\node[draw, circle, minimum size=15pt, inner sep=2pt] at ($(v3-a) + (1,-2)$) (v7-a) {\small $4$};

\node[draw, circle, minimum size=15pt, inner sep=2pt] at ($(v1-a) + (6,0)$) (v8-a) {};
\node[draw, circle, minimum size=15pt, inner sep=2pt] at ($(v8-a) + (-1.5,-2)$) (v9-a) {};
\node[draw, circle, minimum size=15pt, inner sep=2pt] at ($(v8-a) + (1.5,-2)$) (v10-a) {};
\node[draw, circle, minimum size=15pt, inner sep=2pt] at ($(v9-a) + (-1,-2)$) (v11-a) {\small $5$};
\node[draw, circle, minimum size=15pt, inner sep=2pt] at ($(v9-a) + (1,-2)$) (v12-a) {\small $6$};
\node[draw, circle, minimum size=15pt, inner sep=2pt] at ($(v10-a) + (-1,-2)$) (v13-a) {\small $7$};
\node[draw, circle, minimum size=15pt, inner sep=2pt] at ($(v10-a) + (1,-2)$) (v14-a) {\small $8$};

\node[draw, circle, minimum size=15pt, inner sep=2pt] at ($(v8-a) + (5,0)$) (v15-a) {};
\node[draw, circle, minimum size=15pt, inner sep=2pt] at ($(v15-a) + (-1,-4)$) (v16-a) {\small $9$};
\node[draw, circle, minimum size=15pt, inner sep=2pt] at ($(v15-a) +
(0,-4)$) (v17-a) {\small $10$};
\node[draw, circle, minimum size=15pt, inner sep=2pt] at ($(v15-a) +
(1,-4)$) (v18-a) {\small $11$};

\draw[thick, red] (v1-a) -- (v2-a);
\draw[thick, blue] (v1-a) -- (v3-a);
\draw[thick, red] (v2-a) -- (v4-a);
\draw[thick, red] (v2-a) -- (v5-a);
\draw[thick, blue] (v3-a) -- (v6-a);
\draw[thick, blue] (v3-a) -- (v7-a);
\draw[very thick, dashed, red!50!blue] (v1-a) -- (v8-a);
\draw[very thick, dashed, red!50!blue] (v8-a) -- (v9-a);
\draw[thick, blue] (v8-a) -- (v10-a);
\draw[thick, red] (v9-a) -- (v11-a);
\draw[thick, blue] (v9-a) -- (v12-a);
\draw[thick, blue] (v10-a) -- (v13-a);
\draw[thick, blue] (v10-a) -- (v14-a);
\draw[thick, dotted] (v8-a) -- (v15-a);
\draw[thick, green!75!black] (v15-a) -- (v16-a);
\draw[thick, green!75!black] (v15-a) -- (v17-a);
\draw[thick, green!75!black] (v15-a) -- (v18-a);

%
%
\node[draw, circle, minimum size=15pt, inner sep=2pt] at (0,-6) (v1-b) {\small $a$};
\node[draw, circle, minimum size=15pt, inner sep=2pt] at ($(v1-b) + (-1.5,-2)$) (v2-b) {\small $c$};
\node[draw, circle, minimum size=15pt, inner sep=2pt] at ($(v1-b) + (1.5,-2)$) (v3-b) {\small $d$};
\node[draw, circle, minimum size=15pt, inner sep=2pt] at ($(v2-b) + (-1,-2)$) (v4-b) {\small $1$};
\node[draw, circle, minimum size=15pt, inner sep=2pt] at ($(v2-b) + (1,-2)$) (v5-b) {\small $2$};
\node[draw, circle, minimum size=15pt, inner sep=2pt] at ($(v3-b) + (-1,-2)$) (v6-b) {\small $3$};
\node[draw, circle, minimum size=15pt, inner sep=2pt] at ($(v3-b) + (1,-2)$) (v7-b) {\small $4$};

\node[minimum size=15pt, inner sep=2pt] at ($(v1-b) + (6,0)$) (v8-b) {};
\node[minimum size=15pt, inner sep=2pt] at ($(v8-b) + (-1.5,-2)$) (v9-b) {};
\node[draw, circle, minimum size=15pt, inner sep=2pt] at ($(v8-b) + (1.5,-2)$) (v10-b) {\small $e$};
\node[draw, circle, minimum size=15pt, inner sep=2pt] at ($(v9-b) + (-1,-2)$) (v11-b) {\small $5$};
\node[draw, circle, minimum size=15pt, inner sep=2pt] at ($(v9-b) + (1,-2)$) (v12-b) {\small $6$};
\node[draw, circle, minimum size=15pt, inner sep=2pt] at ($(v10-b) + (-1,-2)$) (v13-b) {\small $7$};
\node[draw, circle, minimum size=15pt, inner sep=2pt] at ($(v10-b) + (1,-2)$) (v14-b) {\small $8$};

\node[draw, circle, minimum size=15pt, inner sep=2pt] at ($(v8-b) + (5,0)$) (v15-b) {\small $b$};
\node[draw, circle, minimum size=15pt, inner sep=2pt] at ($(v15-b) + (-1,-4)$) (v16-b) {\small $9$};
\node[draw, circle, minimum size=15pt, inner sep=2pt] at ($(v15-b) +
(0,-4)$) (v17-b) {\small $10$};
\node[draw, circle, minimum size=15pt, inner sep=2pt] at ($(v15-b) +
(1,-4)$) (v18-b) {\small $11$};

\draw[thick, red] (v1-b) -- (v2-b);
\draw[thick, blue] (v1-b) -- (v3-b);
\draw[thick, red] (v2-b) -- (v4-b);
\draw[thick, red] (v2-b) -- (v5-b);
\draw[thick, blue] (v3-b) -- (v6-b);
\draw[thick, blue] (v3-b) -- (v7-b);
\draw[thick, blue] (v1-b) -- (v10-b);
\draw[thick, red] (v1-b)  to[out=330,in=90] (v11-b);
\draw[thick, blue] (v1-b) to[out=330,in=90] (v12-b);
\draw[thick, blue] (v10-b) -- (v13-b);
\draw[thick, blue] (v10-b) -- (v14-b);
\draw[thick, dotted] (v1-b) -- (v15-b);
\draw[thick, green!75!black] (v15-b) -- (v16-b);
\draw[thick, green!75!black] (v15-b) -- (v17-b);
\draw[thick, green!75!black] (v15-b) -- (v18-b);

%
%
\node[draw, circle, minimum size=15pt, inner sep=2pt] at (0,-14) (v1-c) {\small $a^1$};
\node[draw, circle, minimum size=15pt, inner sep=2pt] at ($(v1-c) + (-1.5,-2)$) (c1-c) {\small $c^1$};
\node[draw, circle, minimum size=15pt, inner sep=2pt] at ($(c1-c) + (0,1.5)$) (c0-c) {\small $c$};
\node[draw, circle, minimum size=15pt, inner sep=2pt] at ($(v1-c) + (1.5,-4)$) (v3-c) {\small $5$};
\node[draw, circle, minimum size=15pt, inner sep=2pt] at ($(c1-c) + (-1,-2)$) (v4-c) {\small $1$};
\node[draw, circle, minimum size=15pt, inner sep=2pt] at ($(c1-c) + (1,-2)$) (v5-c) {\small $2$};

\node[draw, circle, minimum size=15pt, inner sep=2pt] at ($(v1-c) + (6,0)$) (v8-c) {\small $a^2$};
\node[draw, circle, minimum size=15pt, inner sep=2pt] at ($(v8-c) + (-2,-2)$) (d2-c) {\small $d^2$};
\node[draw, circle, minimum size=15pt, inner sep=2pt] at ($(d2-c) + (0,1.5)$) (d0-c) {\small $d$};
\node[draw, circle, minimum size=15pt, inner sep=2pt] at ($(v8-c) + (0,-4)$) (mid-c) {\small $6$};
\node[draw, circle, minimum size=15pt, inner sep=2pt] at ($(v8-c) + (2,-2)$) (e2-c) {\small $e^2$};
\node[draw, circle, minimum size=15pt, inner sep=2pt] at ($(e2-c) + (0,1.5)$) (e0-c) {\small $e$};
\node[draw, circle, minimum size=15pt, inner sep=2pt] at ($(d2-c) + (-1,-2)$) (v11-c) {\small $3$};
\node[draw, circle, minimum size=15pt, inner sep=2pt] at ($(d2-c) + (1,-2)$) (v12-c) {\small $4$};
\node[draw, circle, minimum size=15pt, inner sep=2pt] at ($(e2-c) + (-1,-2)$) (v13-c) {\small $7$};
\node[draw, circle, minimum size=15pt, inner sep=2pt] at ($(e2-c) + (1,-2)$) (v14-c) {\small $8$};

\node[draw, circle, minimum size=15pt, inner sep=2pt] at ($(v8-c) + (5,0)$) (v15-c) {\small $b^3$};
\node[draw, circle, minimum size=15pt, inner sep=2pt] at ($(v15-c) + (-1,-4)$) (v16-c) {\small $9$};
\node[draw, circle, minimum size=15pt, inner sep=2pt] at ($(v15-c) +
(0,-4)$) (v17-c) {\small $10$};
\node[draw, circle, minimum size=15pt, inner sep=2pt] at ($(v15-c) +
(1,-4)$) (v18-c) {\small $11$};

\node[draw, circle, minimum size=15pt, inner sep=2pt] at ($(v1-c)!0.5!(v8-c) + (0,2)$) (a-c) {\small $a$};
\node[draw, circle, minimum size=15pt, inner sep=2pt] at ($(v15-c) + (0,2)$) (b-c) {\small $b$};

\draw[thick, red] (v1-c) -- (c1-c);
\draw[thick, red] (v1-c) -- (v3-c);
\draw[thick, red] (c1-c) -- (v4-c);
\draw[thick, red] (c1-c) -- (v5-c);
\draw[thick, dotted] (c1-c) -- (c0-c);
\draw[thick, blue] (v8-c) -- (d2-c);
\draw[thick, blue] (v8-c) -- (mid-c);
\draw[thick, blue] (v8-c) -- (e2-c);
\draw[thick, dotted] (d0-c) -- (d2-c);
\draw[thick, blue] (d2-c) -- (v11-c);
\draw[thick, blue] (d2-c) -- (v12-c);
\draw[thick, dotted] (e0-c) -- (e2-c);
\draw[thick, blue] (e2-c) -- (v13-c);
\draw[thick, blue] (e2-c) -- (v14-c);
\draw[thick, green!75!black] (v15-c) -- (v16-c);
\draw[thick, green!75!black] (v15-c) -- (v17-c);
\draw[thick, green!75!black] (v15-c) -- (v18-c);
\draw[thick, dotted] (a-c) -- (b-c);
\draw[thick, dotted] (a-c) -- (v1-c);
\draw[thick, dotted] (a-c) -- (v8-c);
\draw[thick, dotted] (b-c) -- (v15-c);

%
%
\node[] at (-3.5,0) {\small (a)};
\node[] at (-3.5,-6) {\small (b)};
\node[] at (-3.5,-12) {\small (c)};
\end{tikzpicture}
}
\caption{A depiction of the intermediate trees, $T^{(1)},T^{(2)}$ and $T^{(3)}$ in the analysis of the algorithm. The original tree $T^*$ contains all the edges in (a). The output forest partitions the leaves into three connected components: $S_1 = \{1,2,5\}$, $S_2 = \{3,4,6,7,8\}$ and $S_3 = \{9,10,11\}$. The tree $T_1$, which is the subtree of $T^*$ induced by $S_1$, contains the red and purple-dashed edged. The tree $T_2$ contains the blue and purple-dashed edges. And $T_3$ contains the green edges. In (b), we depict $T^{(2)}$, which is obtained by contracting all the edges that are shared among multiple subtrees $T_r$, for $r=1,2,3$. In this example, those are the purple-dashed edges. In (c), we depict $T^{(2)}$, which is obtained from $T^{(2)}$ by adding auxiliary nodes: for each internal node $v$ and each subtree $T_r$ that it touches, we create a new node $v^r$, connect it with $v$ with an edge of weight $1$, and further, connect to $v^r$ all the edges in $T_r$ that were incident with $v$ in $T^{(1)}$. The forest $T^{(3)}$ is obtained from $T^{(2)}$ by removing all the edges that are incident with the nodes $a,b,c,d,e$. These are the dashed edges.}
\label{fig:appendix13}
\end{figure}

We would like to bound the total variation distance between $\Pr_{T^*,\theta^*}[x]$ and $\Pr_{T^{(1)},\theta^{(1)}}[x]$. Notice that for each contracted edge $e$, we have from Theorem~\ref{thm:DMR} that $\theta_e \ge 1-C\xi$ for some universal constant $C>0$. Contracting each such edge $e$ is equivalent to modifying $\theta_e$ to equal $1$, as argued in Lemma~\ref{lem:contract-no-TV}. Hence, the total variation distance for each contraction is bounded by $C\xi/2$ from Lemma~\ref{lem:change-one-edge}. By the triangle inequality, since we contract at most $O(n)$ edges, then $\TV(\Pr_{T^*,\theta^*}[x],\Pr_{T^{(1)},\theta^{(1)}}[x]) \le O(n\xi)$.

Next, we continue to present the second intermediate model, parameterized by $T^{(2)}=(V^{(2)},E^{(2)})$ and $\theta^{(2)}$. To obtain $T^{(2)}=(V^{(2)},E^{(2)})$ and $\theta^{(2)}$, recall that each vertex of $T^{(1)}=(V^{(1)},E^{(1)})$ may correspond to multiple vertices of $T^*$ since $T^{(1)}$ was obtained by contracting edges in $T^*$.
For any $v\in V^{(1)}$, we denote by $A_v$ the set of preimages of $v$ under the transformation from $T^*$ to $T^{(1)}$ (see Definition~\ref{def:preimages} and the paragraph above this definition). In other words, $A_v$ is the set of nodes of $T^*$ that were contracted into $v$. Denote by $\rho(v)$ the set of indices $r\in \{1,\dots,R\}$ of trees $T_r$ that intersect a node that was contracted into $v$. In other words, $\rho(v)$ is the set of indices $r$ such that $A_v \cap V_r \ne \emptyset$, where $V_r$ is the set of vertices of the subtree $T_r$.
Intuitively, for the construction of $T^{(2)}$, we would like to make $\rho(v)$ copies of node $v$, each for a subtree that contains $v$. To be precise, $T^{(2)}$ is constructed as follows:
\begin{itemize}
    \item The set of vertices $V^{(2)}$ of $T^{(2)}$ is obtained from $V^{(1)}$ by adding a new vertex $v^r$, for each $v$ and $r$ such that $r \in \rho(v)$.
    \item For any $v\in V^{(1)}$ and $r\in \rho(v)$, define an edge $(v,v^r)$ of weight $\theta^{(2)}(v,v^r)=1$.
    \item For any edge $(u,v)\in E^{(1)}$, define a new edge in $E^{(2)}$ according to the following considerations:
    \begin{itemize}
        \item If $|\rho(u)\cap\rho(v)| = 1$, denote $\{r\} = \rho(u)\cap\rho(v)$ and add an edge $(u^r,v^r)$ in $E^{(2)}$, with weight $\theta^{(2)}_{(u^r,v^r)} = \theta^{(1)}_{(u,v)}$. This replaces the edge $(u,v)$.
        \item If $|\rho(u)\cap\rho(v)| = \emptyset$, add an edge $(u,v)$ to $E^{(2)}$ with weight $\theta^{(2)}_{(u,v)} = \theta^{(1)}_{(u,v)}$
        \item It is impossible that $|\rho(u)\cap\rho(v)| > 1$, otherwise, $(u,v)$ would have been contracted.
    \end{itemize}
\end{itemize}
(See Figure~\ref{fig:appendix13} for an example of a tree $T^{(2)}$.)
We note that $T^{(1)}$ can be obtained from $T^{(2)}$ by contracting all the edges $(v,v^r)$ for $v\in V^{(1)}$ and $r\in \rho(v)$.
Hence, from Lemma~\ref{lem:contract-no-TV}, $\Pr_{T^{(2)},\theta^{(2)}}[x] = \Pr_{T^{(1)},\theta^{(1)}}[x]$ for all values $x$ on the leaves. In particular,
\[
\TV\lp(\Pr_{T^{(2)},\theta^{(2)}}[x], \Pr_{T^{(1)},\theta^{(1)}}[x]\rp) = 0.
\]

Next, we define $T^{(3)}$ and $\theta^{(3)}$. Recall that the vertices of $T^{(2)}$ are of two types: (1) vertices that are copies of those in $V^{(1)}$; and (2) vertices $v^r$ for $v \in V^{(1)}$ and $r \in \rho(v)$. Then, $T^{(3)}$ is obtained from $T^{(2)}$ by removing all the edges incident to the vertices of category (1), as depicted in Figure~\ref{fig:appendix13}~(c). This creates a \emph{forest}, and we remove from this forest each connected component that is disconnected from the leaves. We prove the following lemma:
\begin{lemma}
There are $R$ connected components in $T^{(3)}$, and the sets of leaves of the connected components are exactly the sets of leaves of $T_1,\dots,T_R$. Namely, for each $T_r$ there exists one connected component of $T^{(3)}$ that has the same leaf-set as $T_r$.
\end{lemma}
\begin{proof}
First, we argue that any two leaves that are in the same tree $T_r$, are also in the same connected component of $T^{(3)}$. Let $i,j$ be two leaves of $T_r$ and consider the edges on the path $P$ between them in $T_r$. For any such edge $(u,v)$, there are two possibilities: (1) This edge was contracted at some point, and there is some $w\in T^{(1)}$ such that $u,v\in A_w$. (2) This edge was not contracted, and there exist some $(w,z)$ in $E^{(1)}$ that is the analogue of $(u,v)$ in $T^{(1)}$ (see Definition~\ref{def:edge-analogue} and its preceding discussion). By definition of $T^{(2)}$, this edge is moved in $T^{(2)}$ to connect $(w^r,z^r)$. In particular, the path from $i$ to $j$ in $T^{(2)}$ contains only vertices of the form $q^r$. Hence, this path is not disconnected in $T^{(3)}$ from the graph.

Next, we argue why two leaves $i,j$ in two different trees $T_r$ and $T_s$, respectively, cannot be connected in $T^{(3)}$. Indeed, the parent of leaf $i$ in $T^{(2)}$ must be of the form $v^r$ while the parent of $j$ must be of the form $v^s$. Notice that in $T^{(3)}$ there is no edge connecting a vertex from $v^r$ with a vertex from $v^3$, hence, these two leaves are necessarily disconnected.

\end{proof}

We would like to use Theorem~\ref{thm:probabilistic} to bound the total variation distance between $\Pr_{T^{(2)},\theta^{(2)}}[x]$ and $\Pr_{T^{(3)},\theta^{(3)}}[x]$, by analyzing the change in the pairwise correlations. To do so, we use the fact that the pairwise correlations that have changed are only those between leaves $i,j$ of different trees $T_r$ and $T_s$, respectively, and by Theorem~\ref{thm:DMR} we have that in $T^*$ those have correlation $|\alpha^*_{ij}|\le C\sqrt{\delta}$ where $\delta$ is defined in Theorem~\ref{thm:DMR} and $C>0$ is a universal constant. Notice that, though, the correlation between two leaves in $T^{(2)}$ can be greater than their correlation in $T^*$. As argued above, the correlation in $T^{(2)}$ equals that in $T^{(1)}$. To compare between the correlations in $T^{(1)}$ and $T^*$, recall that $T^{(1)}$ is obtained from $T$ by contracting edges. Yet, any contracted edge has weight at least $1-C\xi$, hence, contracting the edge can increase the correlation by at most $1/(1-C\xi)$. Since any path between two leaves can contain at most $n$ edges, the contraction can increase the correlation by at most $1/(1-C\xi)^n$. By assumption, $\xi \le 1/(Cn)$, hence this factor is at most a constant. Under that assumption, the correlation in $T^{(2)}$ between any two leaves of different subtrees $T_r$ is bounded by $O(\sqrt{\delta})$. Since in $T^{(3)}$ their correlation becomes $0$, this implies that the pairwise correlations change by at most $O(\sqrt{\delta})$.

Hence, for any $i,j$, $|\alpha^{(2)}_{ij} - \alpha^{(3)}_{ij}| \le O(\sqrt{\delta})$. Since $T^{(2)}$ and $T^{(3)}$ both share the same underlying graph (as removing edges can be done by just replacing the weight with $0$), it follows from Theorem~\ref{thm:probabilistic} that the total variation distance between the leaf distributions of $T^{(2)}$ and $T^{(3)}$ is bounded by $O(n^2\sqrt{\delta})$.

Next, we would like to bound the total variation distance between the leaf distribution of $T^{(3)}$ and the output of the algorithm. Since $T^{(3)}$ and $\tilde{F}$ have the same connected components, the leaf distributions of both factorize the same. In particular, the leaf-sets $S_1,\dots,S_R$ of $T_1,\dots,T_r$ are independent in these product distributions. To bound the total variation distance between the leaf distributions of $T^{(3)}$ and $\tilde{F}$, it suffices to bound the total variation distance with respect to each connected component separately, and then sum the bounds for each component. 
Hence, we will fix some $r \in \{1,\dots,R\}$.
For this end, let us analyze the weights $\tilde{\theta}_e$ given by Algorithm~\ref{alg:unkonwn-top} on the tree $\tilde{T}_r$. Recall that the last step of this algorithm is to use Algorithm~\ref{alg:fixed} on the tree $\tilde{T}_r$ with some parameter $\eta'>0$ and correlations $\hat{\alpha}_{ij}$ that were estmated from samples of the original tree $T^*$. Further, recall that Algorithm~\ref{alg:fixed} is guaranteed to return some weights $\tilde{\theta}_e$ on the edges, such that the pairwise correlations between the leaves, which we denote by $\tilde{\alpha}_{ij}$, are $\eta'$-close to the correlations $\hat{\alpha}_{ij}$ that were given to it as input, namely, $|\tilde{\alpha}_{ij} - \hat{\alpha}_{ij}| \le \eta$ for any pair $i,j$ of leaves. Yet, Algorithm~\ref{alg:fixed} will succeed only if there exist such weights $\tilde{\theta}_e$ that satisfy the above constraint. To that end, we claim the following:
\begin{lemma}\label{lem:lin-prog-unknown}
Fix $r \in \{1,\dots,R\}$.
Then, there exist weights $\theta'_{ij}$ to the edges of $\tilde{T}_r$ such that the corresponding pairwise correlations, $\alpha'_{ij}$ satisfy $|\alpha'_{ij} - \hat{\alpha}_{ij}| \le \eta + O(n\xi)$ for any leaves $i,j\in S_r$.
\end{lemma}
\begin{proof}
Notice that it is sufficient to find weights $\theta'_e$ such that $|\alpha'_{ij} - \alpha^*_{ij}| \le O(n\xi)$, since $|\hat{\alpha}_{ij} - \alpha^*_{ij}|\le \eta$, by the assumption in Lemma~\ref{lem:main-alg}. Since $T_r$ is the subtree induced by the set of leaves $S_r$, our goal is to show that $\alpha'_{ij}$ is $O(n\xi)$-close to the pairwise correlation of $i$ and $j$ accross $T_r$. Hence, this is what we will do.
As a first solution, we propose to set for each edge $e$ its weight $\theta'_e$ to equal its corresponding weight $\theta^*_e$ in $T_r$.
Recall, though, that in the process of transforming $T_i$ to $\tilde{T}_i$, there are two modifications, which implies that we cannot exactly match the edges of $\tilde{T}_i$ with those of $T_i$. We elaborate below on the transformations and how to set $\theta'_e$ given these transformations. 
\begin{itemize}
    \item The first transformation is obtained by contracting some edges of weight $\theta_e \ge 1-O(\xi)$. For these contracted edges, we will not define $\theta'_e$. Contracting these edges changes the pairwise correlations between the leaves by at most $O(n\xi)$, since there can be at most $O(n)$ edges along each path.
    \item The second transformation is a contraction of some nodes of degree $2$. Yet, for each such contraction, there is an easy way to modify the weights such that the pairwise correlations over the leaves does not change. In particular, if $u$ is a node and $v,w$ are its neighbors, then $u$ is being deleted from the graph and $v,w$ are being connected. If we set the weight of the new edge as a multiplication of the weights of the two old edges, then the pairwise correlations between the leaves do not change. In particular, we will define $\theta'$ under this logic: we will track the changes from $T_i$ to $\tilde{T}_i$, and whenever a degree-2 node is being contracted, we modify the weights accordingly: if edges $e$ and $e'$ were contracted to $e''$, we set the weight of $e''$ to equal the multiplication of weights of $e$ and $e'$.
\end{itemize}
Using the above definition of $\theta'_e$ and the above analysis, it follows that $|\alpha'_{ij} - \alpha^*_{ij}| \le O(n\xi)$. This suffices to complete the proof, as explained above. 
\end{proof}
It follows from Lemma~\ref{lem:lin-prog-unknown} that the execution of Algorithm~\ref{alg:fixed} succeeds, if it is run with $\eta' = Cn\xi+\eta$ and a sufficiently large $C>0$. This implies that for any $i,j\in S_r$, $|\tilde{\alpha}_{ij} - \hat\alpha_{ij}|\le \eta' \le \eta + O(n\xi)$. By the triangle inequality, $|\tilde{\alpha}_{ij} - \alpha^*_{ij}| \le 2\eta + O(n\xi)$. Lastly, following the analysis above, it is easy to show that also $|\alpha^{(3)}_{ij} - \alpha^*_{ij}| \le O(n\xi)$. Indeed, the only modification from $T$ to $T^{(3)}$ that affects the pairwise correlation between $i,j\in S_r$ is the contraction of edges of $\theta^*_e \ge 1-O(\xi)$. This affects the pairwise correlation by $O(n\xi)$. By the triangle inequality, we derive that $|\tilde{\alpha}_{ij} - \alpha^{(3)}_{ij}| \le O(n\xi + \eta)$. The last step would be to apply Theorem~\ref{thm:probabilistic} (same topology) to compare between the distribution over $\tilde{T_i}$ and its corresponding connected component in $T^{(3)}$. Yet, this theorem would apply only if the two components have the same topology. While they do not, we note that both trees are contractions of the same tree $T_r$. Hence, we can view both distributions as defined over the tree $T_r$, where the contracted edges have weight $1$. By Theorem~\ref{thm:probabilistic} (same topology) we derive that the total variation distance between the two distributions over $S_r$ is bounded by $O(n\xi |S_r|^2)$. By summing over all $r$, we derive that 
\[
\TV\lp(\Pr_{T^{(3)}, \theta^{(3)}}[x],\ \Pr_{\tilde{F},\tilde{\theta}}[x]\rp)
\le O\lp( n\xi \sum_{r=1}^r |S_r|^2\rp)
\le O(n^3\xi).
\]

By summing up the total variation distances between the auxiliary distributions parameterized by $T^{(i)}$, this concludes the proof of Lemma~\ref{lem:main-alg}.

\subsection{Concluding the proof of Theorem~\ref{thm:alg} (unknown topology)} \label{sec:alg-conclusion}

We use Lemma~\ref{lem:main-alg}. First of all, let us optimize $\xi$ and $\delta$ for a fixed value of $\eta$. This can be achieved by selecting $\delta = \eta^{2/3}n^{2/3}$ and $\xi = \eta^{1/3}n^{-2/3}$. We note that the requirement $\xi \le O(1/n)$ if $\eta \le O(1/n)$. The final bound is $O(n^{7/3}\eta^{1/3})$. To get this below $\epsilon$, we have to set $\eta \le O(\epsilon^3/n^7)$. This requires a sample of size $n \ge \Omega(\log(n/\delta)/\eta^2) = \Omega(n^{14}\log(n/\delta)/\epsilon^6)$.

\subsection{Translating the notation of \cite{daskalakis2009phylogenies}}\label{sec:translate-DMR}

We note that \cite{daskalakis2009phylogenies} uses a different notation. For convenience, we explain the translation in the Ferromagnetic setting where $\alpha_{ij}, \theta_{kl} \ge 0$, however, in order to transition to the non-Ferromagnetic setting one would simply have to replace these quantities with their absolute values, $|\alpha_{ij}|$ and $|\theta_{kl}|$, respectively.

Instead of edge weight $\theta_{(k,\ell)}$, \cite{daskalakis2009phylogenies} use the metric $d_{k,\ell} = -\log \theta_{(k,\ell)}$. Instead of $\alpha_{i,j}$ they use the metric $d_{i,j} = -\log \alpha_{i,j}$. They denote the true (underlying) metric by $d$, whereas they assume that the algorithm receives a \emph{$(\tau,M)$-distorted metric} on the leaves, denoted $\hat{d}$. This means that $|\hat{d}_{i,j}-d_{i,j}|\le \tau$ whenever $d_{i,j} \le M$. Using our notation, this means that $|\log \hat{\alpha}_{ij} - \log \alpha^*_{i,j}| \le \tau$ whenever $\alpha^*_{ij} \ge e^{-M}$. Equivalently, 
\[
-\tau \le \log \hat{\alpha}_{ij} - \log \alpha^*_{i,j} \le \tau \quad \text{whenever } \alpha^*_{ij} \ge e^{-M}
\]
which is equivalent to
\begin{equation}\label{eq:M-tau}
e^{-\tau} \le \hat{\alpha}_{ij}/\alpha^*_{i,j} \le e^{\tau} \quad \text{whenever } \alpha^*_{ij} \ge e^{-M}\enspace.
\end{equation}
We will show how their guarantees can be implied from our guarantees. First, notice that in order for \eqref{eq:M-tau} to hold, it is sufficient to assume that $\tau \le 1$ and
\[
(1-\tau/2) \le \hat{\alpha}_{ij}/\alpha^*_{i,j} \le (1+\tau/2) \quad \text{whenever } \alpha^*_{ij} \ge e^{-M}\enspace.
\]
If we substitute $\xi = \tau/2$ and $\delta = e^{-M}$, the last inequality substitutes to
\[
\alpha^*_{i,j} - \xi \alpha^*_{ij} \le \hat{\alpha}_{ij} \le \alpha^*_{i,j} + \xi \alpha^*_{ij} \quad \text{whenever } \alpha^*_{ij} \ge \delta \enspace,
\]
which is equivalent to
\begin{equation}\label{eq:xi-delta}
|\hat{\alpha}_{ij} - \alpha^*_{i,j}| \le \xi \alpha^*_{ij} \quad \text{whenever } \alpha^*_{ij} \ge \delta \enspace.
\end{equation}
Eq.~\ref{eq:xi-delta} is guaranteed to hold if $|\hat{\alpha}_{i,j}-\alpha^*_{i,j}| \le \xi\delta$. Since in Theorem~\ref{thm:DMR} we assume that $|\hat{\alpha}_{i,j}-\alpha^*_{i,j}| \le \eta$, it suffices to assume that $\eta \le \xi\delta$ in order to imply \eqref{eq:xi-delta}, which in turn implies the conditions in the paper of \cite{daskalakis2009phylogenies}.
\section{Information theoretic bound} \label{sec:info-theor}

\paragraph{Upper bound.}
While the result below was known, we prove it for completeness.
\begin{theorem} \label{thm:info-upper}
There is an algorithm that, given $m$ samples from the leaf-marginal of some tree structured Ising model with $n$ leaves, returns another tree structured Ising model whose total variation distance to the original model is bounded by $\epsilon$, with sample complexity $m = O(n\log(n/\epsilon)/\epsilon^2)$.
\end{theorem}

While it is apparent that the family of tree-structured Ising models is infinite, we will select a finite set which is an $\epsilon$-cover in total variation, and then we will use the following result to learn in total variation distance over a finite set:
\begin{theorem}[\cite{yatracos1985rates}] \label{thm:yat}
Let $\epsilon,\delta>0$.
Given a finite family $\mathcal{C}$ of distributions and $m$ samples from some arbitrary distribution $\mu$, there exists an algorithm such that, with probability $1-\delta$, returns a distribution $\hat{\mu}\in \mathcal{C}$ that satisfies:
\[
\TV(\mu,\hat{\mu}) \le 
3 \inf_{\nu \in \mathcal{C}} \TV(\mu,\nu) + \epsilon, 
\]
with sample complexity $m = O(\log(|\mathcal{C}|/\delta)/\epsilon^2)$.
\end{theorem}

In order to apply Theorem~\ref{thm:yat}, we will construct an $\epsilon$-cover to the set of tree structured Ising models.

\begin{lemma}\label{lem:cover-ising}
For any $\epsilon>0$, there exists a family $\mathcal{C}$ of tree-structured Ising models of log cardinality $\log |\mathcal{C}|\le O(n\log(n/\epsilon))$, such that for any tree structured Ising model, there exists some model from $\mathcal{C}$ such that the total variation distance between these the two leaf distributions of these models is bounded by $\epsilon$.
\end{lemma}
\begin{proof}
For completeness, we prove this lemma using Theorem~\ref{thm:probabilistic}, yet, there are more direct ways to prove this lemma.

Each element of $\mathcal{C}$ will be parameterized by the following:
\begin{itemize}
    \item A tree topology, with $n$ leaves labeled $1,\dots,n$. There can be at most $n^{O(n)}$ distinct trees.
    \item For each edge of the tree, its weight $\theta_e$ is one of $\{0,1/M, 2/M,\dots,1\}$, where $M = \Theta(n^3/\epsilon)$. There can be at most $M^{O(n)}$ possibilities to select the weights.
\end{itemize}
We derive that $|\mathcal{C}|\le (n/\epsilon)^{O(n)}$.

Given some tree $T$ and weight $\theta$, we will find an element of $\mathcal{C}$ that approximates it. In particular, we will take the element from $\mathcal{C}$ that has the same structure and additionally, each of its weights are $1/M$ close to $\theta$. It is easy to see from \eqref{eq:path-correlation} that the pairwise correlations between the leaves $\alpha_{ij}$, are $O(n/M)$-close in absolute value between the two models. Hence, by Theorem~\ref{thm:probabilistic} (fixed topology), the two models are $O(n^3/M)\le \epsilon$ close in total variation between the leaf distributions, provided that $M\ge \Omega(n^3/\epsilon)$. 
\end{proof}

To conclude the proof, we use the algorithm of Theorem~\ref{thm:yat}, applying it on an $\epsilon/4$-cover using the construction in Lemma~\ref{lem:cover-ising}.

\paragraph{Lower bound}

In order to learn latent tree-structured Ising models, when the topology is unknown, the sample complexity is lower bounded by $\Omega(n\log(n)/\epsilon^2)$.
This follows from \cite{koehler2020note}: they prove that the number of samples that are required to learn a \emph{full tree} from samples is $\Omega(n\log(n)/\epsilon^2)$, yet, this proof extends directly to the setting of latent nodes.

\bibliographystyle{plainnat}
\bibliography{refs}

\begin{thebibliography}{68}
\providecommand{\natexlab}[1]{#1}
\providecommand{\url}[1]{\texttt{#1}}
\expandafter\ifx\csname urlstyle\endcsname\relax
  \providecommand{\doi}[1]{doi: #1}\else
  \providecommand{\doi}{doi: \begingroup \urlstyle{rm}\Url}\fi

\bibitem[Acharya et~al.(2018)Acharya, Bhattacharyya, Daskalakis, and
  Kandasamy]{acharya2018learning}
Jayadev Acharya, Arnab Bhattacharyya, Constantinos Daskalakis, and Saravanan
  Kandasamy.
\newblock Learning and testing causal models with interventions.
\newblock \emph{Advances in Neural Information Processing Systems}, 31, 2018.

\bibitem[Aigner et~al.(1984)Aigner, Hsiao, Kapteyn, and
  Wansbeek]{aigner1984latent}
Dennis~J Aigner, Cheng Hsiao, Arie Kapteyn, and Tom Wansbeek.
\newblock Latent variable models in econometrics.
\newblock \emph{Handbook of econometrics}, 2:\penalty0 1321--1393, 1984.

\bibitem[Bartholomew et~al.(2011)Bartholomew, Knott, and
  Moustaki]{bartholomew2011latent}
David~J Bartholomew, Martin Knott, and Irini Moustaki.
\newblock \emph{Latent variable models and factor analysis: A unified
  approach}, volume 904.
\newblock John Wiley \& Sons, 2011.

\bibitem[Bhattacharyya et~al.(2021)Bhattacharyya, Gayen, Price, and
  Vinodchandran]{bhattacharyya2021near}
Arnab Bhattacharyya, Sutanu Gayen, Eric Price, and NV~Vinodchandran.
\newblock Near-optimal learning of tree-structured distributions by chow-liu.
\newblock In \emph{Proceedings of the 53rd Annual ACM SIGACT Symposium on
  Theory of Computing}, 2021.

\bibitem[Bishop(1998)]{bishop1998latent}
Christopher~M Bishop.
\newblock Latent variable models.
\newblock In \emph{Learning in graphical models}, pages 371--403. Springer,
  1998.

\bibitem[Blei et~al.(2017)Blei, Kucukelbir, and McAuliffe]{blei2017variational}
David~M Blei, Alp Kucukelbir, and Jon~D McAuliffe.
\newblock Variational inference: A review for statisticians.
\newblock \emph{Journal of the American statistical Association}, 112\penalty0
  (518):\penalty0 859--877, 2017.

\bibitem[Boix-Adsera et~al.(2022)Boix-Adsera, Bresler, and
  Koehler]{boix2022chow}
Enric Boix-Adsera, Guy Bresler, and Frederic Koehler.
\newblock Chow-liu++: Optimal prediction-centric learning of tree ising models.
\newblock In \emph{2021 IEEE 62nd Annual Symposium on Foundations of Computer
  Science (FOCS)}, pages 417--426. IEEE, 2022.

\bibitem[Bresler(2015)]{bresler2015efficiently}
Guy Bresler.
\newblock Efficiently learning ising models on arbitrary graphs.
\newblock In \emph{Proceedings of the forty-seventh annual ACM symposium on
  Theory of computing}, pages 771--782, 2015.

\bibitem[Bresler and Karzand(2020)]{bresler2020learning}
Guy Bresler and Mina Karzand.
\newblock Learning a tree-structured ising model in order to make predictions.
\newblock \emph{The Annals of Statistics}, 48\penalty0 (2):\penalty0 713--737,
  2020.

\bibitem[Bresler et~al.(2019)Bresler, Koehler, and Moitra]{bresler2019learning}
Guy Bresler, Frederic Koehler, and Ankur Moitra.
\newblock Learning restricted boltzmann machines via influence maximization.
\newblock In \emph{Proceedings of the 51st Annual ACM SIGACT Symposium on
  Theory of Computing}, pages 828--839, 2019.

\bibitem[Brustle et~al.(2020)Brustle, Cai, and Daskalakis]{brustle2020multi}
Johannes Brustle, Yang Cai, and Constantinos Daskalakis.
\newblock Multi-item mechanisms without item-independence: Learnability via
  robustness.
\newblock In \emph{Proceedings of the 21st ACM Conference on Economics and
  Computation}, pages 715--761, 2020.

\bibitem[Canonne et~al.(2017)Canonne, Diakonikolas, Kane, and
  Stewart]{canonne2017testing}
Cl{\'e}ment~L Canonne, Ilias Diakonikolas, Daniel~M Kane, and Alistair Stewart.
\newblock Testing bayesian networks.
\newblock In \emph{Conference on Learning Theory}, pages 370--448. PMLR, 2017.

\bibitem[Chang(1996)]{chang1996full}
Joseph~T Chang.
\newblock Full reconstruction of markov models on evolutionary trees:
  identifiability and consistency.
\newblock \emph{Mathematical biosciences}, 137\penalty0 (1):\penalty0 51--73,
  1996.

\bibitem[Chor and Tuller(2005)]{chor2005maximum}
Benny Chor and Tamir Tuller.
\newblock Maximum likelihood of evolutionary trees is hard.
\newblock In \emph{Annual International Conference on Research in Computational
  Molecular Biology}, pages 296--310. Springer, 2005.

\bibitem[Chow and Wagner(1973)]{chow1973consistency}
C~Chow and T~Wagner.
\newblock Consistency of an estimate of tree-dependent probability
  distributions (corresp.).
\newblock \emph{IEEE Transactions on Information Theory}, 19\penalty0
  (3):\penalty0 369--371, 1973.

\bibitem[Chow and Liu(1968)]{chow1968approximating}
CKCN Chow and Cong Liu.
\newblock Approximating discrete probability distributions with dependence
  trees.
\newblock \emph{IEEE transactions on Information Theory}, 14\penalty0
  (3):\penalty0 462--467, 1968.

\bibitem[Cryan et~al.(2001)Cryan, Goldberg, and
  Goldberg]{cryan2001evolutionary}
Mary Cryan, Leslie~Ann Goldberg, and Paul~W Goldberg.
\newblock Evolutionary trees can be learned in polynomial time in the two-state
  general markov model.
\newblock \emph{SIAM Journal on Computing}, 31\penalty0 (2):\penalty0 375--397,
  2001.

\bibitem[Csur{\"o}s(2002)]{csuros2002fast}
Mikl{\'o}s Csur{\"o}s.
\newblock Fast recovery of evolutionary trees with thousands of nodes.
\newblock \emph{Journal of Computational Biology}, 9\penalty0 (2):\penalty0
  277--297, 2002.

\bibitem[Dagan et~al.(2021)Dagan, Daskalakis, Dikkala, and
  Kandiros]{dagan2021learning}
Yuval Dagan, Constantinos Daskalakis, Nishanth Dikkala, and Anthimos~Vardis
  Kandiros.
\newblock Learning ising models from one or multiple samples.
\newblock In \emph{Proceedings of the 53rd Annual ACM SIGACT Symposium on
  Theory of Computing}, pages 161--168, 2021.

\bibitem[Daskalakis and Pan(2017)]{daskalakis2017square}
Constantinos Daskalakis and Qinxuan Pan.
\newblock Square hellinger subadditivity for bayesian networks and its
  applications to identity testing.
\newblock In \emph{Conference on Learning Theory}, pages 697--703. PMLR, 2017.

\bibitem[Daskalakis and Pan(2021)]{daskalakis2020tree}
Constantinos Daskalakis and Qinxuan Pan.
\newblock Sample-optimal and efficient learning of tree ising models.
\newblock In \emph{Proceedings of the 53rd Annual ACM SIGACT Symposium on
  Theory of Computing}, 2021.

\bibitem[Daskalakis et~al.(2006)Daskalakis, Mossel, and
  Roch]{daskalakis2006optimal}
Constantinos Daskalakis, Elchanan Mossel, and S{\'e}bastien Roch.
\newblock Optimal phylogenetic reconstruction.
\newblock In \emph{Proceedings of the thirty-eighth annual ACM symposium on
  Theory of computing}, pages 159--168, 2006.

\bibitem[Daskalakis et~al.(2009)Daskalakis, Mossel, and
  Roch]{daskalakis2009phylogenies}
Constantinos Daskalakis, Elchanan Mossel, and S{\'e}bastien Roch.
\newblock Phylogenies without branch bounds: Contracting the short, pruning the
  deep.
\newblock In \emph{Annual International Conference on Research in Computational
  Molecular Biology}, pages 451--465. Springer, 2009.

\bibitem[Daskalakis et~al.(2011)Daskalakis, Mossel, and
  Roch]{daskalakis2011evolutionary}
Constantinos Daskalakis, Elchanan Mossel, and S{\'e}bastien Roch.
\newblock Evolutionary trees and the ising model on the bethe lattice: a proof
  of steel’s conjecture.
\newblock \emph{Probability Theory and Related Fields}, 149\penalty0
  (1):\penalty0 149--189, 2011.

\bibitem[Daskalakis et~al.(2019)Daskalakis, Dikkala, and
  Kamath]{daskalakis2019testing}
Constantinos Daskalakis, Nishanth Dikkala, and Gautam Kamath.
\newblock Testing ising models.
\newblock \emph{IEEE Transactions on Information Theory}, 65\penalty0
  (11):\penalty0 6829--6852, 2019.

\bibitem[Daskalakis et~al.(2022)Daskalakis, Dagan, and
  Kandiros]{daskalakis2022EM}
Constantinos Daskalakis, Yuval Dagan, and Anthimos-Vardis Kandiros.
\newblock Where does em converge in gaussian latent tree models?
\newblock In \emph{Conference on Learning Theory (COLT)}, 2022.

\bibitem[Daskalakis et~al.(2018)Daskalakis, Tzamos, and
  Zampetakis]{daskalakis2018bootstrapping}
Costis Daskalakis, Christos Tzamos, and Manolis Zampetakis.
\newblock Bootstrapping em via power em and convergence in the naive bayes
  model.
\newblock In \emph{International Conference on Artificial Intelligence and
  Statistics}, pages 2056--2064. PMLR, 2018.

\bibitem[Dempster et~al.(1977)Dempster, Laird, and Rubin]{dempster1977maximum}
Arthur~P Dempster, Nan~M Laird, and Donald~B Rubin.
\newblock Maximum likelihood from incomplete data via the em algorithm.
\newblock \emph{Journal of the Royal Statistical Society: Series B
  (Methodological)}, 39\penalty0 (1):\penalty0 1--22, 1977.

\bibitem[Devroye et~al.(2020)Devroye, Mehrabian, and
  Reddad]{devroye2020minimax}
Luc Devroye, Abbas Mehrabian, and Tommy Reddad.
\newblock The minimax learning rates of normal and ising undirected graphical
  models.
\newblock \emph{Electronic Journal of Statistics}, 14\penalty0 (1):\penalty0
  2338--2361, 2020.

\bibitem[Ding et~al.(2021)Ding, Daskalakis, and Feizi]{ding2021gans}
Mucong Ding, Constantinos Daskalakis, and Soheil Feizi.
\newblock Gans with conditional independence graphs: On subadditivity of
  probability divergences.
\newblock In \emph{International Conference on Artificial Intelligence and
  Statistics}, pages 3709--3717. PMLR, 2021.

\bibitem[Erd{\H{o}}s et~al.(1999)Erd{\H{o}}s, Steel, Sz{\'e}kely, and
  Warnow]{erdHos1999few}
P{\'e}ter~L Erd{\H{o}}s, Michael~A Steel, L{\'a}szl{\'o}~A Sz{\'e}kely, and
  Tandy~J Warnow.
\newblock A few logs suffice to build (almost) all trees (i).
\newblock \emph{Random Structures \& Algorithms}, 14\penalty0 (2):\penalty0
  153--184, 1999.

\bibitem[Everett(2013)]{everett2013introduction}
B~Everett.
\newblock \emph{An introduction to latent variable models}.
\newblock Springer Science \& Business Media, 2013.

\bibitem[Felsenstein(1973)]{felsenstein1973maximum}
Joseph Felsenstein.
\newblock Maximum-likelihood estimation of evolutionary trees from continuous
  characters.
\newblock \emph{American journal of human genetics}, 25\penalty0 (5):\penalty0
  471, 1973.

\bibitem[Felsenstein(1981)]{felsenstein1981evolutionary}
Joseph Felsenstein.
\newblock Evolutionary trees from gene frequencies and quantitative characters:
  finding maximum likelihood estimates.
\newblock \emph{Evolution}, pages 1229--1242, 1981.

\bibitem[Felsenstein(2004)]{felsenstein2004inferring}
Joseph Felsenstein.
\newblock \emph{Inferring phylogenies}, volume~2.
\newblock Sinauer associates Sunderland, MA, 2004.

\bibitem[Goel(2019)]{goel2019learning}
Surbhi Goel.
\newblock Learning restricted boltzmann machines with arbitrary external
  fields.
\newblock \emph{arXiv preprint arXiv:1906.06595}, 2019.

\bibitem[Goel et~al.(2020)Goel, Klivans, and Koehler]{goel2020boltzmann}
Surbhi Goel, Adam Klivans, and Frederic Koehler.
\newblock From boltzmann machines to neural networks and back again.
\newblock \emph{Advances in Neural Information Processing Systems},
  33:\penalty0 6354--6365, 2020.

\bibitem[Gronau et~al.(2008)Gronau, Moran, and Snir]{gronau2008fast}
Ilan Gronau, Shlomo Moran, and Sagi Snir.
\newblock Fast and reliable reconstruction of phylogenetic trees with very
  short edges.
\newblock In \emph{SODA}, volume~8, pages 379--388, 2008.

\bibitem[Hamilton et~al.(2017)Hamilton, Koehler, and
  Moitra]{hamilton2017information}
Linus Hamilton, Frederic Koehler, and Ankur Moitra.
\newblock Information theoretic properties of markov random fields, and their
  algorithmic applications.
\newblock \emph{Advances in Neural Information Processing Systems}, 30, 2017.

\bibitem[Huson et~al.(1999)Huson, Nettles, and Warnow]{huson1999disk}
Daniel~H Huson, Scott~M Nettles, and Tandy~J Warnow.
\newblock Disk-covering, a fast-converging method for phylogenetic tree
  reconstruction.
\newblock \emph{Journal of computational biology}, 6\penalty0 (3-4):\penalty0
  369--386, 1999.

\bibitem[Jalali et~al.(2011)Jalali, Ravikumar, Vasuki, and
  Sanghavi]{jalali2011learning}
Ali Jalali, Pradeep Ravikumar, Vishvas Vasuki, and Sujay Sanghavi.
\newblock On learning discrete graphical models using group-sparse
  regularization.
\newblock In \emph{Proceedings of the fourteenth international conference on
  artificial intelligence and statistics}, pages 378--387. JMLR Workshop and
  Conference Proceedings, 2011.

\bibitem[Jordan(2004)]{jordan2004graphical}
Michael~I Jordan.
\newblock Graphical models.
\newblock \emph{Statistical science}, 19\penalty0 (1):\penalty0 140--155, 2004.

\bibitem[Kandiros et~al.(2021)Kandiros, Dagan, Dikkala, Goel, and
  Daskalakis]{kandiros2021statistical}
Vardis Kandiros, Yuval Dagan, Nishanth Dikkala, Surbhi Goel, and Constantinos
  Daskalakis.
\newblock Statistical estimation from dependent data.
\newblock In \emph{International Conference on Machine Learning}, pages
  5269--5278. PMLR, 2021.

\bibitem[King et~al.(2003)King, Zhang, and Zhou]{king2003complexity}
Valerie King, Li~Zhang, and Yunhong Zhou.
\newblock On the complexity of distance-based evolutionary tree.
\newblock In \emph{Proceedings of the fourteenth annual ACM-SIAM symposium on
  Discrete algorithms}, page 444. SIAM, 2003.

\bibitem[Klivans and Meka(2017)]{klivans2017learning}
Adam Klivans and Raghu Meka.
\newblock Learning graphical models using multiplicative weights.
\newblock In \emph{2017 IEEE 58th Annual Symposium on Foundations of Computer
  Science (FOCS)}, pages 343--354. IEEE, 2017.

\bibitem[Koehler(2020)]{koehler2020note}
Frederic Koehler.
\newblock A note on minimax learning of tree models.
\newblock 2020.

\bibitem[Koller and Friedman(2009)]{koller2009probabilistic}
Daphne Koller and Nir Friedman.
\newblock \emph{Probabilistic graphical models: principles and techniques}.
\newblock MIT press, 2009.

\bibitem[Lauritzen(1996)]{lauritzen1996graphical}
Steffen~L Lauritzen.
\newblock \emph{Graphical models}, volume~17.
\newblock Clarendon Press, 1996.

\bibitem[Lee et~al.(2006)Lee, Blay, Mooers, Singh, and Oakley]{lee2006comet}
Chunghau Lee, Sigal Blay, Arne~{\O} Mooers, Ambuj Singh, and Todd~H Oakley.
\newblock Comet: A mesquite package for comparing models of continuous
  character evolution on phylogenies.
\newblock \emph{Evolutionary Bioinformatics}, 2:\penalty0 117693430600200022,
  2006.

\bibitem[Moitra et~al.(2021)Moitra, Mossel, and Sandon]{moitra2021learning}
Ankur Moitra, Elchanan Mossel, and Colin~P Sandon.
\newblock Learning to sample from censored markov random fields.
\newblock In \emph{Conference on Learning Theory}, pages 3419--3451. PMLR,
  2021.

\bibitem[Mossel(2007)]{mossel2007distorted}
Elchanan Mossel.
\newblock Distorted metrics on trees and phylogenetic forests.
\newblock \emph{IEEE/ACM Transactions on Computational Biology and
  Bioinformatics}, 4\penalty0 (1):\penalty0 108--116, 2007.

\bibitem[Mossel and Roch(2005)]{mossel2005learning}
Elchanan Mossel and S{\'e}bastien Roch.
\newblock Learning nonsingular phylogenies and hidden markov models.
\newblock In \emph{Proceedings of the thirty-seventh annual ACM symposium on
  Theory of computing}, pages 366--375, 2005.

\bibitem[Narasimhan and Bilmes(2004)]{narasimhan2004}
Mukund Narasimhan and Jeff~A Bilmes.
\newblock Pac-learning bounded tree-width graphical models.
\newblock In \emph{Proc. 20th Ann. Conf. on Uncertainty in Artificial
  Intelligence (UAI)}, 2004.

\bibitem[Pearl(1988)]{pearl1988probabilistic}
Judea Pearl.
\newblock \emph{Probabilistic reasoning in intelligent systems: networks of
  plausible inference}.
\newblock Morgan kaufmann, 1988.

\bibitem[Ravikumar et~al.(2010)Ravikumar, Wainwright, and
  Lafferty]{ravikumar2010high}
Pradeep Ravikumar, Martin~J Wainwright, and John~D Lafferty.
\newblock High-dimensional ising model selection using l 1-regularized logistic
  regression.
\newblock \emph{The Annals of Statistics}, 38\penalty0 (3):\penalty0
  1287--1319, 2010.

\bibitem[Roch(2006)]{roch2006short}
Sebastien Roch.
\newblock A short proof that phylogenetic tree reconstruction by maximum
  likelihood is hard.
\newblock \emph{IEEE/ACM Transactions on Computational Biology and
  Bioinformatics}, 3\penalty0 (1):\penalty0 92--94, 2006.

\bibitem[Roch(2010)]{roch2010toward}
Sebastien Roch.
\newblock Toward extracting all phylogenetic information from matrices of
  evolutionary distances.
\newblock \emph{Science}, 327\penalty0 (5971):\penalty0 1376--1379, 2010.

\bibitem[Roch and Sly(2017)]{roch2017phase}
Sebastien Roch and Allan Sly.
\newblock Phase transition in the sample complexity of likelihood-based
  phylogeny inference.
\newblock \emph{Probability Theory and Related Fields}, 169\penalty0
  (1):\penalty0 3--62, 2017.

\bibitem[Santhanam and Wainwright(2012)]{santhanam2012information}
Narayana~P Santhanam and Martin~J Wainwright.
\newblock Information-theoretic limits of selecting binary graphical models in
  high dimensions.
\newblock \emph{IEEE Transactions on Information Theory}, 58\penalty0
  (7):\penalty0 4117--4134, 2012.

\bibitem[Stamatakis(2006)]{stamatakis2006raxml}
Alexandros Stamatakis.
\newblock Raxml-vi-hpc: maximum likelihood-based phylogenetic analyses with
  thousands of taxa and mixed models.
\newblock \emph{Bioinformatics}, 22\penalty0 (21):\penalty0 2688--2690, 2006.

\bibitem[Tan et~al.(2011)Tan, Anandkumar, Tong, and Willsky]{tan2011large}
Vincent~YF Tan, Animashree Anandkumar, Lang Tong, and Alan~S Willsky.
\newblock A large-deviation analysis of the maximum-likelihood learning of
  markov tree structures.
\newblock \emph{IEEE Transactions on Information Theory}, 57\penalty0
  (3):\penalty0 1714--1735, 2011.

\bibitem[Truell et~al.(2021)Truell, H{\"u}tter, Squires, Zwiernik, and
  Uhler]{truell2021maximum}
Michael Truell, Jan-Christian H{\"u}tter, Chandler Squires, Piotr Zwiernik, and
  Caroline Uhler.
\newblock Maximum likelihood estimation for brownian motion tree models based
  on one sample.
\newblock \emph{arXiv preprint arXiv:2112.00816}, 2021.

\bibitem[Vuffray et~al.(2016)Vuffray, Misra, Lokhov, and
  Chertkov]{vuffray2016interaction}
Marc Vuffray, Sidhant Misra, Andrey Lokhov, and Michael Chertkov.
\newblock Interaction screening: Efficient and sample-optimal learning of ising
  models.
\newblock \emph{Advances in neural information processing systems}, 29, 2016.

\bibitem[Vuffray et~al.(2022)Vuffray, Misra, and Lokhov]{vuffray2022efficient}
Marc Vuffray, Sidhant Misra, and Andrey~Y Lokhov.
\newblock Efficient learning of discrete graphical models.
\newblock \emph{Journal of Statistical Mechanics: Theory and Experiment},
  2021\penalty0 (12):\penalty0 124017, 2022.

\bibitem[Wang and Zhang(2006)]{wang2006severity}
Yi~Wang and Nevin~Lianwen Zhang.
\newblock Severity of local maxima for the em algorithm: Experiences with
  hierarchical latent class models.
\newblock In \emph{Probabilistic Graphical Models}, pages 301--308. Citeseer,
  2006.

\bibitem[Yang(1997)]{yang1997paml}
Ziheng Yang.
\newblock Paml: a program package for phylogenetic analysis by maximum
  likelihood.
\newblock \emph{Computer applications in the biosciences}, 13\penalty0
  (5):\penalty0 555--556, 1997.

\bibitem[Yatracos(1985)]{yatracos1985rates}
Yannis~G Yatracos.
\newblock Rates of convergence of minimum distance estimators and kolmogorov's
  entropy.
\newblock \emph{The Annals of Statistics}, 13\penalty0 (2):\penalty0 768--774,
  1985.

\bibitem[Zwiernik et~al.(2017)Zwiernik, Uhler, and
  Richards]{zwiernik2017maximum}
Piotr Zwiernik, Caroline Uhler, and Donald Richards.
\newblock Maximum likelihood estimation for linear gaussian covariance models.
\newblock \emph{Journal of the Royal Statistical Society: Series B (Statistical
  Methodology)}, 79\penalty0 (4):\penalty0 1269--1292, 2017.

\end{thebibliography}

\end{document}